\documentclass[letterpaper]{article}

\usepackage[draft]{aaai2026} 
\usepackage{times}
\usepackage{helvet}
\usepackage{courier}
\usepackage[hyphens]{url}
\usepackage{graphicx}
\usepackage{natbib}
\usepackage{caption}
\usepackage{algorithm}
\usepackage{amsmath,amssymb,amsthm}
\usepackage{textcomp}
\usepackage{paralist}
\usepackage{eurosym}
\usepackage{soul}
\usepackage[list=true]{subcaption}
\usepackage{arydshln}
\usepackage{multirow}
\usepackage{siunitx}
\usepackage{adjustbox}
\usepackage{array}
\usepackage{diagbox}
\usepackage{tikz}
\usetikzlibrary{arrows.meta,positioning,calc,fit,backgrounds,shapes.geometric,shadows}
\usepackage[capitalize]{cleveref}
\usepackage{algpseudocode}
\usepackage{makecell}

\providecommand{\todo}[1]{\textit{TODO}}

\usepackage{mathtools}
\usepackage{booktabs}
\usepackage{xcolor}
\usepackage{bm}
\usepackage{dsfont}

\newtheorem{theorem}{Theorem}

\newtheorem{corollary}{Corollary}

\newtheorem{proposition}{Proposition}
\theoremstyle{remark}

\newtheorem{assumption}{Assumption}
\theoremstyle{plain}

\usepackage{newfloat}
\usepackage{listings}
\DeclareCaptionStyle{ruled}{labelfont=normalfont,labelsep=colon,strut=off} 
\floatstyle{ruled}
\newfloat{listing}{tb}{lst}{}
\floatname{listing}{Listing}
\newcounter{checksubsection}
\newcounter{checkitem}[checksubsection]

\newcommand{\checksubsection}[1]{%
  \refstepcounter{checksubsection}%
  \paragraph{\arabic{checksubsection}. #1}%
  \setcounter{checkitem}{0}%
}

\newcommand{\checkitem}{%
  \refstepcounter{checkitem}%
  \item[\arabic{checksubsection}.\arabic{checkitem}.]%
}
\newcommand{\question}[2]{\normalcolor\checkitem #1 #2 \color{blue}}
\newcommand{\ifyespoints}[1]{\makebox[0pt][l]{\hspace{-15pt}\normalcolor #1}}

\title{CPDA: Class-Conditional Path Distribution Alignment for Unsupervised Time-Series Domain Adaptation}

\author {
    Felix Ott\textsuperscript{\rm 1},
    Christopher Mutschler\textsuperscript{\rm 1,2}
}
\affiliations {
    \textsuperscript{\rm 1}Fraunhofer Institute for Integrated Circuits IIS, 90411 Nürnberg, Germany\\
    \textsuperscript{\rm 2}Machine Learning and Positioning Systems Department, University of Technology Nürnberg (UTN), 90461 Nürnberg\\
    \{felix.ott, christopher.mutschler\}@iis.fraunhofer.de
}

\begin{document}

\maketitle

\begin{abstract}
Unsupervised time-series domain adaptation (DA) addresses the challenge of transferring a classifier from a labeled source domain to an unlabeled target domain under distribution shifts induced by different users, sensors, devices, acquisition conditions, or temporal dynamics. Existing methods typically mitigate this shift by aligning marginal feature distributions through adversarial training, optimal transport, or moment-based discrepancies. In this paper, we propose Class-Conditional Path Distribution Alignment (\textbf{CPDA}), a non-adversarial discrepancy-based framework that aligns source and target class-conditional latent path distributions rather than only global feature marginals. CPDA introduces a composite signature--spectral kernel that jointly captures pooled semantic features, temporal path structure, frequency-domain information, and low-rank path-signature dynamics, while using source labels and target soft pseudo-labels to perform class-preserving alignment. We further provide a theoretical analysis showing that CPDA defines a valid kernel discrepancy, admits existing moment-matching methods as restricted cases, and yields a class-conditional target-risk bound. Extensive experiments with CNN, ResNet18, and TCN backbones on 13 different time-series DA benchmarks demonstrate the effectiveness of CPDA against 30 discrepancy, adversarial, and pseudo-labeling baselines.
\end{abstract}


\section{Introduction}
\label{label_introduction}

Time-series classification models are increasingly deployed in diverse applications; however, the training and deployment data are collected under different conditions. Distribution shifts may arise from different users, sensors, sampling conditions, or temporal dynamics. This creates a central challenge for time-series DA: a model is trained with labels from a \textit{source} domain, while only unlabeled data are available from the \textit{target} domain. The objective is to learn representations that remain discriminative for the source task while becoming transferable to the target domain.

A common strategy in unsupervised DA is to reduce the discrepancy between source and target feature distributions. Existing methods address this problem through a wide range of alignment principles, including moment matching~\citep{borgwardt_gretton,long_zhu_mmd,zhang_zhang_lan,alipour_tahmoresnezhad,long_cao_wang_DAN}, covariance alignment~\citep{sun_feng_saenko,chen_fu_chen,cherian_sra,suvrit_sra,harandi_salzmann}, information-theoretic divergences~\citep{tzeng_hoffman_zhang,rahman_fookes,shu_bui_narui}, optimal transport~\citep{ott_acmmm}, and adversarial training~\citep{ganin_ustinova_ajakan,liu_xue,zhu_zhuang_wang,wilson_doppa_cook,long_cao_wang}.

Many discrepancy-based methods align source and target domains only at the level of marginal feature distributions, e.g., through domain-confusion losses, MMD-based adaptation, or covariance alignment~\citep{tzeng_hoffman_zhang,long_cao_wang_DAN,sun_feng_saenko}. This can be insufficient for time-series DA, as marginal alignment may ignore class structure and incorrectly match target samples with source samples from different classes~\citep{long_cao_wang,zhu_zhuang_wang}. Moreover, standard feature-level discrepancies often treat learned representations as unordered vectors, discarding temporal order, dynamic evolution, and frequency-domain structure. These limitations are critical for time-series, where discriminative information may lie in local transitions, periodic patterns, and class-specific latent trajectories rather than only in global feature statistics~\citep{he_queen_koker}. To address these limitations, we propose CPDA, a non-adversarial discrepancy-based method for time-series DA that uses \textit{source labels} and \textit{target soft pseudo-labels} to align class-conditional latent path distributions through a composite signature--spectral kernel capturing pooled, temporal, spectral, and path-signature information.

\textbf{Contributions.} The proposed CPDA method makes the following contributions. (1) We introduce CPDA, a class-conditional discrepancy for time-series DA that aligns ${P_s(Z_{1:T} | Y=c)}$ and ${P_t(Z_{1:T} | Y=c)}$ rather than only the marginal feature distributions, and combine it with target-smoothness regularization through VAT to stabilize pseudo-label-based alignment. (2) We propose a signature--spectral kernel that jointly captures latent temporal paths, local dynamics, frequency structure, and pooled semantic representations. (3) We derive a class-conditional target-risk bound showing that the target error is controlled by the source risk and the CPDA discrepancy. (4) We show that common discrepancy-based methods such as MMD, CORAL, DAN, and HoMM arise as restricted cases or lower-order projections of the proposed framework. (5) CPDA provides a non-adversarial, stable, and theoretically interpretable alternative to domain-discriminator-based alignment while directly addressing class mismatch and temporal structure.
\section{Related Work}
\label{sec:related_work}

A central line of unsupervised DA reduces source--target shift by minimizing explicit discrepancies between feature distributions. DDC encourages domain-invariant representations through a domain-confusion loss~\citep{tzeng_hoffman_zhang}, while MMD-based methods compare source and target distributions via RKHS mean embeddings~\citep{borgwardt_gretton}. DAN extends this principle by applying multi-kernel MMD to deep task-specific layers~\citep{long_cao_wang_DAN}. A second group aligns feature statistics directly: CORAL matches covariance matrices~\citep{sun_feng_saenko}, Jeffreys- and Stein-CORAL variants use geometry-aware covariance discrepancies~\citep{cherian_sra,suvrit_sra,harandi_salzmann}, MMCD jointly matches means and covariances~\citep{zhang_zhang_lan,alipour_tahmoresnezhad}, and HoMM extends alignment to higher-order moments~\citep{chen_fu_chen}. MMDA~\citep{rahman_fookes} combines multiple such discrepancy terms within a deep adaptation objective~\citep{rahman_fookes}. While these methods are stable and non-adversarial, they typically align marginal vector-valued features and therefore do not explicitly preserve class-conditional temporal structure.

In parallel, adversarial and decision-boundary methods learn transferable representations through domain discrimination or classifier disagreement. DANN uses gradient reversal~\citep{ganin_ustinova_ajakan}, CDAN conditions the discriminator on class predictions~\citep{long_cao_wang}, MCD maximizes classifier discrepancy~\citep{saito_watanabe}, and DIRT-T combines entropy minimization with virtual adversarial training~\citep{shu_bui_narui}. DSAN performs local MMD-based subdomain alignment~\citep{zhu_zhuang_wang}, while CoDATS and CoTMix adapt adversarial and contrastive learning to temporal representations~\citep{wilson_doppa_cook,eldele_ragab_cotmix}, and RAINCOAT combines temporal and frequency information under feature and label shifts~\citep{he_queen_koker}. AdaMatch uses distribution alignment and confidence-based pseudo-labeling~\citep{berthelot_roelofs}, OVANet addresses universal DA through one-vs-all classification~\citep{saito_saenko}, DANCE performs neighborhood-based clustering~\citep{saito_kim_dance}, and AdvSKM learns adversarial kernels~\citep{liu_xue}. These methods emphasize conditional and boundary-aware adaptation, but most do not explicitly align class-conditional latent paths. Optimal transport estimates sample-level alignment plans, efficiently approximated by Sinkhorn, but typically matches marginal empirical distributions without encoding class-conditional temporal structure.

CPDA builds on discrepancy-based DA but addresses two key limitations of prior work: marginal alignment and vector-level feature comparison. Instead of matching only $P_s(Z)$ and $P_t(Z)$, CPDA aligns class-conditional latent path distributions using source labels and target soft pseudo-labels. Moreover, its composite signature--spectral kernel captures pooled semantics, temporal path structure, frequency-domain information, and low-rank path-signature dynamics. Thus, MMD, CORAL, DAN, and HoMM can be viewed as restricted cases or lower-order projections of CPDA, while CPDA remains non-adversarial and can be combined with VAT to stabilize pseudo-label-alignment.
\section{Problem Statement}
\label{sec:problem_statement}

We consider unsupervised domain adaptation (UDA) for time-series classification.
Let $\mathcal{X} \subseteq \mathbb{R}^{C_{\mathrm{in}}\times L}$
denote the input space of multivariate time-series with $C_{\mathrm{in}}$ input channels and temporal length $L$.
The label space is $\mathcal{Y}=\{1,\dots,K\}$.
We are given a labeled source-domain sample $\mathcal{D}_s = \{(x_i^s,y_i^s)\}_{i=1}^{n_s} \sim P_s(X,Y)$, and an unlabeled target-domain sample
$\mathcal{D}_t = \{x_j^t\}_{j=1}^{n_t} \sim P_t(X)$.
The goal is to learn a classifier $
    h_{\theta,\phi}(x)
    =
    g_\phi(f_\theta(x))$,
where $f_\theta$ is a time-series feature extractor and $g_\phi$ is a classifier, such that the target risk
\begin{equation}
    R_t(h_{\theta,\phi})
    =
    \mathbb{E}_{(X,Y)\sim P_t}
    \big[
        \ell(h_{\theta,\phi}(X),Y)
    \big]
\end{equation}
is minimized, although target labels are not observed during training.

Most discrepancy-based DA methods align marginal feature distributions, $P_s(Z) \approx P_t(Z)$,
where $Z=f_\theta(X)$ denotes the learned representation.
However, for time-series classification this is often insufficient for three reasons:
(i) First, marginal alignment ignores class structure and may align samples from different classes.
(ii) Second, standard feature-level discrepancies treat the representation as an unordered vector and therefore discard temporal order.
(iii) Third, low-order moment matching, such as mean or covariance matching, may fail to capture non-Gaussian, dynamic, and frequency-dependent differences between source and target time-series.

To address these limitations, we propose \emph{Class-Conditional Path Distribution Alignment} (CPDA).
Instead of aligning only the marginal distributions of vector-valued features, CPDA aligns class-conditional distributions of latent time-series paths:
\begin{equation}
    P_s(Z_{1:T}|Y=c)
    \approx
    P_t(Z_{1:T}|Y=c),
    \quad c=1,\dots,K.
\end{equation}
The target-side class condition is estimated using soft pseudo-labels produced by the current classifier. An overview of notations is given in App.~\ref{app:notations}.

\section{Class-Conditional Path Distribution Alignment}
\label{sec:method}

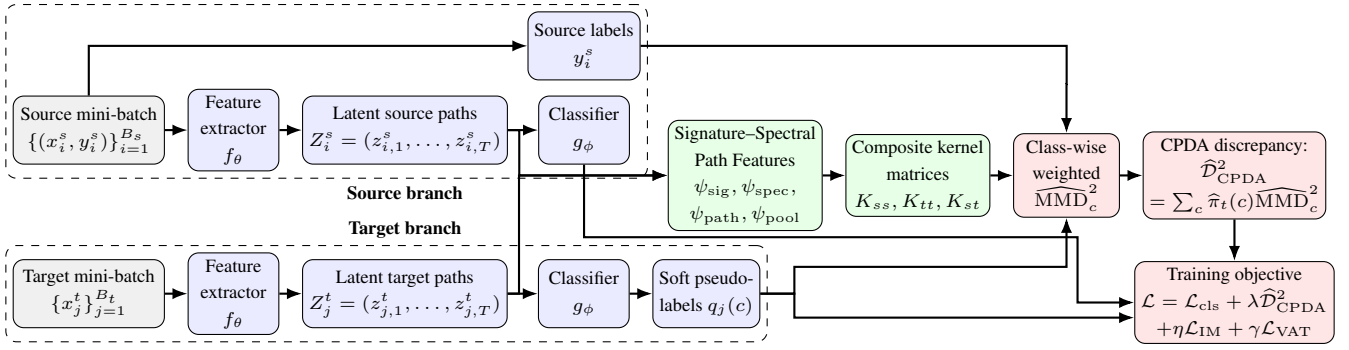
\begin{figure*}[t]
\centering
\scriptsize
\begin{tikzpicture}[
    node distance=0.75cm and 1.0cm,
    block/.style={
        rectangle,
        draw,
        rounded corners,
        align=center,
        minimum width=1.2cm,
        minimum height=0.9cm,
        fill=blue!8
    },
    data/.style={
        rectangle,
        draw,
        rounded corners,
        align=center,
        minimum width=2.0cm,
        minimum height=0.9cm,
        fill=gray!12
    },
    kernel/.style={
        rectangle,
        draw,
        rounded corners,
        align=center,
        minimum width=1.4cm,
        minimum height=0.95cm,
        fill=green!10
    },
    loss/.style={
        rectangle,
        draw,
        rounded corners,
        align=center,
        minimum width=1.4cm,
        minimum height=0.95cm,
        fill=red!10
    },
    arrow/.style={
        -{Latex[length=2mm]},
        thick
    }
]

\node[data] (src) {Source mini-batch\\$\{(x_i^s,y_i^s)\}_{i=1}^{B_s}$};
\node[data, below=1.25cm of src] (trg) {Target mini-batch\\$\{x_j^t\}_{j=1}^{B_t}$};

\node[block, right=0.30cm of src] (fs) {Feature\\extractor\\$f_\theta$};
\node[block, right=0.30cm of trg] (ft) {Feature\\extractor\\$f_\theta$};

\node[block, right=0.30cm of fs] (zs) {Latent source paths\\$Z_i^s=(z_{i,1}^s,\ldots,z_{i,T}^s)$};
\node[block, right=0.30cm of ft] (zt) {Latent target paths\\$Z_j^t=(z_{j,1}^t,\ldots,z_{j,T}^t)$};

\node[block, right=0.40cm of zs] (cs) {Classifier\\$g_\phi$};
\node[block, right=0.40cm of zt] (ct) {Classifier\\$g_\phi$};

\node[block, above=0.20cm of cs] (ys) {Source labels\\$y_i^s$};
\node[block, right=0.30cm of ct] (qt) {Soft pseudo-\\labels $q_j(c)$};

\node[kernel, right=0.5cm of cs, yshift=-0.60cm] (feat) {
Signature--Spectral\\Path Features\\
$\psi_{\rm sig},\psi_{\rm spec},$\\$\psi_{\rm path},\psi_{\rm pool}$
};

\node[kernel, right=0.30cm of feat] (kernels) {
Composite kernel\\matrices\\
$K_{ss},K_{tt},K_{st}$
};

\node[loss, right=0.30cm of kernels] (mmd) {
Class-wise\\weighted\\$\widehat{\mathrm{MMD}}_c^2$
};

\node[loss, right=0.30cm of mmd] (cpda) {
CPDA discrepancy:\\
$\widehat{\mathcal{D}}_{\rm CPDA}^2$\\
$=
\sum_c \widehat{\pi}_t(c)\widehat{\mathrm{MMD}}_c^2$
};

\node[loss, below=0.55cm of cpda] (total) {
Training objective\\
$\mathcal{L} = \mathcal{L}_{\rm cls} + \lambda\widehat{\mathcal{D}}_{\rm CPDA}^2$\\
$+ \eta\mathcal{L}_{\rm IM} + \gamma\mathcal{L}_{\mathrm{VAT}}$
};

\draw[arrow] (src) -- (fs);
\draw[arrow] (fs) -- (zs);
\draw[arrow] (zs) -- (cs);

\draw[arrow] (trg) -- (ft);
\draw[arrow] (ft) -- (zt);
\draw[arrow] (zt) -- (ct);
\draw[arrow] (ct) -- (qt);

\draw[arrow] (feat) -- (kernels);
\draw[arrow] (kernels) -- (mmd);
\draw[arrow] (mmd) -- (cpda);
\draw[arrow] (cpda) -- (total);

\draw[arrow] (zs.east) -- ++(0.15,0) |- (feat.west);
\draw[arrow] (zt.east) -- ++(0.15,0) |- (feat.west);

\coordinate (upperlane) at ($(mmd.north)+(0,0.65)$);
\coordinate (lowerlane) at ($(mmd.south)+(0,-0.65)$);

\draw[arrow] (src.north) -- ++(0,0) |- (ys.west);

\draw[arrow] (ys.east) -- ++(5.636,0) |- (upperlane) -| (mmd.north);
\draw[arrow] (qt.east) -- ++(0.45,0) |- (lowerlane) -| (mmd.south);

\coordinate (losslane) at ($(total.west)+(-0.75,0)$);

\draw[arrow] (cs.south) -- ++(0,-0.96) -| (losslane) -- (total.west);
\draw[arrow] (qt.east) -- ++(0.45,0) |- ($(total.west)+(0,-0.20)$);

\node[
    draw,
    dashed,
    rounded corners,
    fit=(src)(fs)(zs)(cs)(ys),
    inner sep=0.09cm
] (sourcebox) {};

\node[
    draw,
    dashed,
    rounded corners,
    fit=(trg)(ft)(zt)(ct)(qt),
    inner sep=0.09cm
] (targetbox) {};

\coordinate (branchlabelx) at (zs.center);

\node[
    font=\bfseries,
    anchor=north,
    fill=white,
    inner sep=1pt,
    yshift=-0.5mm
] at (sourcebox.south -| branchlabelx)
{Source branch};

\node[
    font=\bfseries,
    anchor=south,
    fill=white,
    inner sep=1pt,
    yshift=0.5mm
] at (targetbox.north -| branchlabelx)
{Target branch};
\end{tikzpicture}
\caption{\textbf{Overview of the proposed CPDA methodology.} Source and target time-series are mapped to latent temporal paths by a shared feature extractor. CPDA constructs pooled, temporal-path, spectral, and signature-based representations from these latent paths and evaluates a composite kernel. Source labels and target soft pseudo-labels provide class-conditional weights.}
\label{fig:cpda_overview}
\end{figure*}

Figure~\ref{fig:cpda_overview} illustrates the CPDA training pipeline. The method uses a shared feature extractor to map source and target time-series into latent temporal paths. These paths are transformed into complementary pooled, temporal, spectral, and signature-based representations, which define the composite CPDA kernel. Source labels and target soft pseudo-labels provide class-conditional weights for computing class-wise MMD discrepancies. The resulting CPDA discrepancy is added to the source classification loss and the optional target information maximization term to update the feature extractor and classifier end-to-end. We provide additional theoretical details in App.~\ref{app:theory}.

\subsection{Latent Path Representation}

Let the feature extractor map an input time-series $x$ to a latent temporal representation $Z_\theta(x) = (z_1,\dots,z_T), z_t\in\mathbb{R}^{d}$.
In convolutional architectures, the flattened feature vector can be reshaped into a latent path
$Z_\theta(x)\in\mathbb{R}^{T\times d}$,
where $T$ corresponds to the temporal feature length and $d$ to the number of latent channels.
For backbones that output a single vector, CPDA reduces to a degenerate one-step path, while still preserving its class-conditional distributional interpretation.

To encode temporal order and dynamics, we define the augmented latent path
$A_\theta(x)_t
    =
    \big[
        \frac{t}{T},
        z_t,
        \Delta z_t
    \big], \Delta z_t = z_t-z_{t-1}$,
with $\Delta z_1=0$.
The explicit time coordinate prevents invariance to arbitrary temporal permutations, and the increment term captures local dynamics.

\subsection{Signature--Spectral Path Features}

CPDA compares time-series using a composite kernel defined on several complementary path representations.

\paragraph{Pooled Feature.}
The pooled latent representation is
\begin{equation}
    \psi_{\mathrm{pool}}(x)
    =
    \frac{1}{T}
    \sum_{t=1}^{T} z_t.
\end{equation}

\paragraph{Path Feature.}
The flattened augmented path feature is
\begin{equation}
    \psi_{\mathrm{path}}(x)
    =
    \operatorname{vec}
    \big(
        A_\theta(x)_1,\dots,A_\theta(x)_T
    \big).
\end{equation}
This term preserves temporal ordering at the latent-feature level~\citep{kiraly_oberhauser}.

\paragraph{Spectral Feature.}
The frequency-domain representation is defined by applying a discrete Fourier transform along the temporal axis:
\begin{equation}
    \psi_{\mathrm{spec}}(x)
    =
    \operatorname{vec}
    \big(
        \log
        \big(
            1+
            \big|
                \mathcal{F}_t
                \big[
                    Z_\theta(x)
                \big]
            \big|^2
        \big)
    \big),
\end{equation}
where $\mathcal{F}_t[\cdot]$ denotes the Fourier transform along the time dimension~\citep{alan_oppenheim}. This term captures periodic and oscillatory structure that may be stable across domains.

\paragraph{Low-Rank Path-Signature Feature.}
Let
$\widetilde{Z}_\theta(x)_t
    =
    \big[
        \frac{t}{T},
        z_t
    \big]
    \in \mathbb{R}^{d+1}.
$
Define increments
\begin{equation}
    u_t
    =
    P
    \big(
        \widetilde{Z}_\theta(x)_{t+1}
        -
        \widetilde{Z}_\theta(x)_t
    \big),
    \quad
    t=1,\dots,T-1,
\end{equation}
where $P\in\mathbb{R}^{r\times(d+1)}$ is a fixed random projection matrix.
The first-order signature feature is
\begin{equation}
    S_1(x)
    =
    \sum_{t=1}^{T-1} u_t,
\end{equation}
and the ordered cross-increment interaction feature is
\begin{equation}
    S_2(x)
    =
    \sum_{1\leq i<j\leq T-1}
    u_i \otimes u_j.
\end{equation}
The resulting low-rank truncated signature representation is
\begin{equation}
    \psi_{\mathrm{sig}}(x)
    =
    \big[
        S_1(x),
        \operatorname{vec}(S_2(x))
    \big].
\end{equation}
Higher-order signature terms may also be used, but the second-order version provides a favorable balance between expressiveness and computational cost~\citep{boedihardjo_geng}.

\subsection{Composite Signature--Spectral Kernel}

For each representation
\begin{equation}
    a\in
    \{
        \mathrm{pool},
        \mathrm{path},
        \mathrm{spec},
        \mathrm{sig}
    \},
\end{equation}
we define a Gaussian kernel
\begin{equation}
    k_a(x,x')
    =
    \exp
    \Big(
        -
        \frac{
            \big\|
                \overline{\psi}_a(x)
                -
                \overline{\psi}_a(x')
            \big\|_2^2
        }{
            \sigma_a^2
        }
    \Big),
\end{equation}
where $\overline{\psi}_a(\cdot)$ denotes the normalized representation.
In practice, a multi-kernel Gaussian form is used~\citep{long_cao_wang_DAN}:
\begin{equation}
    k_a(x,x')
    =
    \frac{1}{M}
    \sum_{m=1}^{M}
    \exp
    \Big(
        -
        \frac{
            \big\|
                \overline{\psi}_a(x)
                -
                \overline{\psi}_a(x')
            \big\|_2^2
        }{
            \sigma_{a,m}^2
        }
    \Big).
\end{equation}
The full CPDA kernel is the nonnegative weighted sum
\begin{equation}
\begin{split}
    k_{\mathrm{CPDA}}(x,x')
    =\,
    &\alpha_{\mathrm{sig}} k_{\mathrm{sig}}(x,x')
    +
    \alpha_{\mathrm{spec}} k_{\mathrm{spec}}(x,x')\\
    +
    &\alpha_{\mathrm{path}} k_{\mathrm{path}}(x,x')
    +
    \alpha_{\mathrm{pool}} k_{\mathrm{pool}}(x,x'),
\end{split}
\end{equation}
where $\alpha_{\mathrm{sig}},
    \alpha_{\mathrm{spec}},
    \alpha_{\mathrm{path}},
    \alpha_{\mathrm{pool}}
    \geq 0$.
The corresponding reproducing kernel Hilbert space (RKHS) is denoted by $\mathcal{H}_{\mathrm{CPDA}}$.

\subsection{Class-Conditional Signature--Spectral MMD}

Let
\begin{equation}
    q_{\theta,\phi}(c|x)
    =
    \operatorname{softmax}
    \big(
        g_\phi(f_\theta(x))
    \big)_c
\end{equation}
denote the classifier posterior probability for class $c$.
For the source domain, class membership is known through $y_i^s$.
For the target domain, class membership is approximated using the soft pseudo-labels $q_{\theta,\phi}(c|x_j^t)$~\citep{lee_pseudo_labels}. For each class $c$, define source and target weights
\begin{equation}
    w_{i,c}^s
    =
    \mathbb{I}[y_i^s=c],
    \qquad
    w_{j,c}^t
    =
    q_{\theta,\phi}(c|x_j^t).
\end{equation}
The normalized class weights are
\begin{equation}
    \bar{w}_{i,c}^s
    =
    \frac{
        w_{i,c}^s
    }{
        \sum_{i'=1}^{n_s} w_{i',c}^s
    },
    \qquad
    \bar{w}_{j,c}^t
    =
    \frac{
        w_{j,c}^t
    }{
        \sum_{j'=1}^{n_t} w_{j',c}^t
    }.
\end{equation}
The empirical class-conditional mean embeddings~\citep{gretton2012kernel,smola_gretton} are
\begin{equation}
\begin{split}
    \widehat{\mu}_{s,c}
    =
    \sum_{i=1}^{n_s}
    \bar{w}_{i,c}^s
    k_{\mathrm{CPDA}}(x_i^s,\cdot),\,\,\widehat{\mu}_{t,c}
    =
    \sum_{j=1}^{n_t}
    \bar{w}_{j,c}^t
    k_{\mathrm{CPDA}}(x_j^t,\cdot).
\end{split}
\end{equation}
The class-conditional CPDA discrepancy is then
\begin{equation}
    \widehat{\mathcal{D}}_{\mathrm{CPDA}}^2
    =
    \sum_{c=1}^{K}
    \widehat{\pi}_t(c)
    \big\|
        \widehat{\mu}_{s,c}
        -
        \widehat{\mu}_{t,c}
    \big\|_{\mathcal{H}_{\mathrm{CPDA}}}^2,
\end{equation}
where
\begin{equation}
    \widehat{\pi}_t(c)
    =
    \frac{1}{n_t}
    \sum_{j=1}^{n_t}
    q_{\theta,\phi}(c|x_j^t)
\end{equation}
is the estimated target class prior.

Equivalently, using kernel matrices, the class-wise empirical discrepancy is
\begin{equation}
\begin{aligned}
    \widehat{\mathrm{MMD}}_{c}^{2}
    &=
    \sum_{i,i'=1}^{n_s}
    \bar{w}_{i,c}^{s}
    \bar{w}_{i',c}^{s}
    k_{\mathrm{CPDA}}(x_i^s,x_{i'}^s)
    \\
    &+
    \sum_{j,j'=1}^{n_t}
    \bar{w}_{j,c}^{t}
    \bar{w}_{j',c}^{t}
    k_{\mathrm{CPDA}}(x_j^t,x_{j'}^t)
    \\
    &-
    2
    \sum_{i=1}^{n_s}
    \sum_{j=1}^{n_t}
    \bar{w}_{i,c}^{s}
    \bar{w}_{j,c}^{t}
    k_{\mathrm{CPDA}}(x_i^s,x_j^t).
\end{aligned}
\end{equation}
Thus,
\begin{equation}
    \widehat{\mathcal{D}}_{\mathrm{CPDA}}^2
    =
    \sum_{c=1}^{K}
    \widehat{\pi}_t(c)
    \widehat{\mathrm{MMD}}_{c}^{2}.
\end{equation}

\subsection{Training Objective}

The source classification loss is
\begin{equation}
    \mathcal{L}_{\mathrm{cls}}
    =
    \frac{1}{n_s}
    \sum_{i=1}^{n_s}
    \ell_\text{CE}
    \Big(
        g_\phi\big(f_\theta(x_i^s)\big),
        y_i^s
    \Big).
\end{equation}
To stabilize target pseudo-labeling and avoid degenerate solutions, we optionally use a target information maximization regularizer:
\begin{equation}
    \mathcal{L}_{\mathrm{IM}}
    =
    \frac{1}{n_t}
    \sum_{j=1}^{n_t}
    H
    \Big(
        q_{\theta,\phi}(\cdot|x_j^t)
    \Big)
    -
    H
    \Big(
        \frac{1}{n_t}
        \sum_{j=1}^{n_t}
        q_{\theta,\phi}(\cdot|x_j^t)
    \Big),
\end{equation}
where $H(\cdot)$ denotes Shannon entropy.
The first term encourages confident target predictions, while the second term discourages class collapse. The complete CPDA objective is
\begin{equation}
    \boxed{
    \mathcal{L}_{\mathrm{CPDA}}
    =
    \mathcal{L}_{\mathrm{cls}}
    +
    \lambda
    \widehat{\mathcal{D}}_{\mathrm{CPDA}}^2
    +
    \eta
    \mathcal{L}_{\mathrm{IM}}
    +
    \gamma
    \mathcal{L}_{\mathrm{VAT}}
    }
\end{equation}
with hyperparameters $\lambda,\eta, \gamma \geq 0$.

While we provide the computational complexity in App.~\ref{app:complexity}, we give details on the CPDA algorithm in App.~\ref{app:algorithm}.

\section{Theoretical Analysis}
\label{sec:theory}

\begin{proposition}[Validity of the CPDA kernel]
\label{prop:valid_kernel}
Assume that
\begin{equation}
    \alpha_{\mathrm{sig}},
    \alpha_{\mathrm{spec}},
    \alpha_{\mathrm{path}},
    \alpha_{\mathrm{pool}}
    \geq 0.
\end{equation}
Then $k_{\mathrm{CPDA}}$ is a positive semidefinite kernel.
\end{proposition}

\begin{proof}
For each representation $\psi_a$, the Gaussian kernel
\begin{equation}
    k_a(x,x')
    =
    \exp
    \Big(
        -
        \frac{
            \|\overline{\psi}_a(x)-\overline{\psi}_a(x')\|_2^2
        }{
            \sigma_a^2
        }
    \Big)
\end{equation}
is positive semidefinite.
A nonnegative weighted sum of positive semidefinite kernels is again positive semidefinite.
Therefore, $
    k_{\mathrm{CPDA}}
    = \sum_a \alpha_a k_a $ is positive semidefinite.
\end{proof}

\begin{proposition}[CPDA as an integral probability metric]
\label{prop:ipm}
Let $\mathcal{H}_{\mathrm{CPDA}}$ be the RKHS induced by $k_{\mathrm{CPDA}}$.
For any class $c$, the population class-conditional CPDA discrepancy satisfies
\begin{equation}
\begin{split}
    &\mathrm{MMD}_{\mathrm{CPDA}}
    (P_s^c,P_t^c)
    =\\
    &\sup_{\|f\|_{\mathcal{H}_{\mathrm{CPDA}}}\leq 1}
    \big|
        \mathbb{E}_{X\sim P_s^c} f(X)
        -
        \mathbb{E}_{X\sim P_t^c} f(X)
    \big|,
\end{split}
\end{equation}
where $P_s^c = P_s(X|Y=c)$ and $P_t^c = P_t(X|Y=c)$.
\end{proposition}

\begin{proof}
By the reproducing property, the kernel mean embeddings are
\begin{equation}
    \mu_s^c
    =
    \mathbb{E}_{X\sim P_s^c}
    k_{\mathrm{CPDA}}(X,\cdot),
    \quad
    \mu_t^c
    =
    \mathbb{E}_{X\sim P_t^c}
    k_{\mathrm{CPDA}}(X,\cdot).
\end{equation}
For any $f\in\mathcal{H}_{\mathrm{CPDA}}$,
\begin{equation}
    \mathbb{E}_{P_s^c} f(X)
    -
    \mathbb{E}_{P_t^c} f(X)
    =
    \langle f,\mu_s^c-\mu_t^c\rangle_{\mathcal{H}_{\mathrm{CPDA}}}.
\end{equation}
Taking the supremum over the unit ball and applying Cauchy--Schwarz gives
\begin{equation}
    \sup_{\|f\|_{\mathcal{H}_{\mathrm{CPDA}}}\leq 1}
    \big|
        \langle f,\mu_s^c-\mu_t^c\rangle_{\mathcal{H}_{\mathrm{CPDA}}}
    \big|
    =
    \|\mu_s^c-\mu_t^c\|_{\mathcal{H}_{\mathrm{CPDA}}}.
\end{equation}
This is exactly the MMD induced by $k_{\mathrm{CPDA}}$.
\end{proof}

\begin{theorem}[Class-conditional target-risk bound]
\label{thm:target_risk_bound}
Let
\begin{equation}
    R_t(h)
    =
    \sum_{c=1}^{K}
    \pi_t(c)
    \mathbb{E}_{X\sim P_t^c}
    \ell \big(h(X), c \big)
\end{equation}
be the target risk, where $\pi_t(c)=P_t(Y=c)$.
Assume that for each class $c$, the function $x \mapsto \ell \big(h(x), c \big)$ belongs to $\mathcal{H}_{\mathrm{CPDA}}$ and satisfies $\|\ell(h(\cdot),c)\|_{\mathcal{H}_{\mathrm{CPDA}}} \leq B$. Then
\begin{equation}
    R_t(h)
    \leq
    R_s^{\pi_t}(h)
    +
    B
    \sum_{c=1}^{K}
    \pi_t(c)
    \mathrm{MMD}_{\mathrm{CPDA}}
    (P_s^c,P_t^c),
\end{equation}
where
\begin{equation}
    R_s^{\pi_t}(h)
    =
    \sum_{c=1}^{K}
    \pi_t(c)
    \mathbb{E}_{X\sim P_s^c}
    \ell \big(h(X), c \big)
\end{equation}
is the source class-conditional risk reweighted by target priors.
\end{theorem}

\begin{proof}
We decompose the target risk class-wise:
\begin{equation}
    R_t(h)
    =
    \sum_{c=1}^{K}
    \pi_t(c)
    \mathbb{E}_{X\sim P_t^c}
    \ell \big(h(X), c \big).
\end{equation}
Adding and subtracting the corresponding source class-conditional expectation gives
\begin{equation}
\begin{aligned}
    &R_t(h)
    =
    \sum_{c=1}^{K}
    \pi_t(c)
    \mathbb{E}_{X\sim P_s^c}
    \ell\big(h(X), c \big)
    \\
    &+
    \sum_{c=1}^{K}
    \pi_t(c)
    \big[
        \mathbb{E}_{X\sim P_t^c}
        \ell\big(h(X), c \big)
        -
        \mathbb{E}_{X\sim P_s^c}
        \ell\big(h(X), c \big)
    \big].
\end{aligned}
\end{equation}
The first term is $R_s^{\pi_t}(h)$.
For the second term, Proposition~\ref{prop:ipm} and the RKHS norm assumption imply
\begin{equation}
\begin{split}
    \big|
        \mathbb{E}_{X\sim P_t^c}
        \ell\big(h(X), c \big)
        -
        &\mathbb{E}_{X\sim P_s^c}
        \ell\big(h(X), c \big)
    \big|\\
    &\leq
    B
    \mathrm{MMD}_{\mathrm{CPDA}}
    (P_s^c,P_t^c).
\end{split}
\end{equation}
Therefore,
\begin{equation}
    R_t(h)
    \leq
    R_s^{\pi_t}(h)
    +
    B
    \sum_{c=1}^{K}
    \pi_t(c)
    \mathrm{MMD}_{\mathrm{CPDA}}
    (P_s^c,P_t^c).
\end{equation}
\end{proof}

\begin{corollary}[Equal class priors]
\label{cor:equal_priors}
If $\pi_s(c)=\pi_t(c)$ for all $c$, then the reweighted source risk reduces to the standard source risk: $R_s^{\pi_t}(h)=R_s(h)$.
Consequently,
\begin{equation}
    R_t(h)
    \leq
    R_s(h)
    +
    B
    \sum_{c=1}^{K}
    \pi_t(c)
    \mathrm{MMD}_{\mathrm{CPDA}}
    (P_s^c,P_t^c).
\end{equation}
\end{corollary}

\begin{proof}
If $\pi_s(c)=\pi_t(c)$, then
\begin{equation}
\begin{split}
    R_s(h)
    =
    &\sum_{c=1}^{K}
    \pi_s(c)
    \mathbb{E}_{X\sim P_s^c}
    \ell\big(h(X), c \big)\\
    =
    &\sum_{c=1}^{K}
    \pi_t(c)
    \mathbb{E}_{X\sim P_s^c}
    \ell\big(h(X), c \big)
    =
    R_s^{\pi_t}(h).
\end{split}
\end{equation}
The result follows directly from Theorem~\ref{thm:target_risk_bound}.
\end{proof}

\begin{theorem}[Effect of pseudo-label error]
\label{thm:pseudo_label_error}
Let $q(c|x)$ be the soft target pseudo-label distribution used by CPDA, and let
\begin{equation}
    \eta_q
    =
    \mathbb{E}_{X\sim P_t}
    \big[
        \|q(\cdot|X)-p_t(\cdot|X)\|_1
    \big],
\end{equation}
where $p_t(c|x)$ is the true target posterior.
Assume the loss is bounded: $0\leq \ell\big(h(X), c \big)\leq M$. Then replacing the true target class posterior by $q(c|x)$ changes the target risk by at most $M\eta_q$.
\end{theorem}

\begin{proof}
The true conditional target risk can be written as
\begin{equation}
    R_t(h)
    =
    \mathbb{E}_{X\sim P_t}
    \sum_{c=1}^{K}
    p_t(c|X)
    \ell\big(h(X), c \big).
\end{equation}
The pseudo-label-weighted target risk is
\begin{equation}
    R_t^q(h)
    =
    \mathbb{E}_{X\sim P_t}
    \sum_{c=1}^{K}
    q(c|X)
    \ell\big(h(X), c \big).
\end{equation}
Therefore,
\begin{equation}
\begin{split}
    |R_t&(h)-R_t^q(h)|
    =\\
    &=\Big|
    \mathbb{E}_{X\sim P_t}
    \sum_{c=1}^{K}
    \big(
        p_t(c|X)-q(c|X)
    \big)
    \ell\big(h(X), c \big)
    \Big|
    \\
    &\leq
    \mathbb{E}_{X\sim P_t}
    \sum_{c=1}^{K}
    \big|
        p_t(c|X)-q(c|X)
    \big|
    \big|
        \ell\big(h(X), c \big)
    \big|
    \\
    &\leq
    M
    \mathbb{E}_{X\sim P_t}
    \|p_t(\cdot|X)-q(\cdot|X)\|_1
    =
    M\eta_q.
\end{split}
\end{equation}
\end{proof}

\begin{proposition}[Connection to existing discrepancy losses]
\label{prop:connections}
CPDA contains several commonly used discrepancy principles as limiting or restricted cases.

\begin{enumerate}
    \item If class conditioning is removed and only $k_{\mathrm{pool}}$ is used, CPDA reduces to standard marginal MMD on pooled features.
    \item If a linear kernel is used on pooled features, CPDA reduces to first-order mean matching.
    \item If a linear kernel is used on centered quadratic feature maps, the induced MMD equals covariance matching, corresponding to CORAL-type alignment.
    \item If polynomial feature maps of order $p$ are used, CPDA aligns moments up to order $p$, recovering the principle behind higher-order moment matching such as HoMM.
    \item If CPDA is evaluated at several network layers, it recovers the multi-layer distribution-matching principle used by DAN-like methods, while retaining class conditioning and temporal structure.
\end{enumerate}
\end{proposition}

\begin{proof}
Each statement follows by restricting the feature map and conditioning structure of CPDA. For item 1, setting
\begin{equation}
    \alpha_{\mathrm{sig}}
    =
    \alpha_{\mathrm{spec}}
    =
    \alpha_{\mathrm{path}}
    =
    0,
    \qquad
    \alpha_{\mathrm{pool}}>0,
\end{equation}
and replacing class-conditional weights by uniform weights yields
$\big\|
        \widehat{\mu}_s
        -
        \widehat{\mu}_t
    \big\|_{\mathcal{H}}^2$,
which is the empirical MMD. For item 2, choosing a linear kernel
\begin{equation}
    k(x,x')=\psi_{\mathrm{pool}}(x)^\top \psi_{\mathrm{pool}}(x')
\end{equation}
gives
\begin{equation}
    \mathrm{MMD}^2
    =
    \big\|
        \mathbb{E}_{P_s}\psi_{\mathrm{pool}}(X)
        -
        \mathbb{E}_{P_t}\psi_{\mathrm{pool}}(X)
    \big\|_2^2,
\end{equation}
which is first-order mean matching. For item 3, define the centered quadratic feature map
\begin{equation}
    \psi_2(x)
    =
    \operatorname{vec}
    \big(
        (\psi(x)-\mu)(\psi(x)-\mu)^\top
    \big).
\end{equation}
A linear MMD in this feature space is
\begin{equation}
    \big\|
        \mathbb{E}_{P_s}\psi_2(X)
        -
        \mathbb{E}_{P_t}\psi_2(X)
    \big\|_2^2
    =
    \|C_s-C_t\|_F^2,
\end{equation}
which is the core covariance discrepancy used by CORAL. For item 4, polynomial feature maps explicitly encode tensor products of features up to order $p$.
Therefore, MMD in this feature space matches moments up to order $p$, which is precisely the principle underlying HoMM. For item 5, summing CPDA discrepancies over multiple layers gives a multi-layer discrepancy objective. This is structurally analogous to DAN-like multi-layer MMD, but CPDA additionally preserves class-conditional and temporal path information (refer to App.~\ref{app:special_cases}, for more details.).
\end{proof}

\section{Experiments}
\label{label_experiments}

\paragraph{Methods.} We evaluate CPDA using the AdaTime~\citep{ragab_eldele_tan} benchmarking framework to ensure a reproducible comparison across DA methods. We benchmark discrepancy-based baselines, including MSE alignment, cross-correlation (CC) and Pearson-correlation (PC) alignment, Kullback-Leibler divergence (KL), linear and kernelized MMD~\citep{borgwardt_gretton,long_zhu_mmd}, Jensen--Shannon divergence (JSD), linear and kernelized mean--covariance discrepancy variants (MMCD)~\citep{zhang_zhang_lan,alipour_tahmoresnezhad}, linear and squared DAN-style losses~\citep{long_cao_wang_DAN}, covariance-based objectives such as CORAL~\citep{sun_feng_saenko}, Jeffreys and Stein CORAL~\citep{cherian_sra,suvrit_sra,harandi_salzmann}, HoMM~\citep{chen_fu_chen}, MMDA~\citep{rahman_fookes}, DDC~\citep{tzeng_hoffman_zhang}, DANN~\citep{ganin_ustinova_ajakan}, CDAN~\citep{long_cao_wang}, DIRT-T~\citep{shu_bui_narui}, DSAN~\citep{zhu_zhuang_wang}, CoDATS~\citep{wilson_doppa_cook}, AdvSKM~\citep{liu_xue}, CoTMix~\citep{eldele_ragab_cotmix}, MCD~\citep{saito_watanabe}, AdaMatch~\citep{berthelot_roelofs}, DANCE~\citep{saito_kim_dance}, OVANet~\citep{saito_saenko}, and RAINCOAT~\citep{he_queen_koker}. We further evaluate Sinkhorn-based optimal transport alignment~\citep{ott_acmmm}. Refer to App.~\ref{app:comparison_methods}, for more details on methods. All methods are compared against CPDA under the same experimental protocol and are evaluated with three AdaTime backbone architectures: CNN, ResNet18, and TCN (App.~\ref{app:backbone}).

\paragraph{Datasets.} In total, we utilize 13 different time-series datasets for evaluation of unsupervised DA. Table~\ref{tab:datasets} in App.~\ref{app:datasets} gives a detailed overview of different applications and number of classes and samples. We utilize the EEG~\citep{goldberger_amaral_glass}, HHAR\_SA~\citep{stisen_blunck_bhattacharya}, UCI HAR~\citep{anguita_ghio_oneto}, WISDM~\citep{kwapisz_weiss_moore}, and uWave~\citep{liu_wang_zhong} datasets that are mainly used in DA benchmarks. Additionally, we evaluate datasets from the UCR/UEA~\citep{blankertz_curio_mueller,caputo_prebianca,villar_vergara_menendez,olivetti_kia_avesani,alimoglu_alpaydin} dataset suite. One prominent application is online handwriting (OnHW) recognition of symbols, equations, and characters~\citep{ott_ijdar,ott_imwut}. For ablation studies, we consider the transformation of sine-signals to cosine-signals as a DA problem~\citep{ott_acmmm}.
\section{Evaluation}
\label{label_evaluation}

We evaluate a large collection of source--target pairs (up to $435$ pairs per dataset), repeat every transfer five times using different random seeds, and report the mean and standard deviation of accuracy and F1-score. Details on the hardware and training setup are provided in App.~\ref{app:hardware_setup}, while additional experimental results are reported in App.~\ref{app:experimental_results}. The source--target domain discrepancies are analyzed in App.~\ref{app:experimental_results_discrepancy}.

\begin{figure}[!t]
    \centering
    \includegraphics[trim=6 6 6 6, clip, width=1.0\linewidth]{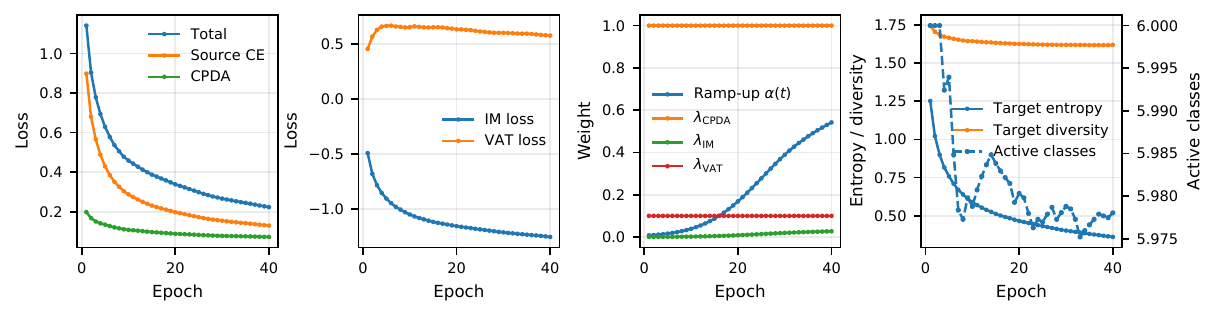}
    \caption{Training dynamics of CPDA: losses, target regularization, effective weights, and target prediction statistics.}
    \label{label_cpda_training_dynamics}
\end{figure}

\paragraph{Training Dynamics.} Figure~\ref{label_cpda_training_dynamics} shows stable CPDA optimization: the total, source, and CPDA losses decrease steadily, while VAT remains active with fixed weight and only the IM weight is gradually increased by the ramp-up schedule. Target entropy decreases without a collapse in target diversity or active classes, indicating more confident target predictions while preserving class coverage.

\begin{table*}[t!]
\centering
\caption{Results for all methods on the exemplary datasets HHAR, EEG, and uWave for CNN, ResNet18, and TCN backbones. Each entry reports accuracy in \% and F1-score. Results for all datasets are given in App.~\ref{app:experimental_results_da_methods}, including standard deviations.}
\label{table_all_results_da}
\scriptsize
\begin{tabular}{
    >{\raggedright\arraybackslash}p{1.45cm}
    >{\centering\arraybackslash}p{1.35cm}
    >{\centering\arraybackslash}p{1.35cm}
    >{\centering\arraybackslash}p{1.35cm}
    >{\centering\arraybackslash}p{1.35cm}
    >{\centering\arraybackslash}p{1.35cm}
    >{\centering\arraybackslash}p{1.35cm}
    >{\centering\arraybackslash}p{1.35cm}
    >{\centering\arraybackslash}p{1.35cm}
    >{\centering\arraybackslash}p{1.35cm}
}
\toprule
\textbf{Method}
& \makecell{\textbf{HHAR}\\\textbf{CNN}}
& \makecell{\textbf{HHAR}\\\textbf{ResNet}}
& \makecell{\textbf{HHAR}\\\textbf{TCN}}
& \makecell{\textbf{EEG}\\\textbf{CNN}}
& \makecell{\textbf{EEG}\\\textbf{ResNet}}
& \makecell{\textbf{EEG}\\\textbf{TCN}}
& \makecell{\textbf{uWave}\\\textbf{CNN}}
& \makecell{\textbf{uWave}\\\textbf{ResNet}}
& \makecell{\textbf{uWave}\\\textbf{TCN}} \\
\midrule
MSE & 62.24, 57.79 & 48.40, 43.37 & 52.08, 46.31 & 69.12, 58.02 & 51.63, 38.55 & 45.87, 33.59 & 83.76, 82.43 & 45.88, 41.66 & 71.69, 68.82 \\
CS & 63.07, 59.09 & 53.68, 50.98 & 56.09, 51.36 & 71.51, 57.79 & 53.13, 36.18 & 46.00, 33.89 & 77.88, 76.09 & 48.51, 44.81 & 69.08, 66.29 \\
PC & 64.74, 60.68 & 54.97, 51.59 & 61.43, 57.20 & 69.80, 59.03 & 51.86, 38.80 & 45.89, 34.15 & 85.62, 84.62 & 49.99, 45.99 & 73.33, 70.22 \\
KL & 60.10, 54.76 & 50.83, 46.08 & 19.36,\,\,\, 8.39 & 69.20, 58.38 & 50.03, 37.36 & 26.56, \,\,\,8.21 & 81.35, 79.89 & 46.41, 40.81 & 12.69,\,\,\, 2.81 \\
JSD & 64.75, 60.62 & 54.55, 50.74 & 17.79,\,\,\, 7.33 & 69.99, 59.30 & 48.80, 35.65 & 26.47, \,\,\,8.21 & 84.10, 82.72 & 49.33, 43.11 & 12.10,\,\,\, 2.70 \\
lMMD & 66.83, 63.13 & 57.50, 54.95 & 68.33, 65.01 & 70.18, 59.39 & 51.74, 38.73 & 46.16, 33.78 & 86.53, 85.58 & 49.79, 45.79 & 83.40, 81.69 \\
kMMD & 63.86, 59.64 & 54.50, 51.25 & 61.48, 57.48 & 69.80, 59.11 & 51.91, 38.96 & 45.98, 34.00 & 86.57, 85.53 & 49.76, 45.78 & 79.22, 77.28 \\
Deep CORAL & 73.51, 71.57 & 61.20, 59.75 & 72.17, 70.23 & 72.44, 61.45 & 52.67, 39.21 & 46.79, 33.54 & 86.51, 85.44 & 50.11, 46.16 & 80.15, 78.30 \\
Jeff CORAL & 64.49, 60.18 & 54.55, 61.11 & 62.13, 58.15 & 69.69, 58.93 & 51.77, 38.81 & 45.93, 34.01 & 85.98, 84.88 & 49.75, 45.61 & 79.44, 77.50 \\
Stein CORAL & 64.10, 59.78 & 54.93, 51.57 & 62.11, 58.18 & 69.95, 59.24 & 51.87, 38.90 & 45.96, 34.08 & 86.41, 85.41 & 49.81, 45.71 & 79.11, 77.15 \\
lMMCD & 66.60, 63.03 & 57.32, 54.87 & 68.04, 64.78 & 70.19, 59.35 & 51.82, 38.77 & 45.85, 33.55 & 86.39, 85.26 & 49.75, 45.73 & 83.68, 81.84 \\
kMMCD & 64.37, 60.11 & 55.26, 52.06 & 62.03, 57.97 & 69.83, 59.04 & 51.74, 38.80 & 45.92, 33.88 & 86.51, 85.44 & 50.11, 46.16 & 80.15, 78.30 \\
MMDA & 70.34, 66.08 & 57.73, 55.05 & 67.26, 64.51 & 64.84, 47.83 & 52.15, 36.35 & 44.48, 28.31 & 92.26, 91.76 & 49.90, 46.24 & 84.64, 82.61 \\
DDC & 63.14, 59.32 & 53.74, 49.80 & 59.70, 55.69 & 72.44, 61.44 & 51.32, 37.98 & 46.77, 33.45 & 86.32, 85.28 & 49.96, 45.96 & 79.43, 77.42 \\
Linear DAN & 72.14, 69.81 & 59.19, 56.96 & 68.40, 65.86 & 70.25, 59.70 & 51.78, 38.82 & 45.98, 34.01 & 87.77, 87.03 & 50.58, 46.59 & 83.46, 82.02 \\
Squared DAN & 72.09, 69.84 & 59.29, 57.09 & 68.48, 66.12 & 70.50, 59.79 & 52.02, 39.02 & 45.96, 34.10 & 89.53, 88.83 & 51.03, 47.03 & 86.84, 85.98 \\
$\text{HoMM}_{p=3}$ & 72.64, 70.15 & 61.59, 59.70 & 70.84, 68.48 & 72.84, 61.74 & 51.16, 37.00 & 40.42, 19.14 & 92.00, 91.27 & 51.28, 47.62 & 90.36, 89.90 \\
DANN & 76.91, 75.86 & 63.21, 62.18 & 73.16, 71.86 & 73.09, 62.00 & 52.54, 38.54 & 46.30, 33.00 & 85.13, 84.13 & 52.01, 48.26 & 84.58, 83.37 \\
CDAN & 79.88, 79.44 & 66.49, 66.11 & 76.78, 75.10 & 71.70, 58.10 & 49.90, 32.28 & \textbf{47.88}, 33.47 & 96.43, 95.97 & 57.95, 53.18 & 81.85, 79.01 \\
DIRT-T & 80.11, 79.45 & 67.89, 65.92 & 76.78, 74.77 & 73.48, 61.23 & 52.37, 37.38 & 47.85, 33.30 & 96.43, 95.97 & 55.01, 49.97 & 80.15, 78.30 \\
DSAN & 77.96, 77.28 & 67.22, 66.60 & 74.64, 73.26 & 71.00, 58.15 & 51.53, 36.27 & 47.43, \textbf{34.71} & 13.11, \,\,\,5.45 & 59.25, 56.85 & 91.92, \textbf{91.55} \\
CoDATS & 76.30, 75.62 & 63.79, 62.73 & 72.48, 71.10 & 70.03, 58.34 & 50.67, 35.63 & 42.39, 26.78 & 88.85, 88.32 & 55.48, 52.19 & 89.31, 88.46 \\
AdvSKM & 66.36, 62.23 & 53.61, 50.28 & 62.32, 58.07 & 72.56, 61.37 & 50.97, 35.32 & 46.62, 33.35 & 85.54, 84.61 & 50.07, 46.18 & 79.74, 77.68 \\
Sinkhorn & 75.15, 73.97 & \textbf{67.98}, \textbf{67.18} & 70.04, 68.11 & 71.37, 59.97 & 51.31, 37.15 & 44.87, 28.48 & 92.89, 91.41 & 58.01, 54.21 & 84.08, 83.01 \\
CoTMix & 76.61, 74.51 & 56.84, 54.60 & 67.32, 65.36 & 70.30, 58.55 & 49.25, 34.55 & 32.22, 22.37 & 91.88, 91.04 & 42.76, 41.62 & 36.46, 34.88 \\
MCD & 73.92, 71.13 & 61.14, 58.43 & 67.12, 63.71 & 72.47, 61.49 & 51.75, 37.41 & 42.98, 31.69 & 89.43, 88.60 & 53.10, 49.66 & 79.34, 77.13 \\
AdaMatch & 76.83, 75.40 & 64.07, 62.60 & 77.43, \textbf{75.83} & 72.79, 61.19 & 50.97, 37.28 & 51.01, 37.38 & 89.31, 88.81 & 53.74, 50.29 & 86.36, 85.52 \\
DANCE & 55.76, 51.05 & 59.18, 55.68 & 63.75, 61.30 & 53.27, 33.42 & 51.28, 31.98 & 41.62, 22.13 & 38.89, 33.90 & 46.47, 42.83 & 82.20, 79.64 \\
OVANet & 65.78, 61.72 & 56.76, 54.08 & 70.20, 66.82 & 65.86, 50.77 & \textbf{53.56}, 37.93 & 43.25, 25.12 & 60.21, 57.86 & 49.72, 46.29 & 83.12, 81.14 \\
RAINCOAT & 72.91, 70.65 & 61.94, 60.13 & 64.58, 61.30 & 63.51, 45.99 & 47.77, 24.73 & 45.82, 20.77 & 88.17, 87.39 & \textbf{63.30}, \textbf{59.71} & 82.39, 80.34 \\
CPDA (ours) & \underline{\textbf{81.82}}, \underline{\textbf{80.27}} & 67.79, 65.75 & \textbf{77.88}, \textbf{75.84} & \underline{\textbf{74.39}}, \underline{\textbf{64.16}} & 53.27, \textbf{40.46} & 46.76, 33.80 & \underline{\textbf{97.30}}, \underline{\textbf{96.79}} & 61.32, 57.53 & \textbf{92.17}, 91.24 \\
\bottomrule
\end{tabular}
\end{table*}

\paragraph{Method Evaluation.} Table~\ref{table_all_results_da} summarizes representative results for HHAR, EEG, and uWave, while App.~\ref{app:experimental_results_da_methods} confirms consistent trends across the remaining datasets. Across the nine representative dataset--backbone configurations, CPDA achieves the highest accuracy/F1-score in five settings, while being competitive with the best baseline in the remaining cases. Its strongest and most consistent gains occur with the CNN backbone and on HHAR and uWave, where class-dependent temporal structure is particularly pronounced. On HHAR, CPDA reaches $81.82\%$ accuracy and $80.27\%$ F1 with CNN and $77.88\%$\,/\,$75.84\%$ with TCN, while obtaining the strongest ResNet18 result. It likewise outperforms all baselines on EEG, despite the substantially lower absolute performance and greater difficulty of this physiological-signal benchmark. The clearest gains are observed on uWave, where CPDA achieves $97.30\%$, $61.32\%$, and $92.17\%$ accuracy with CNN, ResNet18, and TCN, respectively, surpassing adversarial, discrepancy-based, optimal-transport, and time-series-specific competitors. These results support the central motivation of CPDA: aligning class-conditional latent paths is particularly effective when discriminative information is encoded in temporal evolution rather than only in marginal feature statistics. The closest competitors vary across settings, with DIRT-T and CDAN performing strongly on HHAR, AdaMatch and OVANet on EEG, and the recent time-series-specific RAINCOAT method on uWave with ResNet18 and under sinusoidal noise. The appendix further demonstrates strong performance on UCI HAR, WISDM, PenDigits, Epilepsy, and Face Detection, although the gains are more variable for small or highly heterogeneous datasets such as Finger Movements, Gestures Mid Air, and OnHW. On the sinusoidal benchmark, CPDA is among the most robust methods under increasing noise, providing particularly strong performance in the moderate-noise
regime (Figure~\ref{figure_results_sin_cos}, App.~\ref{app:experimental_results_da_methods}).

\begin{figure}[!t]
    \centering
    \includegraphics[trim=10 10 10 10, clip, width=1.0\linewidth]{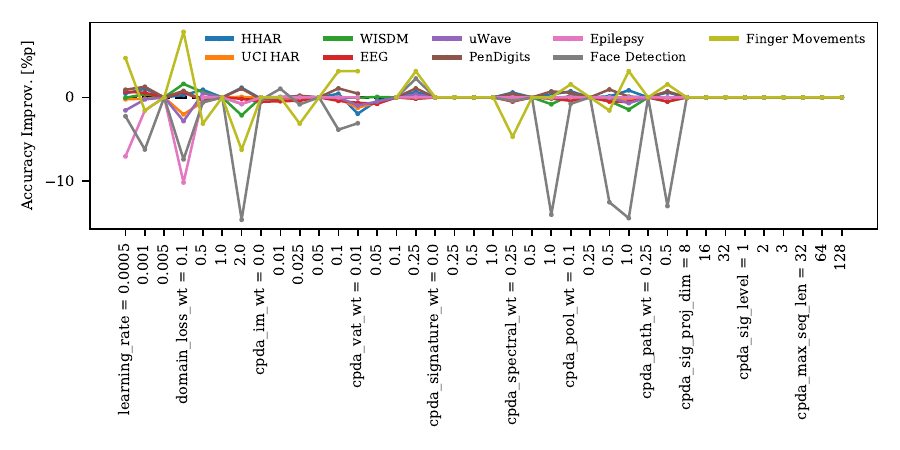}
    \caption{Accuracy improvements in percent points through hyperparameter searches against default parameters.}
    \label{label_hyperparameter_search}
\end{figure}

\paragraph{Hyperparameter Searches.} We perform parameter searches for all datasets using the CNN backbone (see Table~\ref{tab:cpda_hparams} in App.~\ref{app:hardware_setup}). Figure~\ref{label_hyperparameter_search} reports the accuracy improvement over the default CPDA configuration of a post hoc sensitivity analysis. Overall, most hyperparameter changes lead only to small improvements, indicating that the default setting is already robust across datasets. The largest gains are dataset-specific: Finger Movements benefits strongly from selected loss weights, while Face Detection shows larger sensitivity to the pooled\,/\,path-related weights. In contrast, HHAR, UCI HAR, WISDM, EEG, uWave, PenDigits, and Epilepsy remain comparatively stable. This suggests that CPDA is not overly sensitive to most hyperparameters, but that difficult or small datasets can benefit from targeted tuning of the loss weights and kernel-component weights.
\section{Conclusion}
\label{label_conclusion}

We proposed CPDA, a non-adversarial method for unsupervised time-series DA that aligns class-conditional latent path distributions instead of only marginal feature distributions. CPDA combines source labels, target soft pseudo-labels, VAT regularization, and a composite signature--spectral kernel to capture pooled, temporal, spectral, and path-signature information. Experiments across multiple backbones show that CPDA achieves strong or best performance on several structured datasets. These results demonstrate that class-conditional path alignment is particularly effective when discriminative information is encoded in temporal trajectories.

\bibliography{aaai2026}
\clearpage
\appendix

\section{Notations}
\label{app:notations}

We use uppercase letters to denote random variables, lowercase letters for their realizations, bold lowercase letters for vectors, and calligraphic letters for sets and function spaces. Source-domain quantities are marked by the superscript $s$, whereas target-domain quantities are marked by the superscript $t$. The source domain provides labeled samples, while the target domain provides only unlabeled samples. The learned model consists of a feature extractor $f_\theta$ and a classifier $g_\phi$. For time-series inputs, the feature extractor produces a latent temporal path $Z_\theta(x)=(z_1,\dots,z_T)$, where each $z_t$ is a latent feature vector at time index $t$. CPDA aligns source and target class-conditional latent path distributions using a composite signature--spectral kernel and soft target pseudo-labels. Table~\ref{tab:notation1} and Table~\ref{tab:notation2} gives an overview of all notations.

\begin{table}[h]
\setlength{\tabcolsep}{1mm}
\centering
\caption{Summary of notations used throughout the paper.}
\label{tab:notation1}
\small
\begin{tabular}{p{0.18\linewidth} p{0.77\linewidth}}
\toprule
\textbf{Notation} & \textbf{Description} \\
\midrule
\multicolumn{2}{l}{\textbf{Domains and Data}} \\
\midrule
$\mathcal{X}$ & Input space of multivariate time-series \\
$\mathcal{Y}$ & Label space, $\mathcal{Y}=\{1,\dots,K\}$ \\
$K$ & Number of classes \\
$C_{\mathrm{in}}$ & Number of input channels \\
$L$ & Length of the original input time-series \\
$P_s(X,Y)$ & Joint source-domain distribution \\
$P_t(X,Y)$ & Joint target-domain distribution \\
$P_s(X)$, $P_t(X)$ & Source and target marginal input distributions \\
$P_s^c$ & Source class-conditional distr. $P_s(X|Y=c)$ \\
$P_t^c$ & Target class-conditional distr. $P_t(X|Y=c)$ \\
$\mathcal{D}_s$ & Labeled source dataset $\{(x_i^s,y_i^s)\}_{i=1}^{n_s}$ \\
$\mathcal{D}_t$ & Unlabeled target dataset $\{x_j^t\}_{j=1}^{n_t}$ \\
$n_s$, $n_t$ & Number of source and target samples \\
$B_s$, $B_t$ & Source and target mini-batch sizes \\
$x_i^s$, $x_j^t$ & Source and target time-series samples \\
$y_i^s$ & Source label for sample $x_i^s$ \\
\midrule
\multicolumn{2}{l}{\textbf{Model and Latent Representations}} \\
\midrule
$f_\theta$ & Feature extractor with parameters $\theta$ \\
$g_\phi$ & Classifier with parameters $\phi$ \\
$h_{\theta,\phi}$ & Full prediction model $h_{\theta,\phi}=g_\phi\circ f_\theta$ \\
$Z_\theta(x)$ & Latent temporal path extracted from input $x$ \\
$Z_\theta(x)=(z_1,\dots,z_T)$ & Sequence of latent vectors \\
$z_t$ & Latent feature vector at time index $t$ \\
$T$ & Latent temporal length after feature extraction \\
$d$ & Dimension of each latent vector $z_t$ \\
$\Delta z_t$ & Latent increment $z_t-z_{t-1}$ \\
$A_\theta(x)_t$ & Augmented latent path element $[t/T,z_t,\Delta z_t]$ \\
$q_{\theta,\phi}(c\mid x)$ & Softmax probability assigned to class $c$ \\
\midrule
\multicolumn{2}{l}{\textbf{CPDA Features and Kernels}} \\
\midrule
$\psi_{\mathrm{pool}}(x)$ & Pooled latent feature representation \\
$\psi_{\mathrm{path}}(x)$ & Flattened augmented latent path representation \\
$\psi_{\mathrm{spec}}(x)$ & Spectral representation of the latent path \\
$\psi_{\mathrm{sig}}(x)$ & Low-rank truncated path-signature representation \\
$\mathcal{F}_t[\cdot]$ & Fourier transform along the temporal dimension \\
$S_1(x)$ & First-order signature feature \\
\bottomrule
\end{tabular}
\end{table}

\begin{table}[t!]
\setlength{\tabcolsep}{1mm}
\centering
\caption{Table~\ref{tab:notation1} continued.}
\label{tab:notation2}
\small\begin{tabular}{p{0.13\linewidth} p{0.76\linewidth}}
\toprule
\textbf{Notation} & \textbf{Description} \\
\midrule
\multicolumn{2}{l}{\textbf{CPDA Features and Kernels}} \\
\midrule
$S_2(x)$ & Second-order signature feature \\
$r$ & Random projection dimension for signature features \\
$k_{\mathrm{pool}}$ & Kernel on pooled latent features \\
$k_{\mathrm{path}}$ & Kernel on flattened temporal path features \\
$k_{\mathrm{spec}}$ & Kernel on spectral latent features \\
$k_{\mathrm{sig}}$ & Kernel on signature features \\
$k_{\mathrm{CPDA}}$ & Composite CPDA kernel \\
$\mathcal{H}_{\mathrm{CPDA}}$ & RKHS induced by $k_{\mathrm{CPDA}}$ \\
$\alpha_{\mathrm{pool}}$ & Weight of the pooled kernel component \\
$\alpha_{\mathrm{path}}$ & Weight of the temporal-path kernel component \\
$\alpha_{\mathrm{spec}}$ & Weight of the spectral kernel component \\
$\alpha_{\mathrm{sig}}$ & Weight of the signature kernel component \\
\midrule
\multicolumn{2}{l}{\textbf{Class-Conditional Alignment}} \\
\midrule
$w_{i,c}^s$ & Source class weight $\mathbf{1}[y_i^s=c]$ \\
$w_{j,c}^t$ & Target soft class weight $q_{\theta,\phi}(c\mid x_j^t)$ \\
$\bar{w}_{i,c}^s$ & Normalized source class weight \\
$\bar{w}_{j,c}^t$ & Normalized target class weight \\
$\pi_t(c)$ & True target class prior $P_t(Y=c)$ \\
$\widehat{\pi}_t(c)$ & Estimated target class prior from soft pseudo-labels \\
$\mu_{s,c}$ & Source class-conditional kernel mean embedding \\
$\mu_{t,c}$ & Target class-conditional kernel mean embedding \\
$\widehat{\mu}_{s,c}$ & Empirical source class-conditional mean embedding \\
$\widehat{\mu}_{t,c}$ & Empirical target class-conditional mean embedding \\
$\widehat{\mathrm{MMD}}_c^2$ & Empirical class-wise weighted MMD discrepancy \\
$\widehat{\mathcal{D}}_{\mathrm{CPDA}}^2$ & Empirical CPDA discrepancy \\
\midrule
\multicolumn{2}{l}{\textbf{Losses and Risks}} \\
\midrule
$\ell(\cdot,\cdot)$ & Supervised classification loss \\
$\mathcal{L}_{\mathrm{cls}}$ & Source classification loss \\
$\mathcal{L}_{\mathrm{IM}}$ & Target information maximization loss \\
$\mathcal{L}_{\mathrm{CPDA}}$ & Total CPDA training objective \\
$R_s(h)$ & Source risk of classifier $h$ \\
$R_t(h)$ & Target risk of classifier $h$ \\
$R_s^{\pi_t}(h)$ & Source risk reweighted by target class priors \\
$\lambda$ & Weight of the CPDA discrepancy term \\
$\eta$ & Weight of the information maximization term \\
$\gamma$ & VAT weight, e.g., $\gamma = 1.0$ \\
\bottomrule
\end{tabular}
\end{table}

The source and target domains are described by distributions $P_s$ and $P_t$, respectively. The source dataset $\mathcal{D}_s$ contains labeled time-series, whereas the target dataset $\mathcal{D}_t$ contains unlabeled time-series. The model is decomposed into a feature extractor $f_\theta$ and a classifier $g_\phi$. The feature extractor maps each input time-series to a latent temporal path $Z_\theta(x)$, which forms the object aligned by CPDA. Instead of matching only marginal feature distributions, CPDA constructs class-conditional kernel mean embeddings $\mu_{s,c}$ and $\mu_{t,c}$ for each class $c$ and minimizes their discrepancy in the RKHS induced by the composite kernel $k_{\mathrm{CPDA}}$. Since target labels are unavailable, target class weights are obtained from soft pseudo-labels $q_{\theta,\phi}(c\mid x)$.

\section{Additional Theoretical Details}
\label{app:theory}

\subsection{Assumptions}
\label{app:assumptions}

We state the assumptions used in the theoretical analysis.

\begin{assumption}[Bounded loss]
\label{ass:bounded_loss}
The task loss satisfies
\begin{equation}
    0 \leq \ell \big(h(x), y \big) \leq M
\end{equation}
for all $x\in\mathcal{X}$ and $y\in\mathcal{Y}$.
\end{assumption}

\begin{assumption}[RKHS regularity]
\label{ass:rkhs_regular}
For every class $c\in\{1,\dots,K\}$, the function
\begin{equation}
    x\mapsto \ell \big(h(x), c \big)
\end{equation}
belongs to the RKHS $\mathcal{H}_{\mathrm{CPDA}}$ induced by the CPDA kernel and satisfies
\begin{equation}
    \|\ell(h(\cdot),c)\|_{\mathcal{H}_{\mathrm{CPDA}}} \leq B.
\end{equation}
\end{assumption}

\begin{assumption}[Bounded kernel]
\label{ass:bounded_kernel}
The CPDA kernel is bounded:
\begin{equation}
    0 \leq k_{\mathrm{CPDA}}(x,x') \leq \kappa
\end{equation}
for all $x,x'\in\mathcal{X}$.
\end{assumption}

\begin{assumption}[Soft pseudo-label consistency]
\label{ass:pseudo_label}
The target pseudo-label distribution $q(c|x)$ approximates the true target posterior $p_t(c|x)$ with error
\begin{equation}
    \eta_q
    =
    \mathbb{E}_{X\sim P_t}
    \big[
        \|q(\cdot|X)-p_t(\cdot|X)\|_1
    \big].
\end{equation}
\end{assumption}

\subsection{Consistency of the Empirical CPDA Discrepancy}
\label{app:consistency}

We first analyze the empirical class-conditional CPDA discrepancy under known target class labels.
The pseudo-label case is discussed separately. For class $c$, define
\begin{equation}
    P_s^c = P_s(X|Y=c),
    \quad
    P_t^c = P_t(X|Y=c).
\end{equation}
Let
\begin{equation}
    \mu_s^c
    =
    \mathbb{E}_{X\sim P_s^c}
    k_{\mathrm{CPDA}}(X,\cdot),
    \quad
    \mu_t^c
    =
    \mathbb{E}_{X\sim P_t^c}
    k_{\mathrm{CPDA}}(X,\cdot).
\end{equation}
The population class-wise discrepancy is
\begin{equation}
    \mathrm{MMD}_{\mathrm{CPDA}}^2(P_s^c,P_t^c)
    =
    \|\mu_s^c-\mu_t^c\|_{\mathcal{H}_{\mathrm{CPDA}}}^2.
\end{equation}

\begin{theorem}[Consistency of the empirical class-wise estimator]
\label{thm:empirical_consistency}
Assume that $k_{\mathrm{CPDA}}$ is bounded as in Assumption~\ref{ass:bounded_kernel}.
Let
\begin{equation}
    \widehat{\mathrm{MMD}}_{\mathrm{CPDA}}^2(P_s^c,P_t^c)
\end{equation}
be the empirical biased MMD estimator computed from $n_s^c$ source samples and $n_t^c$ target samples of class $c$.
Then
\begin{equation}
    \widehat{\mathrm{MMD}}_{\mathrm{CPDA}}^2(P_s^c,P_t^c)
    \xrightarrow[]{p}
    \mathrm{MMD}_{\mathrm{CPDA}}^2(P_s^c,P_t^c)
\end{equation}
as $n_s^c,n_t^c \rightarrow \infty$.
\end{theorem}

\begin{proof}
The empirical biased MMD estimator can be written as
\begin{equation}
\begin{aligned}
    \widehat{\mathrm{MMD}}^2
    &=
    \frac{1}{(n_s^c)^2}
    \sum_{i,i'=1}^{n_s^c}
    k_{\mathrm{CPDA}}(x_i^s,x_{i'}^s)
    \\
    &+
    \frac{1}{(n_t^c)^2}
    \sum_{j,j'=1}^{n_t^c}
    k_{\mathrm{CPDA}}(x_j^t,x_{j'}^t)
    \\
    &-
    \frac{2}{n_s^c n_t^c}
    \sum_{i=1}^{n_s^c}
    \sum_{j=1}^{n_t^c}
    k_{\mathrm{CPDA}}(x_i^s,x_j^t).
\end{aligned}
\end{equation}
Since the kernel is bounded, all three terms are bounded $V$-statistics.
By the law of large numbers for bounded $V$-statistics, the three empirical terms converge in probability to their corresponding population expectations:
\begin{equation}
    \mathbb{E}_{X,X'\sim P_s^c} k_{\mathrm{CPDA}}(X,X'),
\end{equation}
\begin{equation}
    \mathbb{E}_{Y,Y'\sim P_t^c} k_{\mathrm{CPDA}}(Y,Y'),
\end{equation}
and
\begin{equation}
    \mathbb{E}_{X\sim P_s^c,Y\sim P_t^c} k_{\mathrm{CPDA}}(X,Y).
\end{equation}
Therefore,
\begin{equation}
    \widehat{\mathrm{MMD}}_{\mathrm{CPDA}}^2(P_s^c,P_t^c)
    \xrightarrow[]{p}
    \mathrm{MMD}_{\mathrm{CPDA}}^2(P_s^c,P_t^c).
\end{equation}
\end{proof}

\subsection{Finite-Sample Concentration}
\label{app:finite_sample}

\begin{theorem}[Finite-sample deviation of empirical CPDA]
\label{thm:finite_sample_cpda}
Assume that
\begin{equation}
    0 \leq k_{\mathrm{CPDA}}(x,x') \leq \kappa
\end{equation}
for all $x,x'$.
For a fixed class $c$, with probability at least $1-\delta$,
\begin{equation}
\begin{aligned}
    \big|
    \widehat{\mathrm{MMD}}_{\mathrm{CPDA}}&(P_s^c,P_t^c)
    -
    \mathrm{MMD}_{\mathrm{CPDA}}(P_s^c,P_t^c)
    \big|\\
    &\leq
    2\sqrt{\frac{\kappa}{n_s^c}}
    +
    2\sqrt{\frac{\kappa}{n_t^c}}
    +
    \sqrt{
        2\kappa
        \Big(
            \frac{1}{n_s^c}
            +
            \frac{1}{n_t^c}
        \Big)
        \log\frac{1}{\delta}
    }.
\end{aligned}
\end{equation}
\end{theorem}

\begin{proof}
The result follows from standard concentration arguments for empirical kernel mean embeddings.
Since
\begin{equation}
    \mathrm{MMD}_{\mathrm{CPDA}}(P_s^c,P_t^c)
    =
    \|\mu_s^c-\mu_t^c\|_{\mathcal{H}_{\mathrm{CPDA}}},
\end{equation}
and
\begin{equation}
    \widehat{\mathrm{MMD}}_{\mathrm{CPDA}}(P_s^c,P_t^c)
    =
    \|\widehat{\mu}_s^c-\widehat{\mu}_t^c\|_{\mathcal{H}_{\mathrm{CPDA}}},
\end{equation}
the reverse triangle inequality gives
\begin{equation}
\begin{aligned}
    \big|
    \widehat{\mathrm{MMD}}_{\mathrm{CPDA}}&(P_s^c,P_t^c)
    -
    \mathrm{MMD}_{\mathrm{CPDA}}(P_s^c,P_t^c)
    \big|\\
    &\leq
    \|\widehat{\mu}_s^c-\mu_s^c\|_{\mathcal{H}_{\mathrm{CPDA}}}
    +
    \|\widehat{\mu}_t^c-\mu_t^c\|_{\mathcal{H}_{\mathrm{CPDA}}}.
\end{aligned}
\end{equation}
For bounded kernels, empirical kernel mean embeddings concentrate around their population means at rate $\mathcal{O}(1/\sqrt{n})$.
Applying the concentration bound separately to the source and target empirical embeddings and combining the two deviations yields the result.
\end{proof}

\subsection{Effect of Soft Pseudo-Label Error on CPDA}
\label{app:pseudo_label_cpda}

\begin{theorem}[Stability of CPDA under pseudo-label perturbations]
\label{thm:cpda_pseudo_stability}
Assume that
\begin{equation}
    0 \leq k_{\mathrm{CPDA}}(x,x') \leq \kappa.
\end{equation}
Let $p_t(c|x)$ denote the true target posterior and $q(c|x)$ the soft pseudo-label posterior used by CPDA.
If
\begin{equation}
    \eta_q
    =
    \mathbb{E}_{X\sim P_t}
    \|q(\cdot|X)-p_t(\cdot|X)\|_1,
\end{equation}
then the difference between the ideal class-conditional CPDA discrepancy and the pseudo-label-based discrepancy is bounded by
\begin{equation}
    \big|
    \mathcal{D}_{\mathrm{CPDA}}^2(p_t)
    -
    \mathcal{D}_{\mathrm{CPDA}}^2(q)
    \big|
    =
    \mathcal{O}(\kappa \eta_q).
\end{equation}
\end{theorem}

\begin{proof}
The pseudo-label distribution affects CPDA only through the target class weights and the estimated target class priors.
For any class $c$, the difference between the true and pseudo-label-weighted target embeddings is
\begin{equation}
\begin{aligned}
    \|&\mu_{t,c}^{p}-\mu_{t,c}^{q}\|_{\mathcal{H}_{\mathrm{CPDA}}}
    \\
    &=\big\|
    \mathbb{E}_{X\sim P_t}
    \big[
        p_t(c|X)-q(c|X)
    \big]
    k_{\mathrm{CPDA}}(X,\cdot)
    \big\|_{\mathcal{H}_{\mathrm{CPDA}}}\\
    &\leq
    \mathbb{E}_{X\sim P_t}
    \big[
        |p_t(c|X)-q(c|X)|
        \,
        \|k_{\mathrm{CPDA}}(X,\cdot)\|_{\mathcal{H}_{\mathrm{CPDA}}}
    \big].
\end{aligned}
\end{equation}
By the reproducing property,
\begin{equation}
    \|k_{\mathrm{CPDA}}(X,\cdot)\|_{\mathcal{H}_{\mathrm{CPDA}}}
    =
    \sqrt{k_{\mathrm{CPDA}}(X,X)}
    \leq
    \sqrt{\kappa}.
\end{equation}
Hence,
\begin{equation}
    \|\mu_{t,c}^{p}-\mu_{t,c}^{q}\|_{\mathcal{H}_{\mathrm{CPDA}}}
    \leq
    \sqrt{\kappa}
    \mathbb{E}_{X\sim P_t}
    \big|p_t(c|X)-q(c|X)\big|.
\end{equation}
Summing over classes gives a total perturbation bounded by $\sqrt{\kappa}\eta_q$ at the embedding level.
Since the squared discrepancy is Lipschitz on bounded RKHS balls, the perturbation of the squared CPDA discrepancy is of order $\mathcal{O}(\kappa\eta_q)$.
\end{proof}

\section{Computational Complexity}
\label{app:complexity}

Let $B$ be the mini-batch size, $T$ the latent temporal length, $d$ the latent channel dimension, $r$ the signature projection dimension, $K$ the number of classes, and $M$ the number of Gaussian bandwidths. The CPDA kernel is computed from four components:
\begin{enumerate}
\item \textbf{Pooled component.}
The pooled feature has dimension $d$, and the pairwise kernel computation costs
\begin{equation}
    \mathcal{O}(M B^2 d).
\end{equation}
\item \textbf{Path component.}
The augmented flattened path has dimension $\mathcal{O}(Td)$.
The pairwise kernel computation costs
\begin{equation}
    \mathcal{O}(M B^2 T d).
\end{equation}
\item \textbf{Spectral component.}
The FFT computation costs
\begin{equation}
    \mathcal{O}(B d T \log T),
\end{equation}
and the resulting pairwise kernel computation costs
\begin{equation}
    \mathcal{O}(M B^2 T d).
\end{equation}
\item \textbf{Signature component.}
Using a random projection of dimension $r$, the first-order signature costs
\begin{equation}
    \mathcal{O}(B T d r),
\end{equation}
and the second-order signature costs
\begin{equation}
    \mathcal{O}(B T r^2).
\end{equation}
The pairwise kernel computation then costs
\begin{equation}
    \mathcal{O}(M B^2 r^2).
\end{equation}
\end{enumerate}
Therefore, the total mini-batch complexity is
\begin{equation}
    \mathcal{O}
    \big(
        B d T \log T
        +
        B T d r
        +
        B T r^2
        +
        M B^2 (Td+r^2+d)
    \big).
\end{equation}
In practice, $T$ is the latent temporal length after pooling, and $r$ is chosen small, e.g., $r=16$, making CPDA comparable to standard mini-batch MMD while preserving temporal and spectral information.

\section{CPDA Algorithm}
\label{app:algorithm}

Algorithm~\ref{alg:cpda} describes one training step of CPDA. Given a labeled source mini-batch and an unlabeled target mini-batch, the feature extractor first maps both domains into latent temporal paths, from which the classifier produces source and target logits. Since target labels are unavailable, CPDA uses the target softmax outputs as soft pseudo-labels to estimate class membership probabilities. The latent paths are then transformed into complementary pooled, temporal-path, spectral, and signature-based representations, which are used to construct the composite source--source, target--target, and source--target kernel matrices. For each class, CPDA forms source weights from the true labels and target weights from the soft pseudo-labels, normalizes these weights, and computes a class-wise weighted MMD discrepancy. The final CPDA discrepancy is obtained by averaging the class-wise discrepancies according to the estimated target class priors. The model parameters are updated by minimizing the sum of the supervised source classification loss, the class-conditional CPDA discrepancy, and an optional target information-maximization term that encourages confident but non-collapsed target predictions.

\begin{algorithm}[!h]
\caption{Class-Conditional Path Distribution Alignment}
\label{alg:cpda}
\begin{algorithmic}[1]
\Require Source mini-batch $\{(x_i^s,y_i^s)\}_{i=1}^{B_s}$, target mini-batch $\{x_j^t\}_{j=1}^{B_t}$
\Require Feature extractor $f_\theta$, classifier $g_\phi$, weights $\lambda,\eta$
\State Compute latent source paths $Z_i^s=f_\theta(x_i^s)$
\State Compute latent target paths $Z_j^t=f_\theta(x_j^t)$
\State Compute source logits and target logits
\State Compute target soft pseudo-labels $q_j(c)=\operatorname{softmax}(g_\phi(Z_j^t))_c$
\State Construct pooled, path, spectral, and signature features
\State Compute composite kernel matrices $K_{ss}$, $K_{tt}$, and $K_{st}$
\For{$c=1,\ldots,K$}
    \State Compute source class weights $w_{i,c}^s=\mathds{1}[y_i^s=c]$
    \State Compute target soft weights $w_{j,c}^t=q_j(c)$
    \State Normalize weights within class $c$
    \State Compute class-wise weighted discrepancy $\widehat{\mathrm{MMD}}_c^2$
\EndFor
\State Estimate target priors $\widehat{\pi}_t(c)=B_t^{-1}\sum_{j=1}^{B_t}q_j(c)$
\State Compute CPDA discrepancy $\widehat{\mathcal{D}}_{\mathrm{CPDA}}^2 = \sum_{c=1}^{K} \widehat{\pi}_t(c) \widehat{\mathrm{MMD}}_c^2$
\State Compute source classification loss $\mathcal{L}_{\mathrm{cls}}$
\State Compute optional target information maximization loss $\mathcal{L}_{\mathrm{IM}}$
\State Compute VAT loss $\mathcal{L}_{\mathrm{VAT}}$ on source and target samples
\State Update $\theta,\phi$ by minimizing $\mathcal{L} = \mathcal{L}_{\mathrm{cls}} + \lambda \widehat{\mathcal{D}}_{\mathrm{CPDA}}^2 + \eta \mathcal{L}_{\mathrm{IM}} + \gamma \mathcal{L}_{\mathrm{VAT}}$
\end{algorithmic}
\end{algorithm}

\section{Special Cases of CPDA}
\label{app:special_cases}

\begin{table*}[!t]
\centering
\caption{Existing discrepancy losses as restricted cases of CPDA.}
\label{tab:cpda_special_cases}
\begin{tabular}{lll}
\toprule
\textbf{Restriction} & \textbf{Resulting Objective} & \textbf{Limitation Addressed by CPDA} \\
\midrule
No class conditioning, pooled RBF kernel only & MMD & Adds class-conditional alignment \\
No class conditioning, linear pooled kernel & Mean matching & Adds nonlinear path alignment \\
Centered quadratic feature map & CORAL & Adds higher-order \& temporal structure \\
Polynomial feature map of order $p$ & HoMM-like matching & Adds explicit temporal order \\
Multiple-layer pooled MMD & DAN & Adds class-conditional path kernels \\
Class-conditional pooled MMD only & LMMD/DSAN-like alignment & Adds signature and spectral structure \\
\bottomrule
\end{tabular}
\end{table*}

Table~\ref{tab:cpda_special_cases} summarizes how CPDA generalizes several established discrepancy-based DA objectives. When class conditioning is removed and only a pooled RBF kernel is used, CPDA reduces to standard marginal MMD; with a linear pooled kernel, it further simplifies to first-order mean matching. By choosing centered quadratic feature maps, the objective recovers the covariance-alignment principle of CORAL, while polynomial feature maps of order $p$ yield a HoMM-like higher-order moment-matching objective. Similarly, applying pooled MMD across multiple layers recovers the principle underlying DAN, and using only class-conditional pooled MMD corresponds to lMMD\,/\,DSAN-like alignment. CPDA extends all of these special cases by jointly incorporating class-conditional alignment, nonlinear kernel matching, temporal path structure, signature-based dynamics, and spectral information, thereby addressing the main limitations of marginal, vector-level, or low-order moment-matching approaches.

\paragraph{Compared to MMD.}
Standard MMD aligns the marginal feature distributions $P_s(Z)\approx P_t(Z)$.
This may incorrectly align samples from different classes when the class-conditional distributions are multimodal or when source and target class priors differ.
In contrast, CPDA aligns
\begin{equation}
    P_s(Z_{1:T}|Y=c)
    \approx
    P_t(Z_{1:T}|Y=c),
\end{equation}
which directly targets class-preserving transfer.

\paragraph{Compared to CORAL.}
CORAL aligns only second-order covariance statistics: $C_s\approx C_t$.
Although this is stable and computationally efficient, it ignores higher-order, nonlinear, spectral, and temporal-order information.
CPDA generalizes this idea by aligning distributions through an RKHS kernel over temporal, spectral, and signature-based path representations.

\paragraph{Compared to HoMM.}
HoMM matches higher-order moments of feature vectors.
However, standard HoMM still treats the representation as a vector and does not explicitly encode temporal ordering.
CPDA uses path signatures and explicit time augmentation, thereby aligning higher-order temporal interactions rather than only higher-order feature moments.

\paragraph{Compared to DAN.}
DAN applies MMD to multiple task-specific layers but typically performs marginal alignment.
CPDA can be viewed as a class-conditional and time-series-specific generalization: it aligns source and target distributions at the level of latent paths and can be extended to multiple layers if desired.

\paragraph{Compared to Adversarial Alignment.}
Adversarial DA learns domain-invariant representations by fooling a discriminator.
Although powerful, adversarial alignment can be unstable and may still suffer from class-mismatch under marginal alignment.
CPDA is a direct, non-adversarial discrepancy with an explicit class-conditional target-risk bound.

\section{Implementation Details of CPDA}
\label{app:implementation}

In all experiments, CPDA is applied to the latent representation produced by the final feature extractor layer.
For convolutional backbones, CPDA uses the pre-flattening feature map \(Z\in\mathbb R^{T\times d}\), where the first dimension retains the temporal ordering produced by the encoder. For vector-output backbones, CPDA reduces to a one-step path and therefore remains well-defined. The CPDA kernel uses four components: pooled, flattened path, spectral, and low-rank signature features. All components are $\ell_2$-normalized before the Gaussian kernels are evaluated. The Gaussian bandwidths are chosen using a median heuristic within each mini-batch and expanded into a multi-kernel set using a multiplicative factor. Unless otherwise stated, we use a second-order truncated signature with random projection dimension $r=16$. The random projection matrix is fixed during training and is not optimized. This keeps the signature representation computationally efficient while preserving path-dependent interactions. Target class weights are obtained from the softmax output of the classifier. By default, gradients are not propagated through these target weights when computing the CPDA discrepancy. This stabilizes training by preventing the model from reducing the discrepancy through degenerate pseudo-label manipulation. The theoretical results apply for a fixed feature extractor, projection, and set of kernel bandwidths. During optimization, these quantities are data-dependent; the analysis therefore characterizes each fixed training state rather than the complete non-stationary optimization process.

\section{Hardware \& Training Setup}
\label{app:hardware_setup}

\begin{table*}[t]
\centering
\caption{Main CPDA hyperparameters used in the experiments and their default value.}
\label{tab:cpda_hparams}
\begin{tabular}{lll p{0.18\linewidth}}
\toprule
\textbf{Hyperparameter} & \textbf{Symbol / Key} & \textbf{Default Value} & \textbf{Search} \\
\midrule
Learning rate & \texttt{learning\_rate} & $0.005$ & $[0.0001, 0.0005, 0.001$, $0.003, 0.005]$\\
Weight decay & \texttt{weight\_decay} & $10^{-4}$ & -- \\
Number of epochs & \texttt{num\_epochs} & $40$ & -- \\
Batch size & \texttt{batch\_size} & $32$; $128$ for EEG & -- \\
Source loss weight & $\alpha_{\mathrm{cls}}$ / \texttt{src\_cls\_loss\_wt} & $1.0$ & -- \\
CPDA discrepancy weight & $\lambda$ / \texttt{domain\_loss\_wt} & $1.0$ & $[0.1, 0.25, 0.5, 1.0, 2.0]$\\
Information maximization weight & $\eta$ / \texttt{cpda\_im\_wt} & $0.05$ & $[0.0, 0.01, 0.025$, $0.05, 0.1]$\\
VAT weight & $\gamma$ / \texttt{cpda\_vat\_wt} & $0.1$ & $[0.01, 0.05, 0.1, 0.25]$ \\
\midrule
Signature kernel weight & $\alpha_{\mathrm{sig}}$ / \texttt{cpda\_signature\_wt} & $1.0$ & $[0.0, 0.25, 0.5, 1.0]$ \\
Spectral kernel weight & $\alpha_{\mathrm{spec}}$ / \texttt{cpda\_spectral\_wt} & $0.5$ & $[0.25, 0.5, 1.0]$ \\
Pooled kernel weight & $\alpha_{\mathrm{pool}}$ / \texttt{cpda\_pool\_wt} & $0.25$ & $[0.1, 0.25, 0.5, 1.0]$ \\
Path kernel weight & $\alpha_{\mathrm{path}}$ / \texttt{cpda\_path\_wt} & $0.25$ & $[0.25, 0.5]$ \\
\midrule
Signature projection dimension & $r$ / \texttt{cpda\_sig\_proj\_dim} & $16$ & $[8, 16, 32]$ \\
Signature order & $m_{\mathrm{sig}}$ / \texttt{cpda\_sig\_level} & $2$ & $[1, 2, 3]$ \\
Maximum latent sequence length & $T_{\max}$ / \texttt{cpda\_max\_seq\_len} & $64$ & $[32, 64, 128]$ \\
Number of Gaussian kernels & $M$ / \texttt{cpda\_kernel\_num} & $5$ & -- \\
Kernel bandwidth multiplier & $\rho$ / \texttt{cpda\_kernel\_mul} & $2.0$ & -- \\
Fixed kernel bandwidth & $\sigma^2$ / \texttt{cpda\_fix\_sigma} & None & -- \\
\midrule
Target-prior weighting & \texttt{cpda\_use\_target\_prior} & Enabled & -- \\
Active-class normalization & \texttt{cpda\_normalize\_by\_active} & Enabled & -- \\
Detach target weights & \texttt{cpda\_detach\_target\_weights} & Enabled & -- \\
Pseudo-label confidence threshold & $\tau$ / \texttt{cpda\_min\_target\_confidence} & $0.0$ & -- \\
Feature normalization & \texttt{cpda\_normalize\_features} & Enabled & -- \\
\bottomrule
\end{tabular}
\end{table*}

All experiments are conducted on compute nodes equipped with NVIDIA Tesla V100-SXM2 GPUs with 32~GB of VRAM. Each node is additionally equipped with Intel Xeon CPUs and 192~GB of system memory. This setup provides sufficient GPU memory for mini-batch training of all considered time-series DA methods and enables a consistent comparison across backbones, datasets, and adaptation objectives.

Table~\ref{tab:cpda_hparams} provides an overview of main CPDA hyperparameters. The hyperparameters are specified at two levels: dataset-level training parameters and algorithm-specific optimization parameters. Unless otherwise stated, models are trained for $40$ epochs using weight decay $10^{-4}$. Most datasets use a mini-batch size of $32$, while the EEG dataset uses a mini-batch size of $128$. All non-adversarial methods are optimized using Adam, with the learning rate specified by the corresponding algorithm configuration. For CPDA, we use a learning rate of $5\times 10^{-3}$ and optimize the objective
\begin{equation}
    \mathcal{L}
    =
    \mathcal{L}_{\mathrm{cls}}
    +
    \lambda
    \widehat{\mathcal{D}}_{\mathrm{CPDA}}^2
    +
    \eta
    \mathcal{L}_{\mathrm{IM}},
    +
    \gamma \mathcal{L}_{\mathrm{VAT}}
\end{equation}
with $\lambda=1.0$, $\eta=0.05$, and $\gamma = 0.1$. The source classification loss is weighted by $1.0$. For the final CPDA variant, we additionally include a VAT regularizer, following its use in DIRT-T, to improve local prediction smoothness and stabilize target pseudo-label-based alignment. The corresponding weight is denoted by $\gamma$.

The CPDA kernel combines four components: signature, spectral, pooled, and temporal-path kernels. Their respective weights are set to
\[
    \alpha_{\mathrm{sig}}=1.0,\,\,\,
    \alpha_{\mathrm{spec}}=0.5,\,\,\,
    \alpha_{\mathrm{pool}}=0.25,\,\,\,
    \alpha_{\mathrm{path}}=0.25.
\]
The low-rank signature representation uses projection dimension $16$ and second-order signatures. The maximum latent sequence length is set to $64$. The multi-kernel Gaussian discrepancy uses $5$ bandwidths with multiplicative factor $2.0$, and no fixed bandwidth is imposed. Target-prior weighting and feature normalization are enabled. Gradients are detached through the target pseudo-label weights when computing the CPDA discrepancy, which stabilizes training by preventing degenerate pseudo-label manipulation. No target-confidence threshold is applied.

\paragraph{Ramp-up Schedule.} For CPDA, we apply a ramp-up schedule only to the optional target information-maximization loss. The supervised source classification loss, the CPDA discrepancy, and the VAT regularization are used with their full weights from the beginning of training. In contrast, the information-maximization loss is multiplied by a time-dependent coefficient $\alpha(t)\in[0,1]$, which is gradually increased during the first training iterations. This avoids enforcing overly confident target predictions too early, when the target pseudo-labels may still be unreliable, while preserving the stabilizing effect of CPDA and VAT throughout training. Unless stated otherwise, we use a sigmoid ramp-up with $1{,}000$ optimizer steps. For a training run with $40$ epochs, this corresponds to $1{,}000/N_{\mathrm{it}}$ epochs, where $N_{\mathrm{it}}$ denotes the number of mini-batch updates per epoch.

\section{Backbone Architectures}
\label{app:backbone}

\begin{table*}[t]
\centering
\caption{Summary of the backbone architectures used in the experiments. Architectural values are determined by the dataset-specific AdaTime~\citep{ragab_eldele_tan} configuration.}
\label{tab:backbone_summary}
\begin{tabular}{p{0.13\linewidth} p{0.40 \linewidth} p{0.40\linewidth}}
\toprule
\textbf{Backbone} & \textbf{Main Components} & \textbf{Output Representation} \\
\midrule
CNN & Three Conv1D blocks with BatchNorm, ReLU, MaxPool; dropout in first block & Flattened adaptive-pooled feature map of size $\texttt{features\_len}\cdot\texttt{final\_out\_channels}$ \\
\midrule
ResNet18-style & Four residual stages with configuration $[2,2,2,2]$; $1\times1$ Conv1D residual blocks & Flattened adaptive-pooled feature map of size $\texttt{features\_len}\cdot\texttt{final\_out\_channels}$ \\
\midrule
TCN & Two causal dilated residual Conv1D blocks with dilations $1$ and $2$; chomp operations & Final temporal state of the last residual feature map \\
\bottomrule
\end{tabular}
\end{table*}

\begin{table*}[t]
\centering
\caption{Important configuration parameters used by the backbone implementations.}
\label{tab:backbone_parameters}
\begin{tabular}{ll}
\toprule
\textbf{Parameter} & \textbf{Description} \\
\midrule
\texttt{input\_channels} & Number of input time-series channels $C_{\mathrm{in}}$ \\
\texttt{mid\_channels} & Intermediate convolutional channel width \\
\texttt{final\_out\_channels} & Final convolutional channel width used by CNN and ResNet18 \\
\texttt{kernel\_size} & Kernel size of the first CNN convolution \\
\texttt{stride} & Stride used in the first CNN convolution and first ResNet stage \\
\texttt{dropout} & Dropout probability in the first CNN block \\
\texttt{features\_len} & Adaptive output temporal length for CNN and ResNet18 \\
\texttt{tcn\_layers} & Channel configuration of the TCN residual blocks \\
\texttt{tcn\_kernel\_size} & Kernel size of the TCN convolutions \\
\texttt{num\_classes} & Number of output classes $K$ \\
\bottomrule
\end{tabular}
\end{table*}

To evaluate the proposed CPDA objective independently of a particular feature extractor, we use the backbone architectures provided by the AdaTime~\citep{ragab_eldele_tan} framework. We consider three representative one-dimensional time-series encoders: (i) a convolutional neural network (CNN), (ii) a ResNet18-style architecture, and (iii) a temporal convolutional network (TCN). Table~\ref{tab:backbone_summary} gives an overview of backbone architectures. All backbones operate on multivariate time-series inputs of shape
\begin{equation}
    x \in \mathbb{R}^{C_{\mathrm{in}}\times L},
\end{equation}
where $C_{\mathrm{in}}$ denotes the number of input channels and $L$ the temporal input length. The relevant architectural hyperparameters are specified through the dataset configuration, including the number of input channels \texttt{input\_channels}, the intermediate channel width \texttt{mid\_channels}, the final channel width \texttt{final\_out\_channels}, the convolutional kernel size \texttt{kernel\_size}, the stride \texttt{stride}, the dropout probability \texttt{dropout}, and the output temporal feature length \texttt{features\_len}. For CNN and ResNet18, the final feature map is adaptively pooled to length \texttt{features\_len} and flattened before classification. The resulting representation has dimension
\begin{equation}
    d_{\mathrm{feat}}
    =
    \texttt{features\_len}
    \cdot
    \texttt{final\_out\_channels}.
\end{equation}
A linear classifier then maps this representation to the $K$ output classes. For an overview, refer to Table~\ref{tab:backbone_parameters}.

\paragraph{CNN.}
The CNN backbone is a three-block one-dimensional convolutional encoder. The first block applies a temporal convolution with \texttt{input\_channels} input channels and \texttt{mid\_channels} output channels. The convolution uses kernel size \texttt{kernel\_size}, stride \texttt{stride}, no bias term, and symmetric padding of size $\lfloor \texttt{kernel\_size}/2 \rfloor$. It is followed by batch normalization, a ReLU activation, max pooling with kernel size $2$, stride $2$, and padding $1$, and dropout with probability \texttt{dropout}. The second block maps \texttt{mid\_channels} to $2\cdot\texttt{mid\_channels}$ using a one-dimensional convolution with kernel size $8$, stride $1$, padding $4$, and no bias term, followed again by batch normalization, ReLU, and max pooling. The third block maps $2\cdot\texttt{mid\_channels}$ to \texttt{final\_out\_channels} with the same kernel size $8$, stride $1$, and padding $4$, followed by batch normalization, ReLU, and max pooling. Finally, adaptive average pooling maps the temporal dimension to \texttt{features\_len}, and the output tensor is flattened. This architecture provides a simple local-pattern extractor for time-series data and produces a latent temporal feature map suitable for CPDA path construction.

\paragraph{ResNet18-Style Backbone.}
The ResNet18-style backbone follows a one-dimensional residual design with four residual stages and layer configuration
\begin{equation}
    [2,2,2,2].
\end{equation}
The first stage maps the input from \texttt{input\_channels} to \texttt{mid\_channels} using stride \texttt{stride}. The second stage maps to $2\cdot\texttt{mid\_channels}$, while the third and fourth stages use \texttt{final\_out\_channels}. Each residual block consists of a one-dimensional $1\times 1$ convolution, batch normalization, and ReLU activation, followed by an additive residual connection. If the number of channels or the stride differs between input and output, the residual branch is projected using a $1\times 1$ convolution followed by batch normalization. The block expansion factor is $1$. After the four residual stages, adaptive average pooling maps the temporal dimension to \texttt{features\_len}, and the resulting tensor is flattened into a vector of dimension $\texttt{features\_len}\cdot\texttt{final\_out\_channels}$. Compared to the plain CNN, the residual connections improve gradient propagation and allow the encoder to learn deeper temporal representations.

\paragraph{TCN.}
The TCN backbone is designed to model temporal dependencies using dilated causal convolutions and residual connections. It contains two temporal convolutional residual blocks. The first block uses dilation factor $1$, and the second block uses dilation factor $2$. Both blocks use the kernel size \texttt{tcn\_kernel\_size} and stride $1$. To preserve causal structure, the convolutional padding is set to
\begin{equation}
    p_\ell
    =
    (\texttt{tcn\_kernel\_size}-1) \, \delta_\ell,
\end{equation}
where $\delta_\ell$ denotes the dilation factor of block $\ell$. After each convolution, a chomp operation removes the additional padded time steps, ensuring that the temporal dimension is consistent and that no future information is used. Each block contains two one-dimensional convolutions, chomp operations, batch normalization, and ReLU activations. Residual connections are used in both blocks, with optional $1\times 1$ convolutional downsampling if the input and output channel dimensions differ. The TCN returns the final temporal state of the last residual feature map,
\begin{equation}
    z = Z_{:,T},
\end{equation}
as the sequence-level representation. This makes the TCN particularly suited to causal temporal modeling, where the final state summarizes the preceding sequence.

\section{Mathematical Foundations of Existing Methods}
\label{app:comparison_methods}
 
This appendix provides a unified mathematical analysis of the five DA methods evaluated in our benchmark: MMD, CORAL (Deep CORAL), MMDA, DAN, and HoMM. We derive each loss formally and identify the structural property that depends on whether the encoder representation retains temporal information.

\subsection{Maximum Mean Discrepancy (MMD)}
\label{app:comparison_methods_mmd}
 
Let $\mathcal{H}_k$ be an RKHS with kernel $k : \mathcal{F} \times \mathcal{F} \to \mathbb{R}$. The squared MMD between feature distributions is
\begin{equation}
\label{eq:mmd}
  \mathrm{MMD}^2(\mathbb{P}_S, \mathbb{P}_T)
    = ||
        \mathbb{E}_{x \sim \mathbb{P}_S}[\phi(x)]
        - \mathbb{E}_{x \sim \mathbb{P}_T}[\phi(x)]
      ||^2_{\mathcal{H}_k},
\end{equation}
where $\phi : \mathcal{F} \to \mathcal{H}_k$ is the canonical feature map. Expanding via the kernel trick $k(x,x') = \langle\phi(x),\phi(x')\rangle$:
\begin{equation}
\label{eq:mmd-kernel}
\begin{split}
  \mathrm{MMD}^2(\mathbb{P}_S, \mathbb{P}_T)
    = \mathbb{E}[k(x_s, x_s')]
      &- 2\,\mathbb{E}[k(x_s, x_t)]\\
      &+ \mathbb{E}[k(x_t, x_t')].
\end{split}
\end{equation}
If $k$ is a \textit{characteristic kernel} (e.g.\ Gaussian RBF), then $\mathrm{MMD}(\mathbb{P}_S, \mathbb{P}_T) = 0 \iff \mathbb{P}_S = \mathbb{P}_T$ as marginal distributions~\citep{gretton2012kernel}.
 
\paragraph{Time-series Limitation.} The empirical embedding $\tfrac{1}{T}\sum_t \phi(F_t)$ is invariant to any permutation of the time axis. Consequently, MMD is insensitive to autocorrelation, oscillatory patterns, or any other temporal structure encoded in the \textit{joint} distribution of the sequence.
 
\subsection{Correlation Alignment (CORAL)}
\label{app:comparison_methods_coral}
 
Deep CORAL~\citep{sun_feng_saenko} aligns second-order statistics by minimizing the Frobenius norm between feature covariance matrices:
\begin{equation}
\label{eq:coral}
  \mathcal{L}_{\mathrm{CORAL}}
    = \frac{1}{4d^2}
      ||\bm{C}_S - \bm{C}_T||_F^2,
\end{equation}
where $\bm{C}_S, \bm{C}_T \in \mathbb{R}^{d\times d}$ are the centered feature covariance matrices. This is equivalent to second-order HoMM (see App.~\ref{app:comparison_methods_homm}).
 
\textbf{Time-series Limitation.} CORAL operates on the \textit{feature-level} contemporaneous covariance $\bm{C} = \mathbb{E}[F_t F_t^\top]$, which captures cross-feature correlations at a single time step. It does not capture the \textit{temporal} autocovariance $R(\tau) = \mathbb{E}[F_t F_{t+\tau}^\top]$ for any lag $\tau > 0$. Hence CORAL may fail to preserve detecting shifts in temporal dynamics, even those that are immediately visible to the human eye (e.g., a change from positive to negative autocorrelation).
 
\subsection{Multiple-Moment Domain Adaptation (MMDA)}
\label{app:comparison_methods_mmda}
 
MMDA~\citep{rahman_fookes} jointly minimizes several distribution discrepancies:
\begin{equation}
\label{eq:mmda}
  \mathcal{L}_{\mathrm{MMDA}}
    = \lambda_1 \mathcal{L}_{\mathrm{mean}}
      + \lambda_2 \mathcal{L}_{\mathrm{cov}}
      + \lambda_3 \mathcal{L}_{\mathrm{scatter}},
\end{equation}
where the scatter term uses a kernel-mean embedding at higher order. The theoretical basis rests on a truncated Wasserstein-type bound: if $k$ finite moments of $\mathbb{P}_S$ and $\mathbb{P}_T$ are aligned, then
\begin{equation}
  W_2(\mathbb{P}_S, \mathbb{P}_T)
    \le C \cdot \left(\sum_{p=1}^k
                  ||\mu_S^{(p)} - \mu_T^{(p)}||_F
                \right)^{1/k},
\end{equation}
where $W_2$ is the $2$-Wasserstein distance. However, this bound applies to the \textit{marginal} Wasserstein distance only; the bound cannot close the gap arising from temporal structure differences.
 
\subsection{Deep Adaptation Networks (DAN)}
\label{app:comparison_methods_dan}
 
DAN~\citep{long_cao_wang_DAN} extends MMD to multiple feature layers using multi-kernel MMD (MK-MMD):
\begin{equation}
\label{eq:dan}
  \mathcal{L}_{\mathrm{DAN}}
    = \sum_{\ell=\ell_1}^{\ell_K}
      \mathrm{MMD}^2_\mathcal{K}
      \!\left(\mathcal{D}_S^{(\ell)},\, \mathcal{D}_T^{(\ell)}\right),
\end{equation}
where the composite kernel $\mathcal{K} = \{\sum_u \beta_u k_u : \beta_u \ge 0,\, \sum_u \beta_u = 1\}$ is optimized via the ratio of within-class to cross-class MMD.
 
\textbf{Time-series Limitation.}
The multi-layer aggregation still operates on \textit{frame-level} activations at each layer. For a TCN or transformer with receptive field $\le T$, the temporal context captured by deep layers has a bounded horizon, and the implicit i.i.d.~assumption underlying Equation~\ref{eq:dan} is violated. More fundamentally, each layer's activation is still summarized by its marginal distribution.
 
\subsection{Higher-Order Moment Matching (HoMM)}
\label{app:comparison_methods_homm}
 
HoMM~\citep{chen_fu_chen} performs order-$p$ tensor moment matching:
\begin{equation}
\label{eq:homm}
  \mathcal{L}_{\mathrm{HoMM}}^{(p)}
    = \big|\big|
        \mathbb{E}_S[F^{\otimes p}]
        - \mathbb{E}_T[F^{\otimes p}]
      \big|\big|_F^2,
\end{equation}
where $F^{\otimes p}$ denotes the $p$-fold tensor product. The first-order HoMM ($p=1$) is equivalent to MMD with a linear kernel, and the second-order ($p=2$) is equivalent to CORAL. For $p \ge 3$, HoMM targets fine-grained non-Gaussian structure.
 
\textbf{Computational Cost.} Storing $\mathbb{E}[F^{\otimes p}] \in \mathbb{R}^{d^p}$ requires $O(d^p)$ memory, which is exponential in $p$. For $d = 512$, $p = 3$, this is approximately $1.3 \times 10^8$ floats. In practice, random tensor sketches reduce this to $O(d)$ per projection, but at the cost of approximation bias.
 
\textbf{Time-series Limitation.} HoMM matches the \textit{marginal} $p$-th moment tensor $\mathbb{E}[F^{\otimes p}]$ but does not explicitly enforce \textit{cross-time} moments $\mathbb{E}[F_t^{\otimes p_1} \otimes F_{t+\tau}^{\otimes p_2}]$. These joint moments encode the dynamical structure (e.g.\ AR coefficients, oscillation periods) that defines the identity of a time-series. This failure persists at \textit{all} orders $p$.

\subsection{Virtual Adversarial Training (VAT)}
\label{app:comparison_methods_vat}

VAT, proposed by ~\citet{shu_bui_narui}, is added as a target-smoothness regularizer to stabilize the decision boundary in regions of high target density. For an input $x$, VAT penalizes the change in the model prediction under the worst-case small perturbation,
\begin{equation}
    \mathcal{L}_{\mathrm{VAT}}
    =
    \mathbb{E}_{x\in \mathcal{D}_s\cup\mathcal{D}_t}
    \big[
        \mathrm{KL}
        \big(
            p_{\theta,\phi}(\cdot\mid x)
            \,\Vert\,
            p_{\theta,\phi}(\cdot\mid x+r_{\mathrm{adv}})
        \big)
    \big],
\end{equation}
where
\begin{equation}
    r_{\mathrm{adv}}
    =
    \arg\max_{\|r\|\leq \epsilon}
    \mathrm{KL}
    \big(
        p_{\theta,\phi}(\cdot\mid x)
        \,\Vert\,
        p_{\theta,\phi}(\cdot\mid x+r)
    \big).
\end{equation}
The parameter $\gamma$ controls the contribution of the VAT regularizer.

While CPDA aligns class-conditional latent path distributions, VAT improves target-side robustness by encouraging local prediction smoothness around both source and target samples, thereby reducing the effect of noisy target pseudo-labels during class-conditional alignment.

\section{Dataset Details}
\label{app:datasets}

\begin{table*}[t!]
\centering
\caption{Overview of datasets used in the evaluation. Sample counts refer to the total number of samples used in the corresponding standard train\,/\,test split or benchmark preprocessing. Datasets are available at \url{https://researchdata.ntu.edu.sg/dataverse/adatime}, \url{https://www.timeseriesclassification.com/dataset.php}, or \url{https://www.iis.fraunhofer.de/de/ff/lv/dataanalytics/anwproj/schreibtrainer/onhw-dataset.html}.}
\label{tab:datasets}
\small
\begin{tabular}{p{0.14\linewidth} p{0.21\linewidth} p{0.30\linewidth} r r r r r }
\toprule
\textbf{Dataset} & \textbf{Reference} & \textbf{Description} & \rotatebox{90}{\textbf{\# Domains}} & \rotatebox{90}{\textbf{\# Channels}} & \rotatebox{90}{\textbf{\# Classes}} & \rotatebox{90}{\textbf{\# Length}} & \rotatebox{90}{\textbf{\# Samples}} \\
\midrule
HHAR & \citet{stisen_blunck_bhattacharya} & Smartphone-based human activity recognition with subject/domain shift & $9$ & $3$ & $6$ & $128$ & $17{,}934$ \\
UCI HAR & \citet{anguita_ghio_oneto} & Inertial human activity recognition from accelerometer and gyroscope signals & $30$ & $9$ & $6$ & $128$ & $3{,}290$ \\
WISDM & \citet{kwapisz_weiss_moore} & Smartphone accelerometer-based human activity recognition & $36$ & $3$ & $6$ & $128$ & $2{,}070$ \\
EEG (Sleep-EDF) & \citet{goldberger_amaral_glass} & Single-channel EEG sleep-stage classification dataset & $20$ & $1$ & $5$ & $3{,}000$ & $20{,}410$ \\
uWave & UCR/UEA~\citep{liu_wang_zhong} & Accelerometer-based gesture recognition from the UWave gesture library & $8$ & $3$ & $5$ & $150$ & $4{,}478$ \\
\midrule
PenDigits & UCR/UEA~\citep{alimoglu_alpaydin} & Pen-trajectory recognition of handwritten digits & $2$ & $2$ & $10$ & $8$ & $10{,}992$ \\
Epilepsy & UCR/UEA~\citep{villar_vergara_menendez} & Multivariate accelerometer recordings for epilepsy-related movement classification & $2$ & $3$ & $4$ & $206$ & $275$ \\
Face {detection} & UCR/UEA~\citep{olivetti_kia_avesani} & MEG-based binary classification of face versus scrambled-image stimuli & $2$ & $144$ & $2$ & $62$ & $9{,}414$ \\
Finger movements (univariate) & UCR/UEA~\citep{blankertz_curio_mueller} & EEG-based binary finger-movement classification, univariate variant & $28$ & $1$ & $2$ & $50$ & $208$ \\
Finger movements (multi) & UCR/UEA~\citep{blankertz_curio_mueller} & Multivariate EEG-based binary finger-movement classification & $28$ & $1$ & $2$ & $50$ & $208$ \\
\midrule
OnHW-symbols & \citet{ott_ijdar} & Online handwriting recognition of mathematical symbols & $2$ & $13$ & $15$ & $79$ & $2{,}667$ \\
OnHW-equations & \citet{ott_ijdar} & Online handwriting recognition from split equation symbols & $2$ & $13$ & $15$ & $79$ & $45{,}923$ \\
OnHW-chars (combined) & \citet{ott_imwut} & Online handwriting recognition of uppercase and lowercase characters & $2$ & $13$ & $52$ & $64$ & $33{,}545$ \\
\midrule
sin--cos (synthetic) & \citet{ott_acmmm} & Controlled synthetic waveform dataset for ablation studies & $2$ & $1$ & $10$ & $1{,}000$ & $12{,}000$ \\
\bottomrule
\end{tabular}
\end{table*}

We evaluate the proposed method on a diverse collection of time-series datasets covering human activity recognition, sleep-stage classification, gesture recognition, biomedical signal analysis, handwriting recognition, and synthetic waveform classification. Table~\ref{tab:datasets} summarizes the main dataset statistics, including the number of domains, channels, classes, temporal length, and train/test sample sizes. In all DA experiments, domains correspond either to individual users, subjects, devices, or predefined source--target partitions, depending on the dataset.

\paragraph{HHAR.} The Heterogeneity Human Activity Recognition (HHAR) dataset~\citep{stisen_blunck_bhattacharya} contains human activity recordings from nine users. The signals were collected using three-axis accelerometers with sampling rates between $25\,\text{Hz}$ and $200\,\text{Hz}$. In total, the dataset includes measurements from $36$ smartphones, tablets, and smartwatches, enabling the analysis of domain shifts induced by different sensors, devices, and workloads. We treat users as domains and evaluate DA over all $36$ possible user-to-user transfer scenarios. This provides a more comprehensive evaluation than prior studies, which commonly considered only a small subset of transfer pairs~\citep{wilson_doppa_cook,liu_xue,he_queen_koker,ragab_eldele_tan}.

\paragraph{UCI HAR.} The UCI Human Activity Recognition (HAR) dataset~\citep{anguita_ghio_oneto} consists of inertial measurements from $30$ users performing six activities of daily living: standing, sitting, lying down, walking, walking upstairs, and walking downstairs. The data were recorded with waist-mounted smartphones equipped with inertial sensors sampled at $50\,\text{Hz}$. Each sample contains nine channels, corresponding to three-axis accelerometer, gyroscope, and magnetometer measurements. We treat each user as a separate domain and evaluate DA across all $435$ possible source--target user combinations.

\paragraph{WISDM.} The WISDM dataset~\citep{kwapisz_weiss_moore} contains smartphone accelerometer recordings from $36$ users performing six activities: walking, jogging, ascending stairs, descending stairs, sitting, and standing. Each recording consists of three-axis acceleration signals. We define each user as an individual domain and evaluate DA over $356$ source--target user combinations.

\paragraph{EEG.} The sleep-stage classification (SSC) dataset~\citep{goldberger_amaral_glass} contains electroencephalography (EEG) recordings labeled according to five sleep stages: wakefulness, three non-rapid-eye-movement stages, and rapid-eye-movement sleep. The dataset includes univariate recordings from $20$ users, with each time-series containing $3{,}000$ time steps. Each user is treated as a distinct domain, resulting in $190$ possible source--target transfer scenarios.

\paragraph{uWave.} The uWave dataset~\citep{liu_wang_zhong} consists of gesture recordings obtained from a handheld device equipped with a three-axis accelerometer. The dataset contains $4{,}478$ samples from eight users and comprises five gesture classes. Following our preprocessing, all time-series are interpolated to a fixed length of $150$ time steps. We evaluate DA over all $28$ possible source--target combinations between the eight users. In contrast to some prior work~\citep{wilson_doppa_cook,liu_xue}, we use an 80\%\,/\,20\% train/validation split.

\paragraph{PenDigits.} The PenDigits dataset~\citep{alimoglu_alpaydin} consists of handwritten digit trajectories recorded on a digital screen. Each sample is represented by two channels corresponding to the two-dimensional pen coordinates. The task is classification into $10$ digit classes. The dataset contains $10{,}992$ samples collected from $44$ writers. As with the Epilepsy dataset, we construct source and target domains by randomly partitioning the available samples.

\paragraph{Epilepsy.} The Epilepsy dataset~\citep{villar_vergara_menendez} was recorded using a tri-axial accelerometer sampled at $16\,\text{Hz}$ and worn on the dominant wrist. Six healthy participants performed four activities: walking, running, sawing, and seizure mimicking, with each activity repeated at least ten times. Due to the limited dataset size, we randomly split the samples into source and target domains with a 50\%\,/\,50\% ratio and further divide each domain into training and testing sets. Consequently, samples from all participants may appear in both domains, but with unequal proportions.

\paragraph{Face Detection.} The Face Detection dataset~\citep{olivetti_kia_avesani} addresses binary classification of magnetoencephalography (MEG) recordings. The task is to determine whether a participant was presented with a face image or a scrambled image. The dataset contains $144$-channel MEG recordings from $16$ participants and comprises $9{,}414$ samples.

\paragraph{Finger Movements.} The Finger Movements dataset~\citep{blankertz_curio_mueller} consists of EEG recordings from a subject performing key-press movements with the index and little fingers while seated in a standard typing position. The task is binary classification, distinguishing left- and right-hand finger movements. The dataset is comparatively small, containing $104$ training and testing samples. We evaluate DA over $378$ possible source--target scenarios derived from the available domain partitions.





\begin{table*}[t!]
\begin{center}
\caption{Results for the \textit{source} and \textit{target} test datasets, which are evaluated on the model \textit{trained solely on target samples} to show the discrepancy between source and target domains. For all datasets, we report the accuracy and F1-score in \% for the CNN, ResNet18, and TCN encoder networks. These metrics are averaged over all domain scenarios and five training runs. \textbf{Bold} refers to the best-performing encoder network on target and source accuracy.}
\label{table_results_trg_src}
\footnotesize\begin{tabular}{lrrrrrrrrrrrr}
\toprule
\textbf{Dataset}
& \shortstack{\textbf{CNN}\\\textbf{Tar.}\\\textbf{Acc.}}
& \shortstack{\textbf{CNN}\\\textbf{Tar.}\\\textbf{F1}}
& \shortstack{\textbf{CNN}\\\textbf{Sour.}\\\textbf{Acc.}}
& \shortstack{\textbf{CNN}\\\textbf{Sour.}\\\textbf{F1}}
& \shortstack{\textbf{ResNet}\\\textbf{Tar.}\\\textbf{Acc.}}
& \shortstack{\textbf{ResNet}\\\textbf{Tar.}\\\textbf{F1}}
& \shortstack{\textbf{ResNet}\\\textbf{Sour.}\\\textbf{Acc.}}
& \shortstack{\textbf{ResNet}\\\textbf{Sour.}\\\textbf{F1}}
& \shortstack{\textbf{TCN}\\\textbf{Tar.}\\\textbf{Acc.}}
& \shortstack{\textbf{TCN}\\\textbf{Tar.}\\\textbf{F1}}
& \shortstack{\textbf{TCN}\\\textbf{Sour.}\\\textbf{Acc.}}
& \shortstack{\textbf{TCN}\\\textbf{Sour.}\\\textbf{F1}} \\
\midrule
HHAR & \textbf{98.70} & 98.70 & 63.88 & 58.87 & 98.01 & 98.05 & 47.42 & 42.54 & 98.02 & 98.06 & 60.77 & 55.83 \\
UCI HAR & 99.42 & 99.48 & 74.31 & 71.27 & \textbf{99.77} & 99.78 & 69.96 & 64.99 & 98.68 & 98.68 & 72.44 & 68.79 \\
WISDM & \textbf{99.28} & 98.21 & 55.23 & 38.44 & 99.22 & 98.09 & 55.60 & 32.84 & 93.83 & 88.46 & 56.36 & 37.25 \\
EEG & \textbf{79.62} & 70.18 & 66.69 & 54.84 & 62.13 & 48.34 & 49.71 & 36.43 & 52.45 & 39.90 & 45.11 & 33.52 \\
uWave & \textbf{99.82} & 99.83 & 83.54 & 82.05 & 96.43 & 96.15 & 43.38 & 38.12 & 99.43 & 99.41 & 79.65 & 77.70 \\
PenDigits & \textbf{99.69} & 99.68 & 96.07 & 95.99 & 84.41 & 84.28 & 78.92 & 78.78 & 99.54 & 99.53 & 96.03 & 96.01 \\
Epilepsy & 80.31 & 75.40 & 77.50 & 72.87 & \textbf{95.31} & 95.33 & 97.81 & 97.87 & 74.38 & 68.06 & 74.69 & 73.30 \\
Face detection & 52.59 & 40.59 & 56.46 & 48.86 & 53.52 & 50.94 & 51.75 & 45.88 & \textbf{69.92} & 69.92 & 64.10 & 64.09 \\
Finger movements & 55.91 & 51.08 & 51.23 & 42.55 & 50.66 & 45.47 & 50.51 & 46.00 & \textbf{57.00} & 55.16 & 51.48 & 46.43 \\
OnHW-symbols & \textbf{77.19} & 78.07 & 20.00 & 14.87 & 76.25 & 76.59 & 13.75 & 9.74 & 19.69 & 15.12 & 5.00 & 3.54 \\
OnHW-equations & \textbf{89.66} & 89.02 & 36.83 & 38.29 & 82.34 & 82.30 & 21.92 & 21.60 & 56.94 & 54.48 & 29.49 & 29.44 \\
OnHW-chars & \textbf{73.52} & 71.76 & 37.44 & 36.83 & 49.93 & 48.02 & 16.10 & 15.05 & 1.29 & 0.86 & 1.92 & 0.07 \\
sin--cos ($b=0.0$) & 23.99 & 15.24 & 40.12 & 29.18 & 58.69 & 50.45 & 10.18 & 3.45 & \textbf{100.0} & 100.0 & 0.00 & 0.00 \\
sin--cos ($b=0.5$) & 46.27 & 39.35 & 43.55 & 32.50 & 45.56 & 36.72 & 10.91 & 5.12 & \textbf{98.35} & 98.33 & 8.65 & 7.09 \\
sin--cos ($b=1.0$) & 34.74 & 26.95 & 37.28 & 29.55 & 21.01 & 9.65 & 8.67 & 2.33 & \textbf{84.92} & 84.24 & 0.87 & 1.12 \\
sin--cos ($b=1.5$) & 25.79 & 17.57 & 30.97 & 25.37 & 11.11 & 2.79 & 8.41 & 3.59 & \textbf{65.04} & 62.32 & 0.38 & 0.44 \\
sin--cos ($b=1.9$) & 28.75 & 28.95 & 22.69 & 24.17 & 10.30 & 2.30 & 8.87 & 3.69 & \textbf{59.05} & 57.58 & 0.50 & 0.59 \\
\bottomrule
\end{tabular}
\end{center}
\end{table*}

\paragraph{Online Handwriting.} The Online Handwriting (OnHW) datasets contain handwriting samples captured with sensor-enhanced pens. The OnHW-chars dataset~\citep{ott_imwut} addresses recognition of lowercase and uppercase characters, resulting in $52$ classes. It contains $31{,}275$ samples from $119$ right-handed writers and $2{,}270$ samples from nine left-handed writers. The OnHW-symbols and split OnHW-equations datasets~\citep{ott_ijdar} contain handwritten numbers and mathematical symbols. In these datasets, writer handedness induces a natural domain shift: the right-handed domain consists of $27$ and $55$ writers, respectively, while the left-handed domain contains four writers in both cases. We evaluate adaptation in both transfer directions, from left-handed to right-handed writers (L$\rightarrow$R) and from right-handed to left-handed writers (R$\rightarrow$L). Although these datasets contain fewer domains than the user-based sensor datasets, they are challenging due to their larger number of channels and classes.

\paragraph{Generated Sinusoidal Data.} For controlled ablation studies, we use the synthetic sinusoidal dataset introduced by~\citet{ott_acmmm}. The dataset consists of univariate sinusoidal time-series with class-specific frequencies for $10$ classes. Domain shift is introduced by adding noise sampled from a continuous uniform distribution $U(a,b)$ with $a=0.0$ and
\begin{equation}
    b \in \mathcal{B}=\{0.0,0.1,0.2,\ldots,1.9\}.
\end{equation}
The noisy sinusoidal signals form the target domain. For the source domain, the sign of the time-series is flipped and noise sampled from $U(a,b/2)$ is added. This controlled setup enables systematic evaluation of DA methods under gradually increasing noise levels and sign-induced distribution shift.

\section{Experimental Results}
\label{app:experimental_results}

\subsection{Discrepancy Between Source and Target Domains} 
\label{app:experimental_results_discrepancy}

We first quantify the source-only transfer gap between source and target domains by comparing two settings: training and evaluating a model on the target-domain data, and training on the source-domain data while evaluating on the target domain. The corresponding results are reported in Table~\ref{table_results_trg_src}. Across the considered datasets, the performance differences indicate a clear domain shift. The observed gaps amount to 34.82\%p for HHAR, 25.11\%p for UCI HAR, 37.47\%p for WISDM, 12.93\%p for SSC, 16.28\%p for uWave, and 23.95\%p for Gestures Mid Air. Thus, all datasets exhibit a measurable mismatch between source and target distributions. The domain shift is particularly pronounced for the OnHW datasets, where the gaps reach 57.19\%p, 52.83\%p, and 36.08\%p. This is consistent with the substantial differences between sensor recordings from right-handed and left-handed writers reported in~\cite{klass_lorenz_strl}. In contrast, the Finger Movements, Epilepsy, Face Detection, and PenDigits datasets show only small domain gaps, which is expected because their source and target partitions are generated by random shuffling rather than by naturally distinct domains. For the synthetic sinusoidal dataset, the target-domain accuracy decreases substantially as the noise level increases, in line with previous observations~\cite{ott_acmmm,ott_tmlr}. Specifically, the accuracy drops from 100\% at $b=0.0$ to 59.05\% at $b=1.9$. This degradation indicates a strong source--target mismatch, since the model trained on the source domain fails to extract features that remain reliably discriminative under the noisy target-domain conditions. The gap is smaller when using the CNN encoder, but this reduction comes at the cost of lower target-domain accuracy.

\subsection{Evaluation Results for all DA Methods} 
\label{app:experimental_results_da_methods}

\begin{table*}[t!]
\centering
\caption{DA results in \% (mean and standard deviation over all scenarios and five runs) for the HHAR~\citep{stisen_blunck_bhattacharya} dataset. \textbf{Bold} are best results.}
\label{table_results_da_all_methods1}
\footnotesize
\begin{tabular}{lrrrrrr}
\toprule
\textbf{Method}
& \makecell{\textbf{CNN}\\\textbf{Accuracy}}
& \makecell{\textbf{CNN}\\\textbf{F1}}
& \makecell{\textbf{ResNet18}\\\textbf{Accuracy}}
& \makecell{\textbf{ResNet18}\\\textbf{F1}}
& \makecell{\textbf{TCN}\\\textbf{Accuracy}}
& \makecell{\textbf{TCN}\\\textbf{F1}} \\
\midrule
MSE & 62.24\,$\pm$\,16.7 & 57.79\,$\pm$\,18.0 & 48.40\,$\pm$\,16.4 & 43.37\,$\pm$\,18.1 & 52.08\,$\pm$\,14.1 & 46.31\,$\pm$\,16.2 \\
CS & 63.07\,$\pm$\,18.7 & 59.09\,$\pm$\,21.0 & 53.68\,$\pm$\,19.6 & 50.98\,$\pm$\,22.1 & 56.09\,$\pm$\,17.8 & 51.36\,$\pm$\,20.8 \\
PC & 64.74\,$\pm$\,17.3 & 60.68\,$\pm$\,19.1 & 54.97\,$\pm$\,19.6 & 51.59\,$\pm$\,22.1 & 61.43\,$\pm$\,16.4 & 57.20\,$\pm$\,19.1 \\
KL & 60.10\,$\pm$\,16.9 & 54.76\,$\pm$\,18.6 & 50.83\,$\pm$\,18.7 & 46.08\,$\pm$\,21.7 & 19.36\,$\pm$\,\,\,\,6.0 & 8.39\,$\pm$\,\,\,\,5.8 \\
JSD & 64.75\,$\pm$\,17.4 & 60.62\,$\pm$\,20.0 & 54.55\,$\pm$\,20.4 & 50.74\,$\pm$\,23.4 & 17.79\,$\pm$\,\,\,\,3.6 & 7.33\,$\pm$\,\,\,\,3.1 \\
Linear MMD & 66.83\,$\pm$\,17.5 & 63.13\,$\pm$\,19.7 & 57.50\,$\pm$\,20.3 & 54.95\,$\pm$\,22.8 & 68.33\,$\pm$\,18.3 & 65.01\,$\pm$\,21.3 \\
Kernelized MMD & 63.86\,$\pm$\,17.1 & 59.64\,$\pm$\,18.9 & 54.50\,$\pm$\,19.8 & 51.25\,$\pm$\,22.3 & 61.48\,$\pm$\,17.4 & 57.48\,$\pm$\,20.6 \\
\bottomrule
\end{tabular}
\end{table*}

\begin{table*}[t!]
\centering
\caption{DA results in \% (mean and standard deviation over all scenarios and five runs) for the UCI HAR~\citep{anguita_ghio_oneto} dataset. \textbf{Bold} are best results.}
\label{table_results_da_all_methods2}
\footnotesize
\begin{tabular}{lrrrrrr}
\toprule
\textbf{Method}
& \makecell{\textbf{CNN}\\\textbf{Accuracy}}
& \makecell{\textbf{CNN}\\\textbf{F1}}
& \makecell{\textbf{ResNet18}\\\textbf{Accuracy}}
& \makecell{\textbf{ResNet18}\\\textbf{F1}}
& \makecell{\textbf{TCN}\\\textbf{Accuracy}}
& \makecell{\textbf{TCN}\\\textbf{F1}} \\
\midrule
MSE & 72.41\,$\pm$\,18.0 & 69.14\,$\pm$\,19.9 & 68.04\,$\pm$\,16.4 & 63.93\,$\pm$\,17.7 & 69.73\,$\pm$\,16.0 & 65.92\,$\pm$\,17.3 \\
CS & 77.62\,$\pm$\,18.0 & 75.43\,$\pm$\,19.9 & 73.88\,$\pm$\,16.3 & 70.49\,$\pm$\,17.7 & 68.99\,$\pm$\,16.5 & 65.34\,$\pm$\,17.8 \\
PC & 80.52\,$\pm$\,15.9 & 78.79\,$\pm$\,17.1 & 77.51\,$\pm$\,15.5 & 74.53\,$\pm$\,16.8 & 78.26\,$\pm$\,14.6 & 76.08\,$\pm$\,15.6 \\
KL & 69.17\,$\pm$\,18.9 & 65.30\,$\pm$\,21.1 & 72.04\,$\pm$\,16.4 & 67.42\,$\pm$\,18.0 & 16.64\,$\pm$\,\,\,\,4.8 & 5.36\,$\pm$\,\,\,\,3.2 \\
JSD & 79.70\,$\pm$\,16.3 & 77.95\,$\pm$\,17.5 & 72.97\,$\pm$\,14.7 & 67.67\,$\pm$\,16.5 & 15.85\,$\pm$\,\,\,\,3.8 & 4.74\,$\pm$\,\,\,\,1.9 \\
Linear MMD & 82.48\,$\pm$\,14.7 & 81.22\,$\pm$\,15.7 & 80.97\,$\pm$\,13.4 & 78.44\,$\pm$\,14.6 & 84.62\,$\pm$\,12.6 & 83.35\,$\pm$\,13.6 \\
Kernelized MMD & 80.95\,$\pm$\,15.7 & 79.26\,$\pm$\,17.0 & 77.88\,$\pm$\,15.2 & 74.94\,$\pm$\,16.6 & 78.46\,$\pm$\,14.8 & 76.28\,$\pm$\,15.8 \\
\bottomrule
\end{tabular}
\end{table*}

\begin{table*}[t!]
\centering
\caption{DA results in \% (mean and standard deviation over all scenarios and five runs) for the WISDM~\citep{kwapisz_weiss_moore} dataset. \textbf{Bold} are best results.}
\label{table_results_da_all_methods3}
\footnotesize
\begin{tabular}{lrrrrrr}
\toprule
\textbf{Method}
& \makecell{\textbf{CNN}\\\textbf{Accuracy}}
& \makecell{\textbf{CNN}\\\textbf{F1}}
& \makecell{\textbf{ResNet18}\\\textbf{Accuracy}}
& \makecell{\textbf{ResNet18}\\\textbf{F1}}
& \makecell{\textbf{TCN}\\\textbf{Accuracy}}
& \makecell{\textbf{TCN}\\\textbf{F1}} \\
\midrule
MSE & 49.91\,$\pm$\,19.8 & 33.28\,$\pm$\,18.4 & 48.92\,$\pm$\,16.8 & 27.87\,$\pm$\,14.4 & 50.19\,$\pm$\,15.1 & 27.08\,$\pm$\,12.9 \\
CS & 50.69\,$\pm$\,20.2 & 34.60\,$\pm$\,18.2 & 50.84\,$\pm$\,19.9 & 30.87\,$\pm$\,16.5 & 39.26\,$\pm$\,15.0 & 25.93\,$\pm$\,13.6 \\
PC & 58.21\,$\pm$\,21.2 & 40.51\,$\pm$\,20.6 & 57.60\,$\pm$\,18.8 & 35.47\,$\pm$\,17.4 & 53.79\,$\pm$\,20.3 & 37.50\,$\pm$\,19.5 \\
KL & 45.27\,$\pm$\,17.6 & 30.57\,$\pm$\,17.2 & 54.72\,$\pm$\,18.7 & 32.75\,$\pm$\,17.1 & 29.34\,$\pm$\,12.3 & 9.66\,$\pm$\,\,\,\,4.8 \\
JSD & 58.46\,$\pm$\,20.0 & 41.23\,$\pm$\,20.6 & 54.71\,$\pm$\,20.4 & 35.28\,$\pm$\,19.2 & 29.03\,$\pm$\,11.9 & 8.75\,$\pm$\,\,\,\,4.2 \\
Linear MMD & 59.68\,$\pm$\,19.9 & 41.04\,$\pm$\,20.1 & 58.28\,$\pm$\,18.7 & 36.27\,$\pm$\,17.8 & 60.37\,$\pm$\,19.7 & 38.93\,$\pm$\,18.9 \\
Kernelized MMD & 59.15\,$\pm$\,20.6 & 41.28\,$\pm$\,20.6 & 57.70\,$\pm$\,18.6 & 35.43\,$\pm$\,17.2 & 57.40\,$\pm$\,18.9 & 38.13\,$\pm$\,18.7 \\
\bottomrule
\end{tabular}
\end{table*}

\begin{table*}[t!]
\centering
\caption{DA results in \% (mean and standard deviation over all scenarios and five runs) for the EEG~\citep{goldberger_amaral_glass} dataset. \textbf{Bold} are best results.}
\label{table_results_da_all_methods4}
\footnotesize
\begin{tabular}{lrrrrrr}
\toprule
\textbf{Method}
& \makecell{\textbf{CNN}\\\textbf{Accuracy}}
& \makecell{\textbf{CNN}\\\textbf{F1}}
& \makecell{\textbf{ResNet18}\\\textbf{Accuracy}}
& \makecell{\textbf{ResNet18}\\\textbf{F1}}
& \makecell{\textbf{TCN}\\\textbf{Accuracy}}
& \makecell{\textbf{TCN}\\\textbf{F1}} \\
\midrule
MSE & 69.12\,$\pm$\,11.1 & 58.02\,$\pm$\,12.1 & 51.63\,$\pm$\,\,\,\,9.1 & 38.55\,$\pm$\,8.8 & 45.87\,$\pm$\,\,\,\,4.8 & 33.59\,$\pm$\,4.7 \\
CS & 71.51\,$\pm$\,\,\,\,9.8 & 57.79\,$\pm$\,11.3 & \textbf{53.13}\,$\pm$\,\,\,\,8.1 & 36.18\,$\pm$\,9.2 & 46.00\,$\pm$\,\,\,\,4.9 & 33.89\,$\pm$\,4.8 \\
PC & 69.80\,$\pm$\,11.1 & 59.03\,$\pm$\,11.6 & 51.86\,$\pm$\,\,\,\,9.1 & 38.80\,$\pm$\,8.7 & 45.89\,$\pm$\,\,\,\,5.0 & 34.15\,$\pm$\,4.9 \\
KL & 69.20\,$\pm$\,11.0 & 58.38\,$\pm$\,11.7 & 50.03\,$\pm$\,\,\,\,9.5 & 37.36\,$\pm$\,8.7 & 26.56\,$\pm$\,14.3 & 8.21\,$\pm$\,3.7 \\
JSD & 69.99\,$\pm$\,10.5 & 59.30\,$\pm$\,11.0 & 48.80\,$\pm$\,10.2 & 35.65\,$\pm$\,9.2 & 26.47\,$\pm$\,14.1 & 8.21\,$\pm$\,3.7 \\
Linear MMD & 70.18\,$\pm$\,10.4 & 59.39\,$\pm$\,11.1 & 51.74\,$\pm$\,\,\,\,9.2 & 38.73\,$\pm$\,8.8 & 46.16\,$\pm$\,\,\,\,5.2 & 33.78\,$\pm$\,5.3 \\
Kernelized MMD & 69.80\,$\pm$\,10.8 & 59.11\,$\pm$\,11.4 & 51.91\,$\pm$\,\,\,\,9.2 & 38.96\,$\pm$\,8.8 & 45.98\,$\pm$\,\,\,\,5.0 & 34.00\,$\pm$\,4.9 \\
\bottomrule
\end{tabular}
\end{table*}

\begin{table*}[t!]
\centering
\caption{DA results in \% (mean and standard deviation over all scenarios and five runs) for the uWave~\citep{liu_wang_zhong} dataset. \textbf{Bold} are best results.}
\label{table_results_da_all_methods5}
\footnotesize
\begin{tabular}{lrrrrrr}
\toprule
\textbf{Method}
& \makecell{\textbf{CNN}\\\textbf{Accuracy}}
& \makecell{\textbf{CNN}\\\textbf{F1}}
& \makecell{\textbf{ResNet18}\\\textbf{Accuracy}}
& \makecell{\textbf{ResNet18}\\\textbf{F1}}
& \makecell{\textbf{TCN}\\\textbf{Accuracy}}
& \makecell{\textbf{TCN}\\\textbf{F1}} \\
\midrule
MSE & 83.76\,$\pm$\,\,\,\,9.6 & 82.43\,$\pm$\,10.3 & 45.88\,$\pm$\,12.0 & 41.66\,$\pm$\,12.3 & 71.69\,$\pm$\,13.0 & 68.82\,$\pm$\,14.3 \\
CS & 77.88\,$\pm$\,12.6 & 76.09\,$\pm$\,13.6 & 48.51\,$\pm$\,11.5 & 44.81\,$\pm$\,12.0 & 69.08\,$\pm$\,14.9 & 66.29\,$\pm$\,15.5 \\
PC & 85.62\,$\pm$\,\,\,\,8.9 & 84.62\,$\pm$\,\,\,\,9.6 & 49.99\,$\pm$\,12.0 & 45.99\,$\pm$\,12.9 & 73.33\,$\pm$\,14.4 & 70.22\,$\pm$\,15.4 \\
KL & 81.35\,$\pm$\,11.5 & 79.89\,$\pm$\,12.4 & 46.41\,$\pm$\,11.4 & 40.81\,$\pm$\,12.0 & 12.69\,$\pm$\,\,\,\,2.6 & 2.81\,$\pm$\,\,\,\,0.5 \\
JSD & 84.10\,$\pm$\,11.2 & 82.72\,$\pm$\,12.1 & 49.33\,$\pm$\,10.6 & 43.11\,$\pm$\,11.5 & 12.10\,$\pm$\,\,\,\,3.0 & 2.70\,$\pm$\,\,\,\,0.6 \\
Linear MMD & 86.53\,$\pm$\,10.1 & 85.58\,$\pm$\,10.7 & 49.79\,$\pm$\,11.7 & 45.79\,$\pm$\,12.4 & 83.40\,$\pm$\,\,\,\,9.9 & 81.69\,$\pm$\,10.9 \\
Kernelized MMD & 86.57\,$\pm$\,\,\,\,9.3 & 85.53\,$\pm$\,10.1 & 49.76\,$\pm$\,12.0 & 45.78\,$\pm$\,12.7 & 79.22\,$\pm$\,13.2 & 77.28\,$\pm$\,14.0 \\
\bottomrule
\end{tabular}
\end{table*}

\begin{table*}[t!]
\centering
\caption{DA results in \% (mean and standard deviation over all scenarios and five runs) for the PenDigits~\citep{alimoglu_alpaydin} dataset. \textbf{Bold} are best results.}
\label{table_results_da_all_methods10}
\footnotesize
\begin{tabular}{lrrrrrr}
\toprule
\textbf{Method}
& \makecell{\textbf{CNN}\\\textbf{Accuracy}}
& \makecell{\textbf{CNN}\\\textbf{F1}}
& \makecell{\textbf{ResNet18}\\\textbf{Accuracy}}
& \makecell{\textbf{ResNet18}\\\textbf{F1}}
& \makecell{\textbf{TCN}\\\textbf{Accuracy}}
& \makecell{\textbf{TCN}\\\textbf{F1}} \\
\midrule
MSE & 96.34\,$\pm$\,0.2 & 96.30\,$\pm$\,0.2 & 77.15\,$\pm$\,0.8 & 77.01\,$\pm$\,1.0 & 96.16\,$\pm$\,\,\,\,0.6 & 96.10\,$\pm$\,\,\,\,0.6 \\
CS & 96.05\,$\pm$\,0.4 & 95.97\,$\pm$\,0.4 & 75.83\,$\pm$\,1.1 & 75.73\,$\pm$\,1.1 & 95.67\,$\pm$\,\,\,\,1.6 & 94.56\,$\pm$\,\,\,\,0.6 \\
PC & 96.15\,$\pm$\,0.5 & 96.10\,$\pm$\,0.5 & 77.59\,$\pm$\,0.8 & 77.49\,$\pm$\,1.0 & 95.43\,$\pm$\,\,\,\,2.0 & 95.36\,$\pm$\,\,\,\,2.1 \\
KL & 95.81\,$\pm$\,0.8 & 95.76\,$\pm$\,0.9 & 63.97\,$\pm$\,4.8 & 60.31\,$\pm$\,6.4 & 10.42\,$\pm$\,\,\,\,0.6 & 1.89\,$\pm$\,\,\,\,0.1 \\
JSD & 95.82\,$\pm$\,0.7 & 95.77\,$\pm$\,0.7 & 56.37\,$\pm$\,5.4 & 54.03\,$\pm$\,5.1 & 10.01\,$\pm$\,\,\,\,0.6 & 1.82\,$\pm$\,\,\,\,0.1 \\
Linear MMD & 96.29\,$\pm$\,0.4 & 96.23\,$\pm$\,0.4 & 77.14\,$\pm$\,0.7 & 76.94\,$\pm$\,0.9 & 88.49\,$\pm$\,13.1 & 87.31\,$\pm$\,14.9 \\
Kernelized MMD & 95.91\,$\pm$\,0.6 & 95.83\,$\pm$\,0.6 & 77.54\,$\pm$\,0.5 & 77.40\,$\pm$\,0.9 & 95.46\,$\pm$\,\,\,\,0.5 & 95.42\,$\pm$\,\,\,\,0.5 \\
\bottomrule
\end{tabular}
\end{table*}

\begin{table*}[t!]
\centering
\caption{DA results in \% (mean and standard deviation over all scenarios and five runs) for the Epilepsy~\citep{villar_vergara_menendez} dataset. \textbf{Bold} are best results.}
\label{table_results_da_all_methods8}
\footnotesize
\begin{tabular}{lrrrrrr}
\toprule
\textbf{Method}
& \makecell{\textbf{CNN}\\\textbf{Accuracy}}
& \makecell{\textbf{CNN}\\\textbf{F1}}
& \makecell{\textbf{ResNet18}\\\textbf{Accuracy}}
& \makecell{\textbf{ResNet18}\\\textbf{F1}}
& \makecell{\textbf{TCN}\\\textbf{Accuracy}}
& \makecell{\textbf{TCN}\\\textbf{F1}} \\
\midrule
MSE & 98.12\,$\pm$\,1.2 & 98.17\,$\pm$\,1.1 & 97.19\,$\pm$\,1.8 & 97.34\,$\pm$\,1.7 & 68.75\,$\pm$\,2.2 & 69.05\,$\pm$\,\,\,\,2.4 \\
CS & 84.38\,$\pm$\,6.3 & 83.81\,$\pm$\,6.5 & 92.19\,$\pm$\,1.0 & 91.89\,$\pm$\,1.1 & 66.56\,$\pm$\,3.5 & 65.33\,$\pm$\,\,\,\,4.1 \\
PC & \textbf{99.06}\,$\pm$\,1.2 & \textbf{99.09}\,$\pm$\,1.2 & \textbf{99.38}\,$\pm$\,0.8 & \textbf{99.40}\,$\pm$\,0.7 & 77.81\,$\pm$\,3.2 & 77.31\,$\pm$\,\,\,\,3.1 \\
KL & 98.12\,$\pm$\,1.2 & 98.17\,$\pm$\,1.1 & 94.06\,$\pm$\,2.7 & 93.72\,$\pm$\,2.9 & 47.19\,$\pm$\,3.9 & 39.09\,$\pm$\,\,\,\,6.1 \\
JSD & 98.12\,$\pm$\,1.2 & 98.17\,$\pm$\,1.1 & 91.25\,$\pm$\,3.6 & 91.06\,$\pm$\,3.9 & 37.50\,$\pm$\,4.0 & 21.89\,$\pm$\,\,\,\,4.5 \\
Linear MMD & 98.12\,$\pm$\,1.2 & 98.17\,$\pm$\,1.1 & 98.75\,$\pm$\,0.6 & 98.79\,$\pm$\,0.6 & 75.00\,$\pm$\,1.4 & 74.91\,$\pm$\,\,\,\,1.6 \\
Kernelized MMD & 98.12\,$\pm$\,1.2 & 98.17\,$\pm$\,1.1 & \textbf{99.38}\,$\pm$\,0.8 & \textbf{99.40}\,$\pm$\,0.7 & 78.75\,$\pm$\,1.9 & 78.97\,$\pm$\,\,\,\,1.5 \\
\bottomrule
\end{tabular}
\end{table*}

\begin{table*}[t!]
\centering
\caption{DA results in \% (mean and standard deviation over all scenarios and five runs) for the Face detection~\citep{olivetti_kia_avesani} dataset. \textbf{Bold} are best results.}
\label{table_results_da_all_methods9}
\footnotesize
\begin{tabular}{lrrrrrr}
\toprule
\textbf{Method}
& \makecell{\textbf{CNN}\\\textbf{Accuracy}}
& \makecell{\textbf{CNN}\\\textbf{F1}}
& \makecell{\textbf{ResNet18}\\\textbf{Accuracy}}
& \makecell{\textbf{ResNet18}\\\textbf{F1}}
& \makecell{\textbf{TCN}\\\textbf{Accuracy}}
& \makecell{\textbf{TCN}\\\textbf{F1}} \\
\midrule
MSE & 60.12\,$\pm$\,2.3 & 58.22\,$\pm$\,4.2 & 51.63\,$\pm$\,1.5 & 48.06\,$\pm$\,5.7 & 64.86\,$\pm$\,0.8 & 64.84\,$\pm$\,\,\,\,0.8 \\
CS & 49.16\,$\pm$\,0.0 & 32.96\,$\pm$\,0.0 & 48.92\,$\pm$\,0.4 & 38.28\,$\pm$\,4.5 & 49.16\,$\pm$\,0.0 & 32.96\,$\pm$\,\,\,\,0.0 \\
PC & 59.59\,$\pm$\,2.0 & 58.22\,$\pm$\,3.2 & 52.62\,$\pm$\,1.1 & 47.76\,$\pm$\,3.7 & 65.43\,$\pm$\,0.8 & 65.41\,$\pm$\,\,\,\,0.8 \\
KL & 59.01\,$\pm$\,1.8 & 56.36\,$\pm$\,3.1 & 51.97\,$\pm$\,0.4 & 46.44\,$\pm$\,4.4 & 50.15\,$\pm$\,0.8 & 33.40\,$\pm$\,\,\,\,0.4 \\
JSD & 59.59\,$\pm$\,2.1 & 57.96\,$\pm$\,3.8 & 50.87\,$\pm$\,1.4 & 41.62\,$\pm$\,5.9 & 50.52\,$\pm$\,0.6 & 33.59\,$\pm$\,\,\,\,0.2 \\
Linear MMD & 63.20\,$\pm$\,1.2 & \textbf{63.06}\,$\pm$\,1.3 & 53.31\,$\pm$\,1.1 & \textbf{52.74}\,$\pm$\,1.5 & 64.66\,$\pm$\,0.8 & 64.58\,$\pm$\,\,\,\,0.8 \\
Kernelized MMD & 59.47\,$\pm$\,0.6 & 58.17\,$\pm$\,1.2 & 53.31\,$\pm$\,0.5 & 51.94\,$\pm$\,1.0 & 65.47\,$\pm$\,0.8 & 65.46\,$\pm$\,\,\,\,0.8 \\
\bottomrule
\end{tabular}
\end{table*}

\begin{table*}[t!]
\centering
\caption{DA results in \% (mean and standard deviation over all scenarios and five runs) for the Finger movements~\citep{blankertz_curio_mueller} dataset. \textbf{Bold} are best results.}
\label{table_results_da_all_methods6}
\footnotesize
\begin{tabular}{lrrrrrr}
\toprule
\textbf{Method}
& \makecell{\textbf{CNN}\\\textbf{Accuracy}}
& \makecell{\textbf{CNN}\\\textbf{F1}}
& \makecell{\textbf{ResNet18}\\\textbf{Accuracy}}
& \makecell{\textbf{ResNet18}\\\textbf{F1}}
& \makecell{\textbf{TCN}\\\textbf{Accuracy}}
& \makecell{\textbf{TCN}\\\textbf{F1}} \\
\midrule
MSE & 50.18\,$\pm$\,4.5 & 45.62\,$\pm$\,5.1 & 49.58\,$\pm$\,4.3 & 43.72\,$\pm$\,4.9 & 53.66\,$\pm$\,4.0 & 46.03\,$\pm$\,7.1 \\
CS & \textbf{53.27}\,$\pm$\,1.6 & 40.34\,$\pm$\,3.5 & \textbf{53.47}\,$\pm$\,2.0 & 39.03\,$\pm$\,4.5 & 53.26\,$\pm$\,4.2 & 49.53\,$\pm$\,5.6 \\
PC & 50.50\,$\pm$\,4.4 & 45.71\,$\pm$\,5.2 & 49.69\,$\pm$\,4.4 & 43.14\,$\pm$\,5.5 & 52.78\,$\pm$\,4.6 & 48.21\,$\pm$\,5.8 \\
KL & 50.18\,$\pm$\,4.6 & 45.53\,$\pm$\,5.2 & 51.26\,$\pm$\,4.2 & 41.59\,$\pm$\,5.4 & 50.90\,$\pm$\,5.0 & 36.00\,$\pm$\,6.8 \\
JSD & 50.18\,$\pm$\,4.5 & 45.68\,$\pm$\,5.2 & 51.23\,$\pm$\,4.0 & 41.48\,$\pm$\,5.9 & 50.80\,$\pm$\,5.0 & 35.69\,$\pm$\,6.4 \\
Linear MMD & 50.56\,$\pm$\,4.3 & 45.96\,$\pm$\,5.0 & 49.51\,$\pm$\,4.3 & 43.56\,$\pm$\,5.0 & 54.48\,$\pm$\,2.7 & 39.47\,$\pm$\,5.6 \\
Kernelized MMD & 50.36\,$\pm$\,4.3 & 45.97\,$\pm$\,5.1 & 49.67\,$\pm$\,4.4 & 43.28\,$\pm$\,5.5 & 52.63\,$\pm$\,4.6 & 48.05\,$\pm$\,5.8 \\
\bottomrule
\end{tabular}
\end{table*}

\begin{table*}[t!]
\centering
\caption{DA results in \% (mean and standard deviation over all scenarios and five runs) for the Gestures mid air \citep{caputo_prebianca} dataset. \textbf{Bold} are best results.}
\label{table_results_da_all_methods7}
\footnotesize
\begin{tabular}{lrrrrrr}
\toprule
\textbf{Method}
& \makecell{\textbf{CNN}\\\textbf{Accuracy}}
& \makecell{\textbf{CNN}\\\textbf{F1}}
& \makecell{\textbf{ResNet18}\\\textbf{Accuracy}}
& \makecell{\textbf{ResNet18}\\\textbf{F1}}
& \makecell{\textbf{TCN}\\\textbf{Accuracy}}
& \makecell{\textbf{TCN}\\\textbf{F1}} \\
\midrule
MSE & 42.56\,$\pm$\,14.5 & 37.83\,$\pm$\,14.9 & 37.32\,$\pm$\,14.3 & 32.65\,$\pm$\,14.7 & 43.04\,$\pm$\,15.0 & 38.05\,$\pm$\,15.8 \\
CS & 30.52\,$\pm$\,12.4 & 23.83\,$\pm$\,13.1 & 35.48\,$\pm$\,14.5 & 30.39\,$\pm$\,15.0 & 20.26\,$\pm$\,13.1 & 13.08\,$\pm$\,13.6 \\
PC & 41.96\,$\pm$\,14.2 & 37.12\,$\pm$\,14.8 & 37.21\,$\pm$\,14.9 & 32.43\,$\pm$\,15.2 & 43.49\,$\pm$\,14.7 & 38.18\,$\pm$\,15.6 \\
KL & 42.23\,$\pm$\,14.3 & 37.40\,$\pm$\,15.0 & 35.60\,$\pm$\,14.2 & 30.46\,$\pm$\,14.1 & 26.26\,$\pm$\,19.5 & 20.44\,$\pm$\,19.7 \\
JSD & 41.50\,$\pm$\,14.8 & 36.56\,$\pm$\,15.4 & 34.97\,$\pm$\,13.0 & 29.57\,$\pm$\,13.0 & 8.78\,$\pm$\,\,\,\,8.3 & 2.58\,$\pm$\,\,\,\,7.9 \\
Linear MMD & 42.43\,$\pm$\,14.3 & 37.47\,$\pm$\,14.7 & 37.03\,$\pm$\,14.9 & 32.30\,$\pm$\,15.1 & 40.42\,$\pm$\,14.4 & 35.25\,$\pm$\,14.8 \\
Kernelized MMD & 41.81\,$\pm$\,14.7 & 37.10\,$\pm$\,15.1 & 37.73\,$\pm$\,14.3 & 32.72\,$\pm$\,14.6 & 42.78\,$\pm$\,15.0 & 37.80\,$\pm$\,15.7 \\
\bottomrule
\end{tabular}
\end{table*}

\begin{table*}[t!]
\centering
\caption{DA results in \% (mean and standard deviation over all scenarios and five runs) for the OnHW-symbols, split OnHW-equations, and OnHW-chars~\citep{ott_imwut,ott_ijdar} datasets. CNN backbone as encoder network. \textbf{Bold} are best results.}
\label{table_results_da_all_methods11}
\footnotesize
\begin{tabular}{lrrrrrr}
\toprule
\textbf{Method}
& \makecell{\textbf{Symbols}\\\textbf{L$\rightarrow$R} \textbf{[Acc.]}}
& \makecell{\textbf{Symbols}\\\textbf{L$\rightarrow$R} \textbf{[F1]}}
& \makecell{\textbf{Symbols}\\\textbf{R$\rightarrow$L} \textbf{[Acc.]}}
& \makecell{\textbf{Symbols}\\\textbf{R$\rightarrow$L} \textbf{[F1]}}
& \makecell{\textbf{Equations}\\\textbf{L$\rightarrow$R} \textbf{[Acc.]}}
& \makecell{\textbf{Equations}\\\textbf{L$\rightarrow$R} \textbf{[F1]}} \\
\midrule
MSE & 77.81\,$\pm$\,\,\,\,4.3 & 78.56\,$\pm$\,\,\,\,5.5 & 79.33\,$\pm$\,3.0 & 78.96\,$\pm$\,3.3 & 85.29\,$\pm$\,\,\,\,1.7 & 84.39\,$\pm$\,\,\,\,1.8 \\
PC & 77.19\,$\pm$\,\,\,\,2.8 & 77.08\,$\pm$\,\,\,\,5.7 & 80.13\,$\pm$\,2.7 & 79.46\,$\pm$\,3.4 & 78.96\,$\pm$\,\,\,\,6.6 & 76.81\,$\pm$\,\,\,\,7.3 \\
KL & 78.12\,$\pm$\,\,\,\,3.7 & 79.16\,$\pm$\,\,\,\,4.3 & 80.67\,$\pm$\,4.1 & 80.55\,$\pm$\,4.7 & 84.46\,$\pm$\,\,\,\,3.5 & 83.07\,$\pm$\,\,\,\,3.9 \\
JSD & 76.88\,$\pm$\,\,\,\,4.0 & 77.18\,$\pm$\,\,\,\,5.8 & 78.84\,$\pm$\,3.1 & 78.56\,$\pm$\,3.6 & 10.74\,$\pm$\,\,\,\,6.2 & 6.34\,$\pm$\,\,\,\,4.8 \\
Linear MMD & 76.56\,$\pm$\,\,\,\,2.9 & 75.88\,$\pm$\,\,\,\,3.2 & 81.79\,$\pm$\,2.1 & 81.41\,$\pm$\,2.7 & 88.33\,$\pm$\,\,\,\,0.9 & 87.78\,$\pm$\,\,\,\,1.3 \\
Kernelized MMD & 78.44\,$\pm$\,\,\,\,2.0 & 79.75\,$\pm$\,\,\,\,1.7 & 79.78\,$\pm$\,3.2 & 79.39\,$\pm$\,3.8 & 87.00\,$\pm$\,\,\,\,0.6 & 86.25\,$\pm$\,\,\,\,0.8 \\
Deep CORAL & 84.23\,$\pm$\,\,\,\,1.0 & 83.94\,$\pm$\,\,\,\,0.9 & 77.73\,$\pm$\,5.9 & 77.72\,$\pm$\,5.6 & 84.66\,$\pm$\,\,\,\,1.7 & 83.61\,$\pm$\,\,\,\,1.9 \\
\bottomrule
\end{tabular}
\end{table*}

\paragraph{Aggregation of Results.} For a fixed dataset, backbone, and adaptation method, let $m_{s,r}$ denote the target accuracy or macro-F1 score obtained for source--target scenario $s\in\{1,\ldots,S\}$ and random-seed run $r\in\{1,\ldots,R\}$, with $R=5$. The values reported in the aggregate result tables are computed by pooling all $SR$ scenario--run observations:
\begin{equation}
    \overline{m}
    =
    \frac{1}{SR}
    \sum_{s=1}^{S}\sum_{r=1}^{R}m_{s,r},
\end{equation}
and
\begin{equation}
    \sigma_{\mathrm{pool}}
    =
    \sqrt{
    \frac{1}{SR}
    \sum_{s=1}^{S}\sum_{r=1}^{R}
    \left(m_{s,r}-\overline{m}\right)^2
    }.
\end{equation}
Thus, $\sigma_{\mathrm{pool}}$ is the descriptive population standard deviation over all evaluated scenario--seed combinations, consistent with the default normalization of \texttt{numpy.std}. Because every source--target scenario is evaluated with the same number of runs, all scenarios receive equal weight. The reported deviation jointly reflects variation between source--target scenarios and stochastic variation between the five runs; it should therefore not be interpreted as the seed-induced standard deviation for a single fixed transfer task or as a confidence interval. Equivalently, since each scenario has the same number of runs,
\begin{equation}
    \sigma_{\mathrm{pool}}^2
    =
    \frac{1}{S}\sum_{s=1}^{S}\sigma_s^2
    +
    \frac{1}{S}\sum_{s=1}^{S}
    \left(\overline{m}_s-\overline{m}\right)^2,
\end{equation}
where $\overline{m}_s$ and $\sigma_s$ are the mean and population standard deviation across the five runs of scenario $s$. The first term captures average within-scenario seed variability, whereas the second captures variability between source--target scenarios.

\paragraph{Evaluation of DA Methods on Standard Time-Series Datasets.} Tables~\ref{table_results_da_all_methods1}--\ref{table_results_da_all_methods11} show that CPDA is most effective on datasets where the domain shift preserves a class-dependent temporal structure, such as human activity, gesture, and trajectory data. On HHAR, CPDA achieves the best CNN accuracy with $81.82\%$, outperforming DIRT-T ($80.11\%$) and CDAN ($79.88\%$), and also obtains the best TCN accuracy with $77.88\%$ compared with $76.78\%$ for CDAN and DIRT-T. This indicates that class-conditional path alignment is beneficial when user- or sensor-dependent shifts affect the temporal dynamics but the class-wise movement patterns remain transferable. For ResNet18, however, Sinkhorn is marginally better ($67.98\%$ vs.~$67.79\%$), suggesting that direct transport-based matching can be competitive when the learned representation is less explicitly temporal.

On UCI HAR and WISDM, CPDA also performs particularly well. On UCI HAR, CPDA obtains the best ResNet18 and TCN accuracies with $89.53\%$ and $90.31\%$, respectively, improving over DSAN ($88.17\%$ with ResNet18) and HoMM ($87.90\%$ with TCN). With CNN, DIRT-T remains slightly stronger ($92.02\%$ vs.~$91.88\%$), which suggests that adversarial entropy-regularized adaptation can still be highly effective when the CNN representation is already well separated. On WISDM, CPDA achieves the best accuracy for all three backbones, with $63.92\%$ for CNN, $63.18\%$ for ResNet18, and $61.06\%$ for TCN. The improvement over the strongest alternatives is moderate but consistent, indicating that CPDA improves robustness across architectures by aligning class-conditional rather than purely marginal distributions.

The strongest evidence for CPDA appears on uWave and PenDigits, both of which contain trajectory-like or gesture-like signals. On uWave, CPDA achieves the best accuracy for all backbones: $97.30\%$ with CNN, $61.32\%$ with ResNet18, and $92.17\%$ with TCN. These results improve over strong baselines such as CDAN/DIRT-T on CNN ($96.43\%$), DSAN on ResNet18 ($59.25\%$), and DSAN on TCN ($91.92\%$). Similarly, on PenDigits, CPDA obtains the best CNN and TCN accuracies with $96.98\%$ and $97.08\%$, respectively. Since both datasets are characterized by temporal trajectories, these improvements support the motivation of CPDA: matching class-conditional latent paths can capture discriminative temporal evolution better than marginal feature alignment alone. However, CoDATS remains clearly stronger for PenDigits with ResNet18 ($83.68\%$ vs.~$75.30\%$), suggesting that CPDA is sensitive to the interaction between the backbone representation and the path-alignment objective.

For EEG, Epilepsy, and Face Detection, the results are more mixed. On EEG, CPDA obtains the best CNN accuracy ($74.39\%$) and the best ResNet18 accuracy ($53.27\%$), but it does not dominate the TCN setting, where CDAN obtains the highest accuracy ($47.88\%$). This suggests that CPDA helps when the representation preserves stable class-conditioned structure, but EEG remains challenging because noisy or imbalanced target predictions can weaken pseudo-label-based alignment. On Epilepsy, CPDA is not best for CNN or ResNet18, where PC~\citep{karl_pearson} and several kernel\,/\,covariance methods reach around $99\%$ accuracy, but it strongly improves the TCN result, achieving $85.62\%$ accuracy compared with $80.00\%$ for DIRT-T. On Face Detection, CPDA obtains the best CNN and ResNet18 accuracies with $63.70\%$ and $54.05\%$, respectively, but TCN is better served by global or higher-order representation matching, with CoDATS reaching $68.29\%$ and HoMM reaching $67.87\%$.

In contrast, CPDA is less effective on Finger Movements, Gestures Mid Air, and the OnHW handwriting tasks. On Finger Movements, CPDA reaches $50.81\%$, $49.86\%$, and $52.70\%$ accuracy for CNN, ResNet18, and TCN, whereas the best results are obtained by CS for CNN and ResNet18 ($53.27\%$ and $53.47\%$) and by Sinkhorn for TCN ($55.07\%$). This dataset is small and close to chance-level performance for many methods, so target pseudo-labels may be too unreliable for class-conditional alignment. On Gestures Mid Air, CPDA performs substantially below the best methods, with $38.20\%$, $31.00\%$, and $21.00\%$ accuracy, while MMDA, DIRT-T, and HoMM reach $46.16\%$, $45.93\%$, and $47.65\%$, respectively. This suggests that the class-conditional CPDA objective can be harmed when target pseudo-labels are unstable across many gesture classes or when inter-subject variability is high. For OnHW, CPDA does not consistently outperform the strongest baselines; for example, on OnHW-symbols L$\rightarrow$R, CPDA obtains $78.12\%$, while OT with kernelized MMD reaches $85.09\%$, and on split OnHW-equations L$\rightarrow$R, CPDA obtains $82.23\%$, while linear DAN reaches $88.37\%$. Overall, CPDA is strongest on structured temporal and trajectory-like shifts, whereas simpler moment, transport, adversarial, or covariance-based methods can remain preferable when the domain shift is dominated by small-sample effects, highly variable target pseudo-labels, or non-temporal distributional changes.

\begin{figure*}[!t]
    \centering
	\begin{minipage}[t]{0.325\linewidth}
        \centering
    	\includegraphics[trim=10 10 10 10, clip, width=1.0\linewidth]{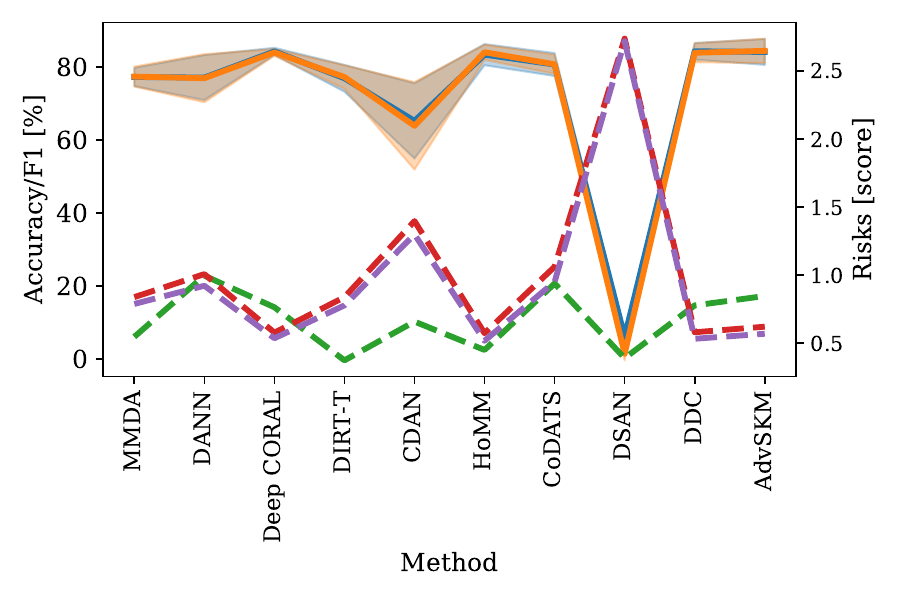}
        \subcaption{OnHW-symbols (L $\rightarrow$ R).}
        \label{label_figure_results_onhw1}
    \end{minipage}
    \hfill
	\begin{minipage}[t]{0.325\linewidth}
        \centering
    	\includegraphics[trim=10 10 10 10, clip, width=1.0\linewidth]{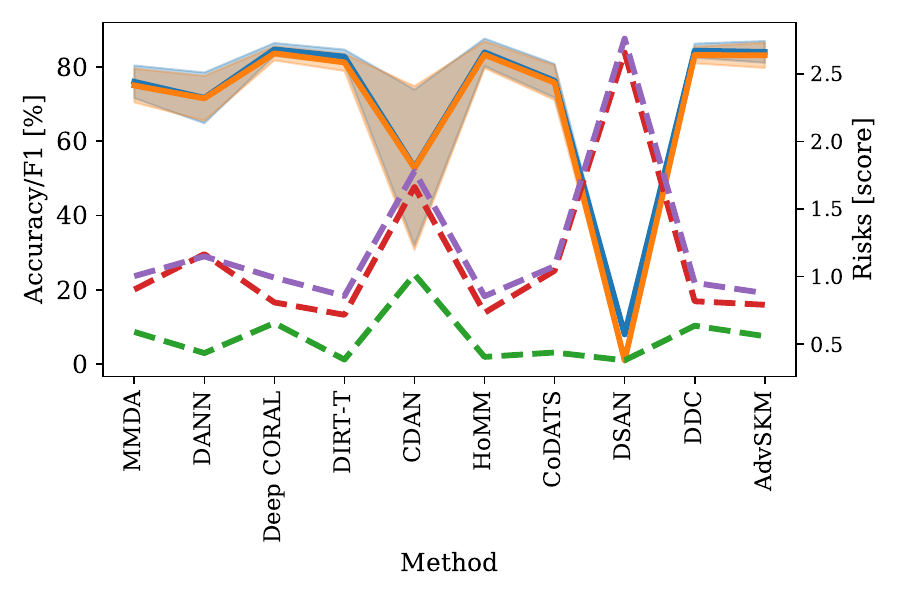}
        \subcaption{Split OnHW-equations (L $\rightarrow$ R).}
        \label{label_figure_results_onhw2}
    \end{minipage}
    \hfill
	\begin{minipage}[t]{0.325\linewidth}
        \centering
    	\includegraphics[trim=10 10 10 10, clip, width=1.0\linewidth]{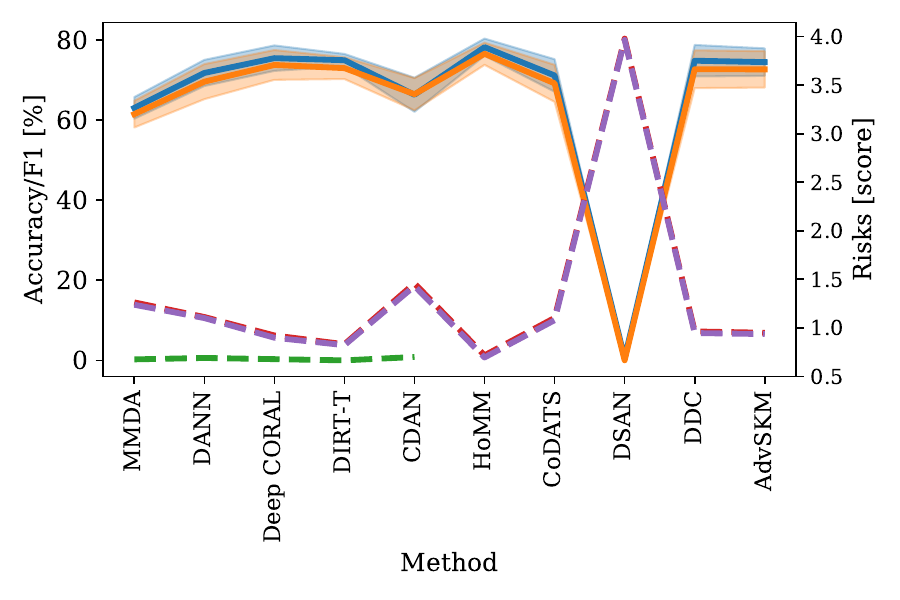}
        \subcaption{OnHW-chars (L $\rightarrow$ R).}
        \label{label_figure_results_onhw3}
    \end{minipage}
	\begin{minipage}[t]{0.325\linewidth}
        \centering
    	\includegraphics[trim=10 10 10 10, clip, width=1.0\linewidth]{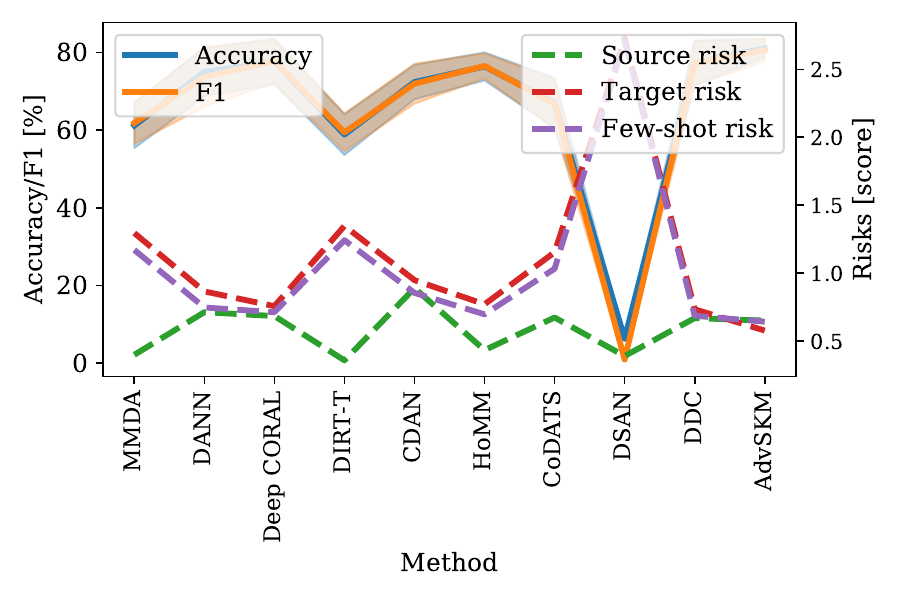}
        \subcaption{OnHW-symbols (R $\rightarrow$ L).}
        \label{label_figure_results_onhw4}
    \end{minipage}
    \hfill
	\begin{minipage}[t]{0.325\linewidth}
        \centering
    	\includegraphics[trim=10 10 10 10, clip, width=1.0\linewidth]{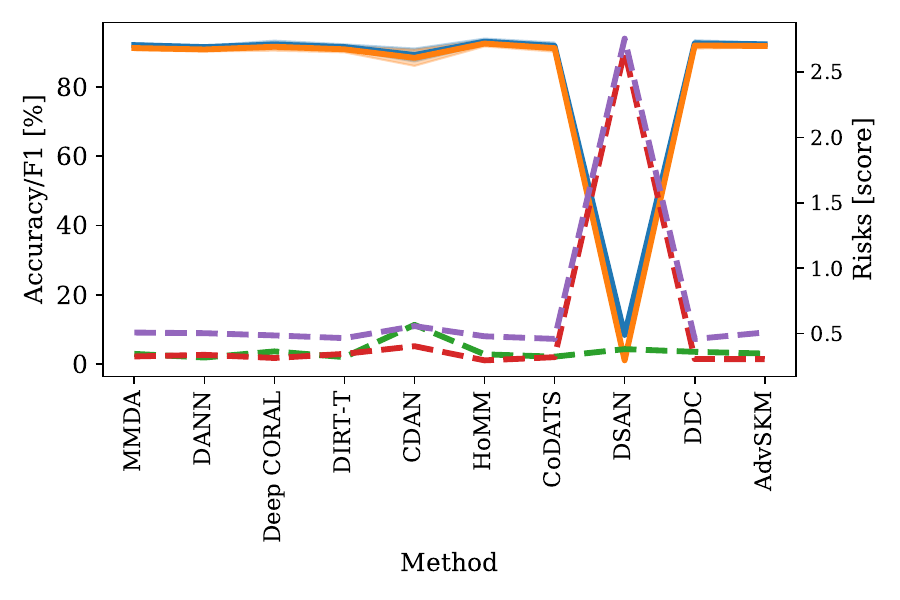}
        \subcaption{Split OnHW-equations (R $\rightarrow$ L).}
        \label{label_figure_results_onhw5}
    \end{minipage}
    \hfill
	\begin{minipage}[t]{0.325\linewidth}
        \centering
    	\includegraphics[trim=10 10 10 10, clip, width=1.0\linewidth]{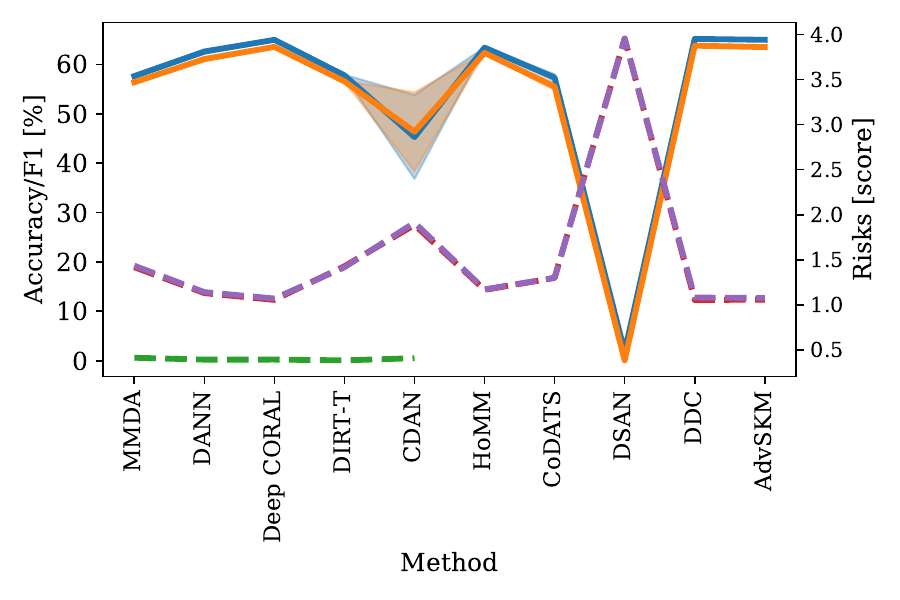}
        \subcaption{OnHW-chars (R $\rightarrow$ L).}
        \label{label_figure_results_onhw6}
    \end{minipage}
    \caption{Comparison of the recent DA methods on the OnHW datasets for the adaptation from left- to right-handed writers (L $\rightarrow$ R) and the adaptation from right- to left-handed writers (R $\rightarrow$ L).}
    \label{label_figure_results_onhw}
\end{figure*}

Figure~\ref{label_figure_results_onhw} shows that most recent DA methods achieve stable performance on the OnHW-symbols and split OnHW-equations tasks, whereas the OnHW-chars setting is more challenging. Across the symbol and equation transfers, MMDA, DANN, Deep CORAL, DIRT-T, HoMM, CoDATS, DDC, and AdvSKM generally obtain high accuracy and F1-scores, often in the range of approximately $75$--$90\%$, indicating that these methods can handle the handedness-induced domain shift when the number of classes is moderate. Deep CORAL, HoMM, DDC, and AdvSKM are particularly robust, as they combine high target performance with relatively low source, target, and few-shot risks. In contrast, CDAN is less stable and shows clear performance drops in several settings, especially on the symbol and character tasks, suggesting that adversarial conditional alignment may be sensitive to the target-domain structure in this benchmark. DSAN consistently fails; this indicates severe instability or negative transfer under the OnHW domain shift. The OnHW-chars tasks are more challenging than the symbols and equations tasks, with lower overall accuracy and larger differences between methods, which is expected because character recognition involves a much larger class space and more fine-grained writing variations. Overall, the figure suggests that covariance-, moment-, and higher-order alignment methods such as Deep CORAL, HoMM, DDC, and AdvSKM are the most reliable among the compared recent DA baselines on OnHW, while DSAN and, in some settings, CDAN are less suitable for this handwriting-based DA scenario.

\begin{figure*}[!t]
    \centering
	\begin{minipage}[t]{0.495\linewidth}
    	\begin{minipage}[t]{0.325\linewidth}
            \centering
        	\includegraphics[trim=11 45 11 11, clip, width=1.0\linewidth]{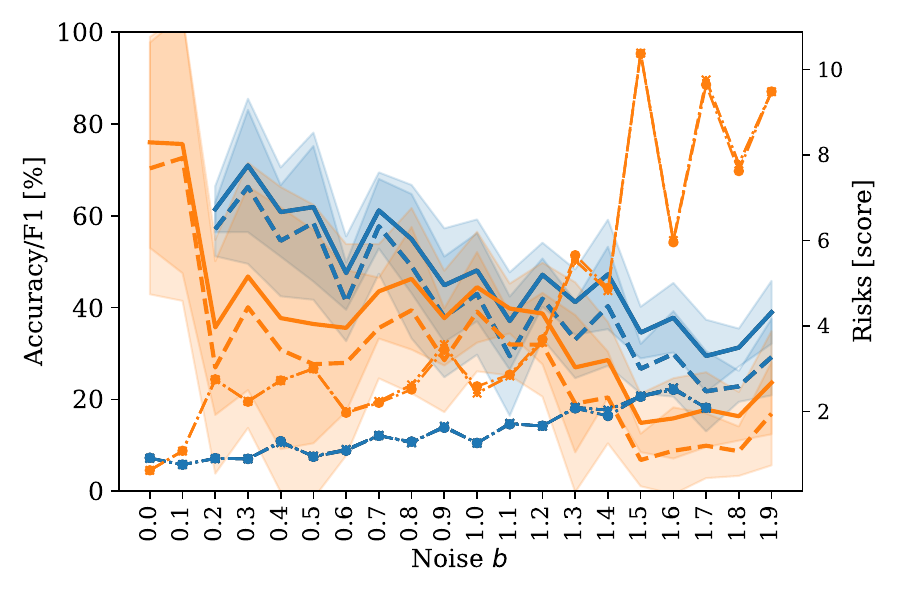}
        \end{minipage}
        \hfill
    	\begin{minipage}[t]{0.325\linewidth}
            \centering
        	\includegraphics[trim=11 45 11 11, clip, width=1.0\linewidth]{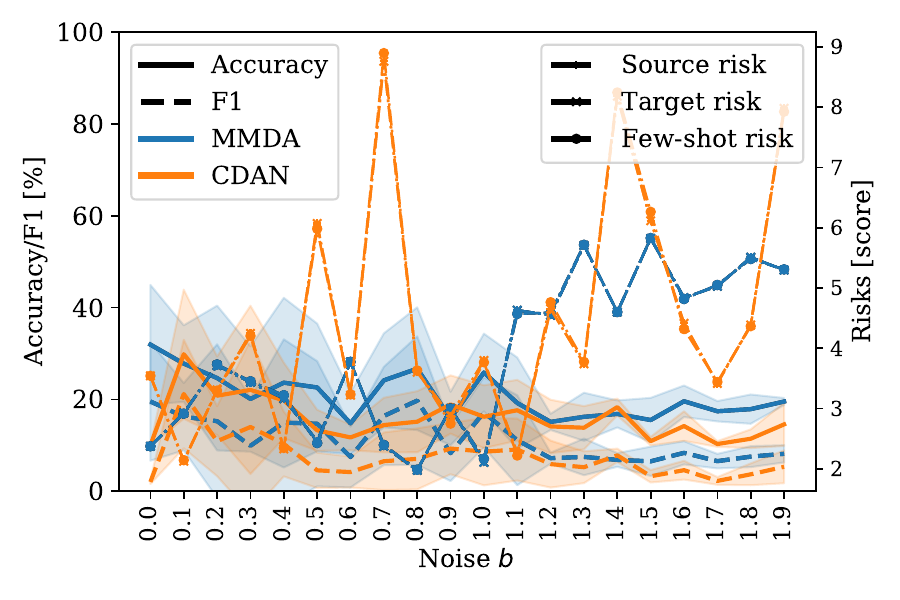}
        \end{minipage}
        \hfill
    	\begin{minipage}[t]{0.325\linewidth}
            \centering
        	\includegraphics[trim=11 45 11 11, clip, width=1.0\linewidth]{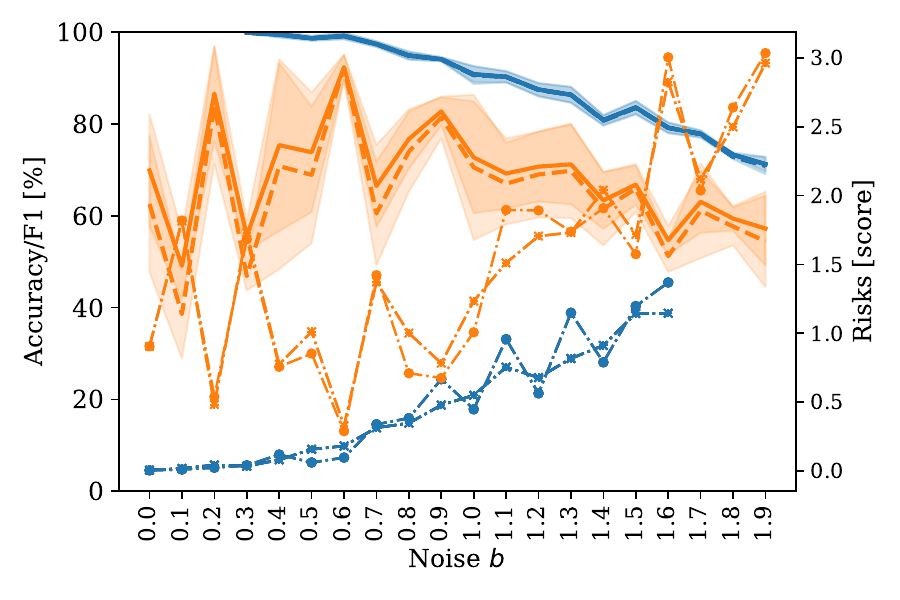}
        \end{minipage}
        \subcaption{Methods MMDA and CDAN.}
        \label{figure_results_sin_cos1}
    \end{minipage}
    \hfill
	\begin{minipage}[t]{0.495\linewidth}
    	\begin{minipage}[t]{0.325\linewidth}
            \centering
        	\includegraphics[trim=11 45 11 11, clip, width=1.0\linewidth]{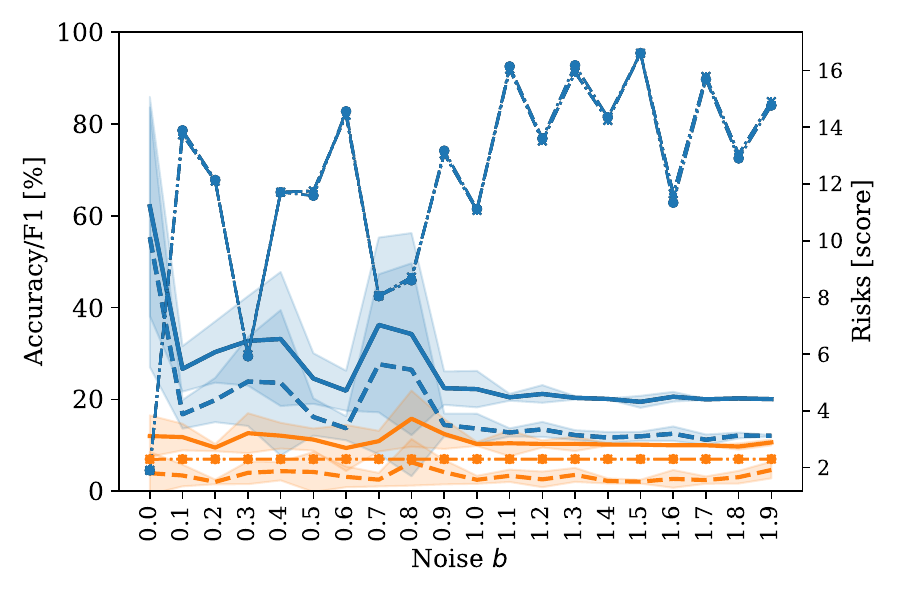}
        \end{minipage}
        \hfill
    	\begin{minipage}[t]{0.325\linewidth}
            \centering
        	\includegraphics[trim=11 45 11 11, clip, width=1.0\linewidth]{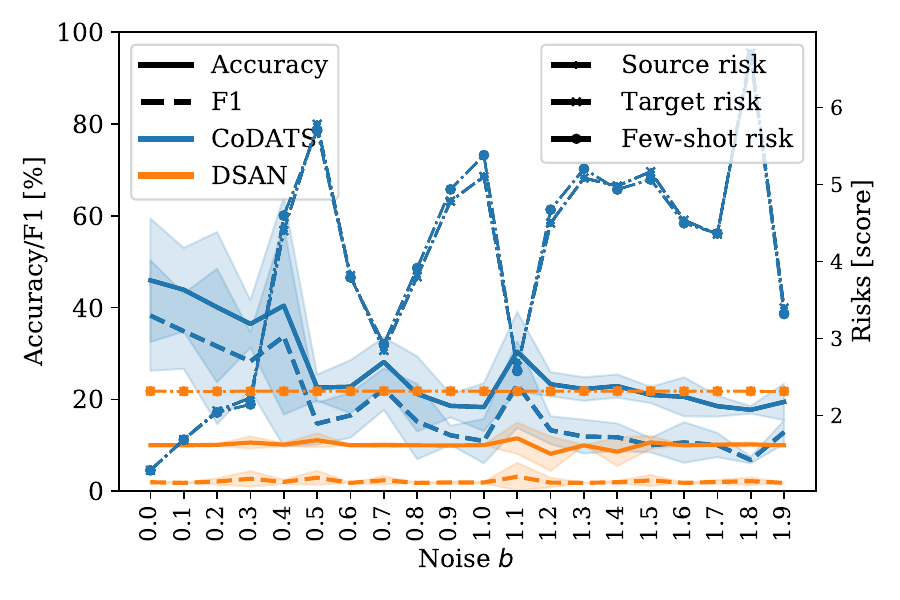}
        \end{minipage}
        \hfill
    	\begin{minipage}[t]{0.325\linewidth}
            \centering
        	\includegraphics[trim=11 45 11 11, clip, width=1.0\linewidth]{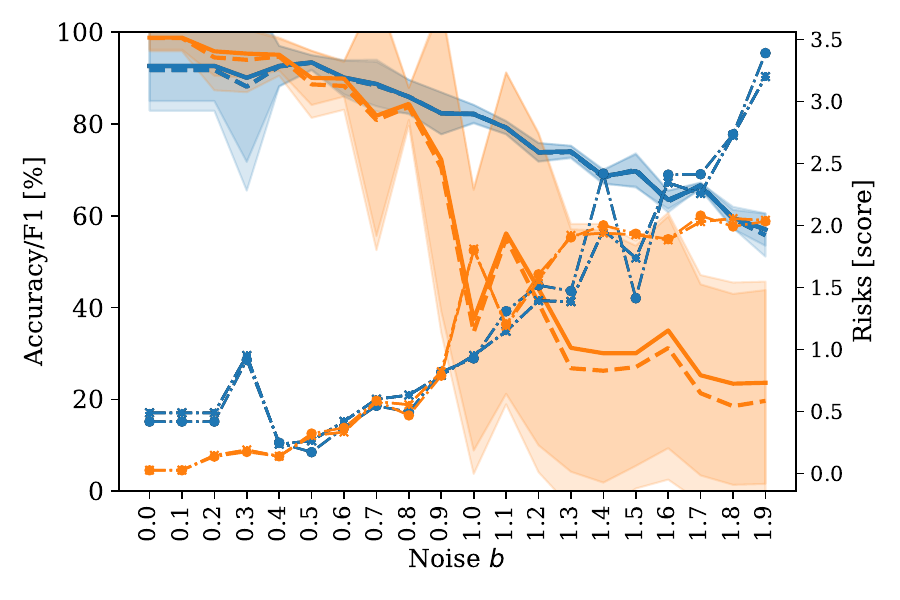}
        \end{minipage}
        \subcaption{Methods CoDATS and DSAN.}
        \label{figure_results_sin_cos2}
    \end{minipage}
	\begin{minipage}[t]{0.495\linewidth}
    	\begin{minipage}[t]{0.325\linewidth}
            \centering
        	\includegraphics[trim=11 45 11 11, clip, width=1.0\linewidth]{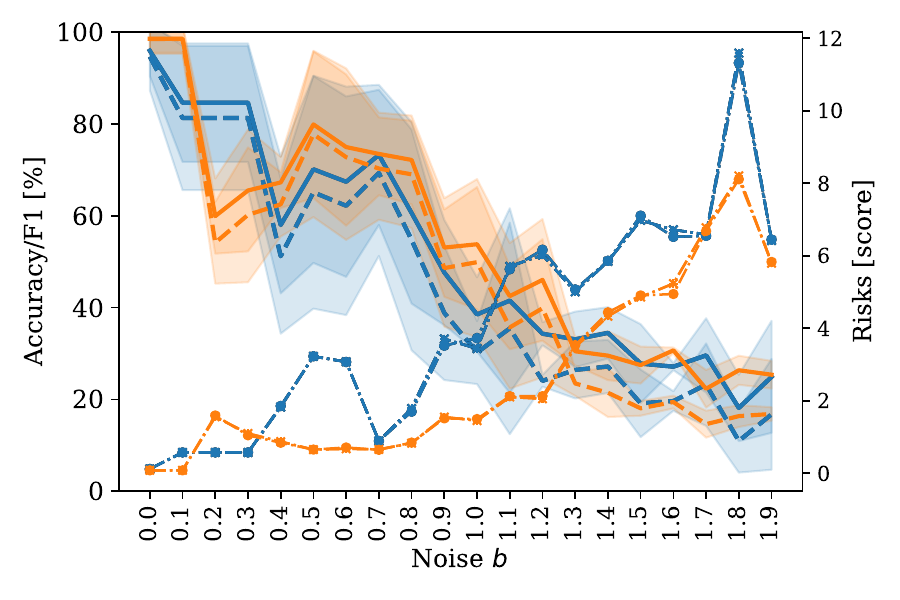}
        \end{minipage}
        \hfill
    	\begin{minipage}[t]{0.325\linewidth}
            \centering
        	\includegraphics[trim=11 45 11 11, clip, width=1.0\linewidth]{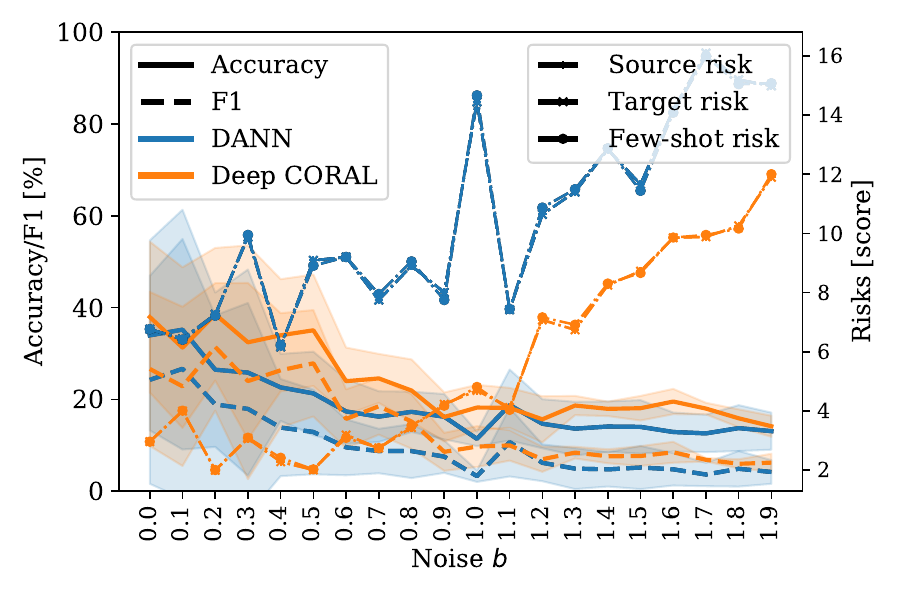}
        \end{minipage}
        \hfill
    	\begin{minipage}[t]{0.325\linewidth}
            \centering
        	\includegraphics[trim=11 45 11 11, clip, width=1.0\linewidth]{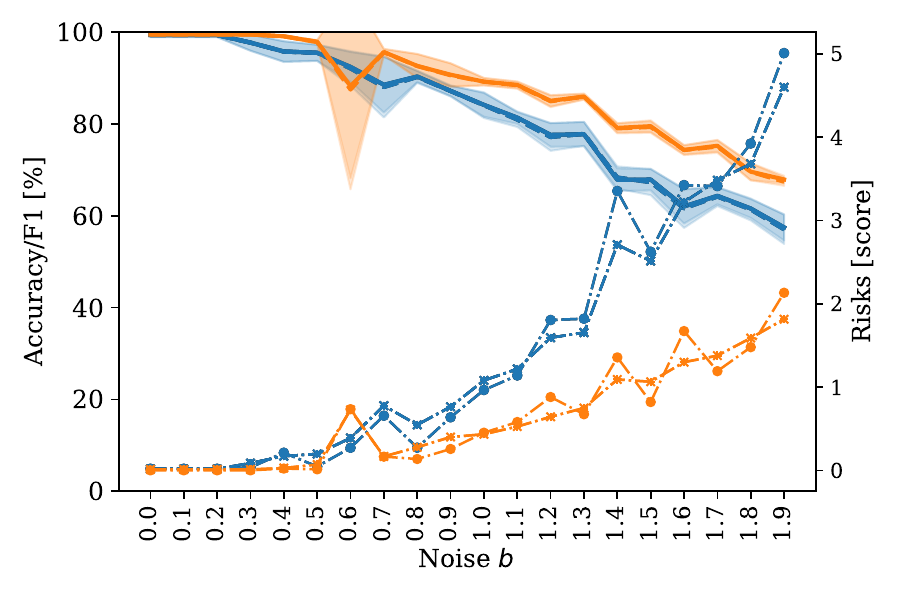}
        \end{minipage}
        \subcaption{Methods DANN and Deep CORAL.}
        \label{figure_results_sin_cos3}
    \end{minipage}
    \hfill
	\begin{minipage}[t]{0.495\linewidth}
    	\begin{minipage}[t]{0.325\linewidth}
            \centering
        	\includegraphics[trim=11 45 11 11, clip, width=1.0\linewidth]{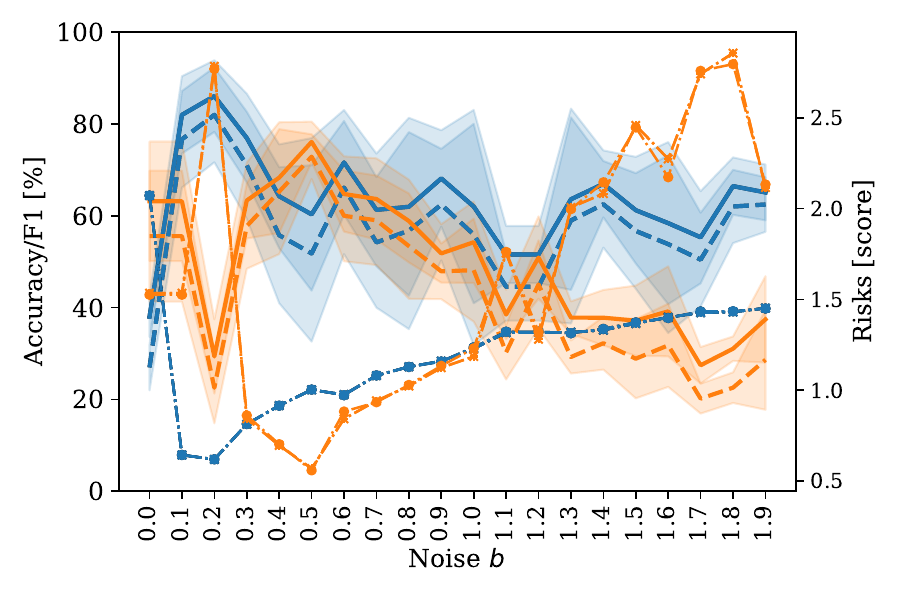}
        \end{minipage}
        \hfill
    	\begin{minipage}[t]{0.325\linewidth}
            \centering
        	\includegraphics[trim=11 45 11 11, clip, width=1.0\linewidth]{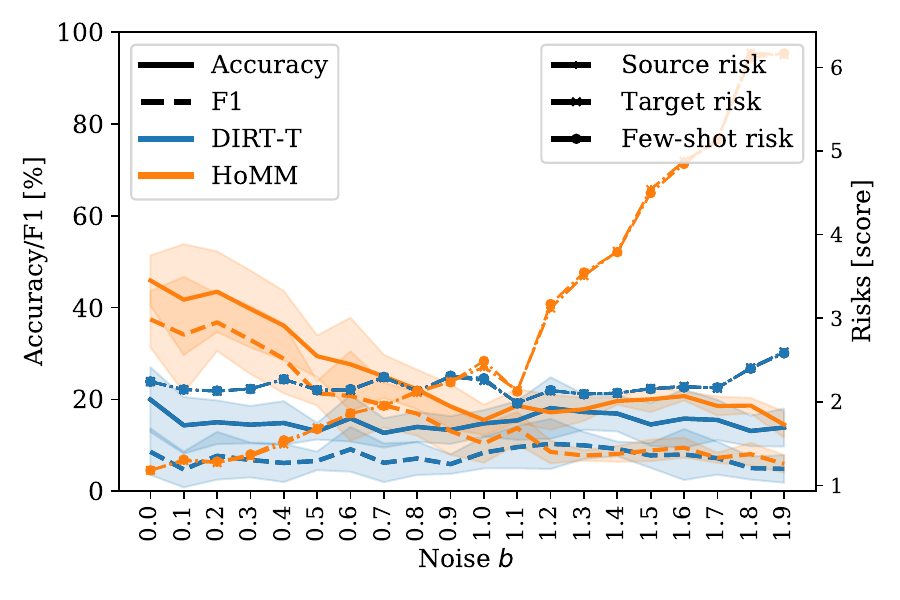}
        \end{minipage}
        \hfill
    	\begin{minipage}[t]{0.325\linewidth}
            \centering
        	\includegraphics[trim=11 45 11 11, clip, width=1.0\linewidth]{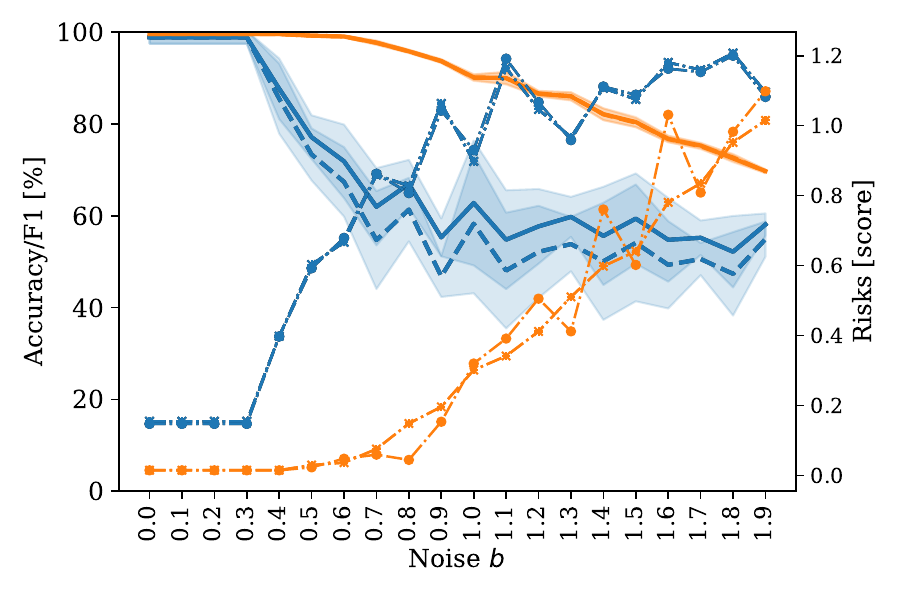}
        \end{minipage}
        \subcaption{Methods DIRT-T and $\text{HoMM}_{p=3}$.}
        \label{figure_results_sin_cos4}
    \end{minipage}
	\begin{minipage}[t]{0.495\linewidth}
    	\begin{minipage}[t]{0.325\linewidth}
            \centering
        	\includegraphics[trim=11 45 11 11, clip, width=1.0\linewidth]{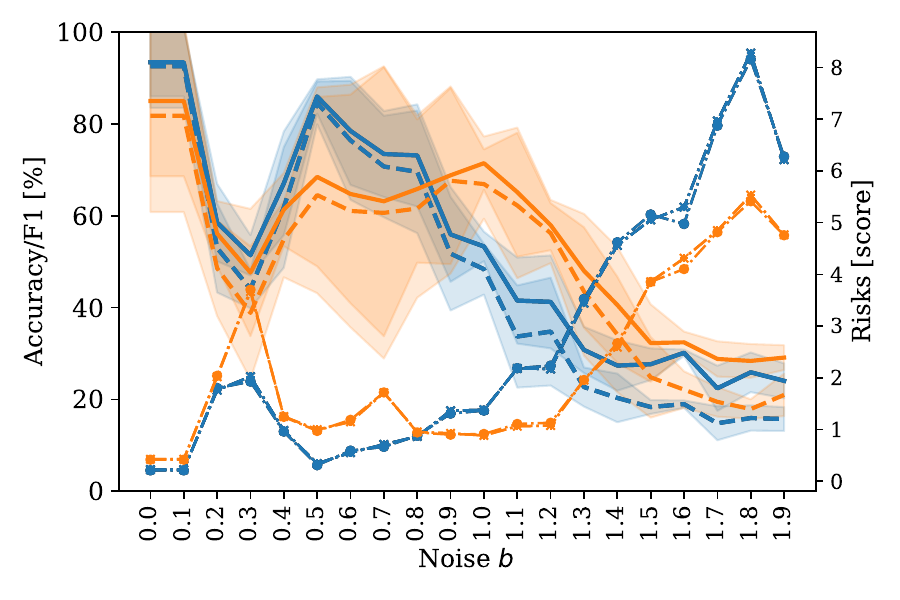}
        \end{minipage}
        \hfill
    	\begin{minipage}[t]{0.325\linewidth}
            \centering
        	\includegraphics[trim=11 45 11 11, clip, width=1.0\linewidth]{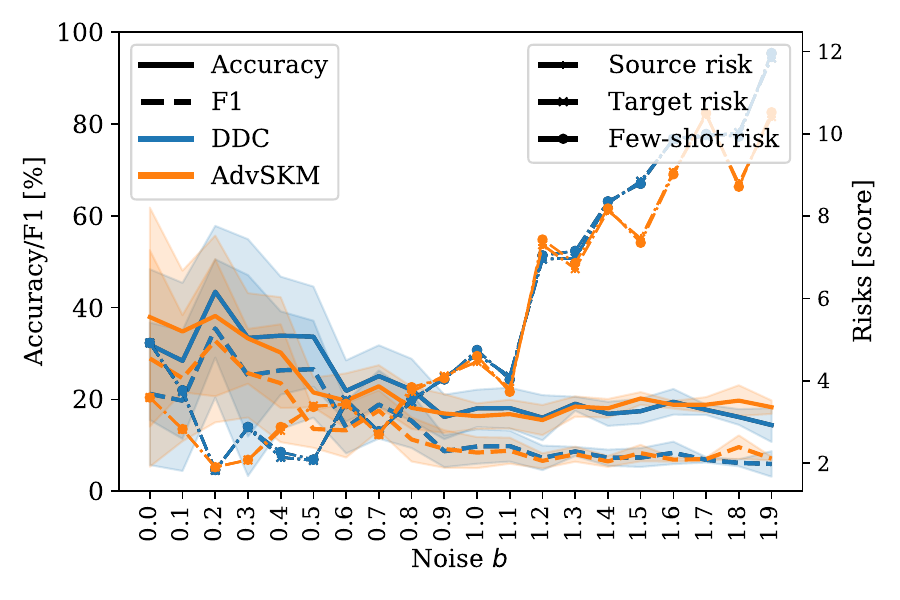}
        \end{minipage}
        \hfill
    	\begin{minipage}[t]{0.325\linewidth}
            \centering
        	\includegraphics[trim=11 45 11 11, clip, width=1.0\linewidth]{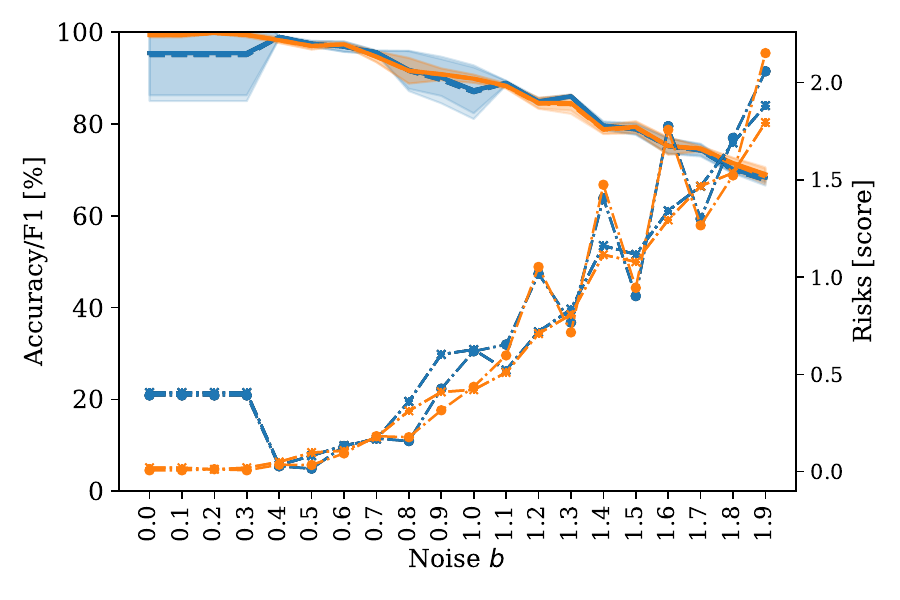}
        \end{minipage}
        \subcaption{Methods DDC and AdvSKM.}
        \label{figure_results_sin_cos5}
    \end{minipage}
    \hfill
	\begin{minipage}[t]{0.495\linewidth}
    	\begin{minipage}[t]{0.325\linewidth}
            \centering
        	\includegraphics[trim=11 45 11 11, clip, width=1.0\linewidth]{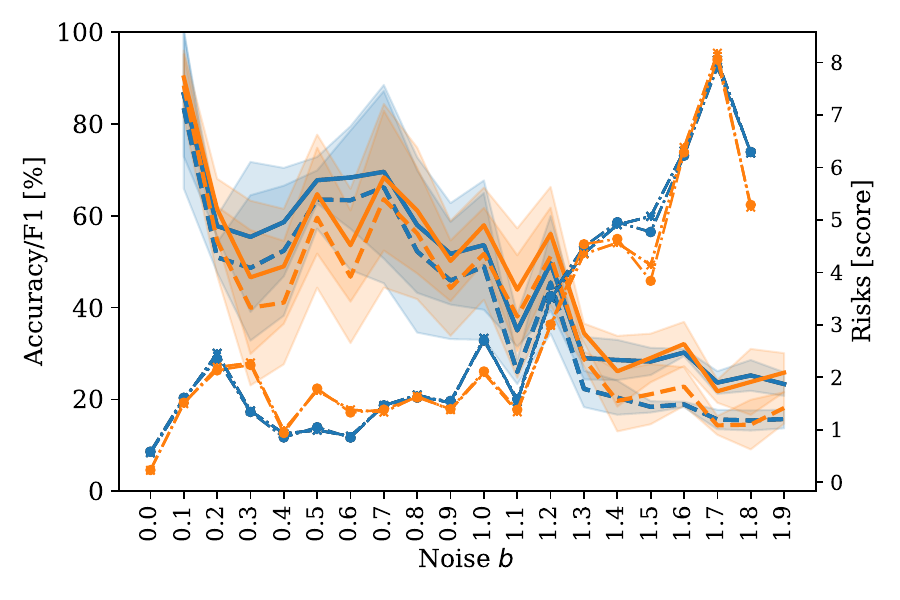}
        \end{minipage}
        \hfill
    	\begin{minipage}[t]{0.325\linewidth}
            \centering
        	\includegraphics[trim=11 45 11 11, clip, width=1.0\linewidth]{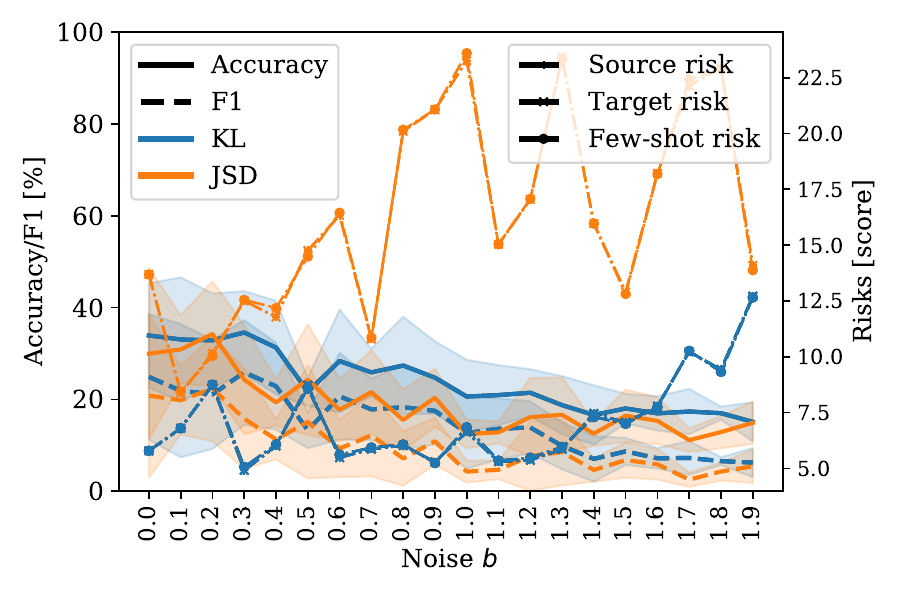}
        \end{minipage}
        \hfill
    	\begin{minipage}[t]{0.325\linewidth}
            \centering
        	\includegraphics[trim=11 45 11 11, clip, width=1.0\linewidth]{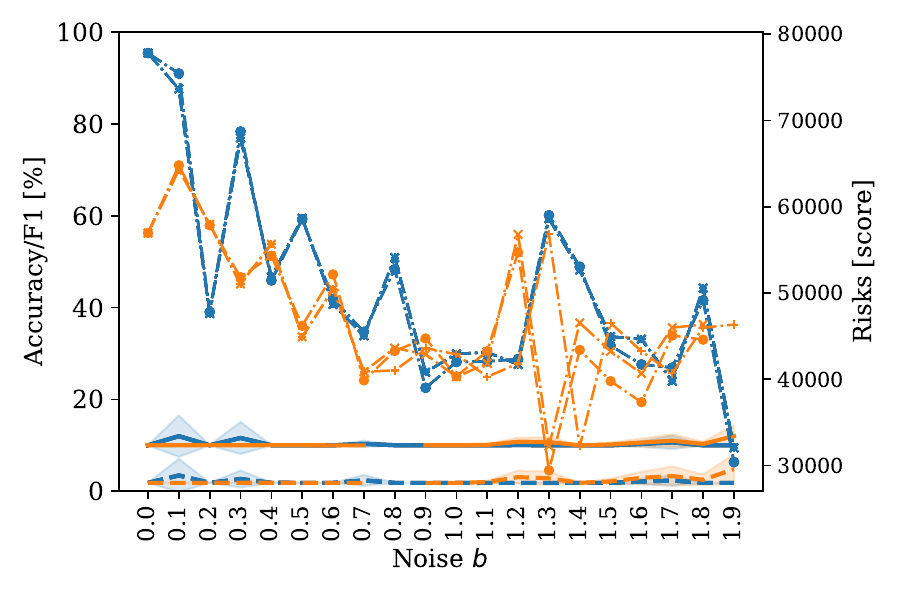}
        \end{minipage}
        \subcaption{Methods KL and JSD.}
        \label{figure_results_sin_cos6}
    \end{minipage}
	\begin{minipage}[t]{0.495\linewidth}
    	\begin{minipage}[t]{0.325\linewidth}
            \centering
        	\includegraphics[trim=11 45 11 11, clip, width=1.0\linewidth]{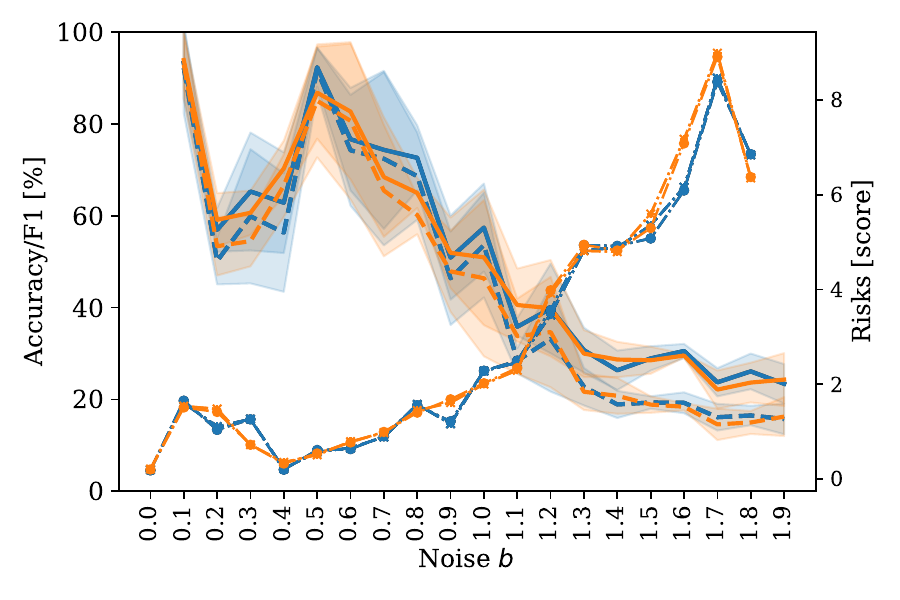}
        \end{minipage}
        \hfill
    	\begin{minipage}[t]{0.325\linewidth}
            \centering
        	\includegraphics[trim=11 45 11 11, clip, width=1.0\linewidth]{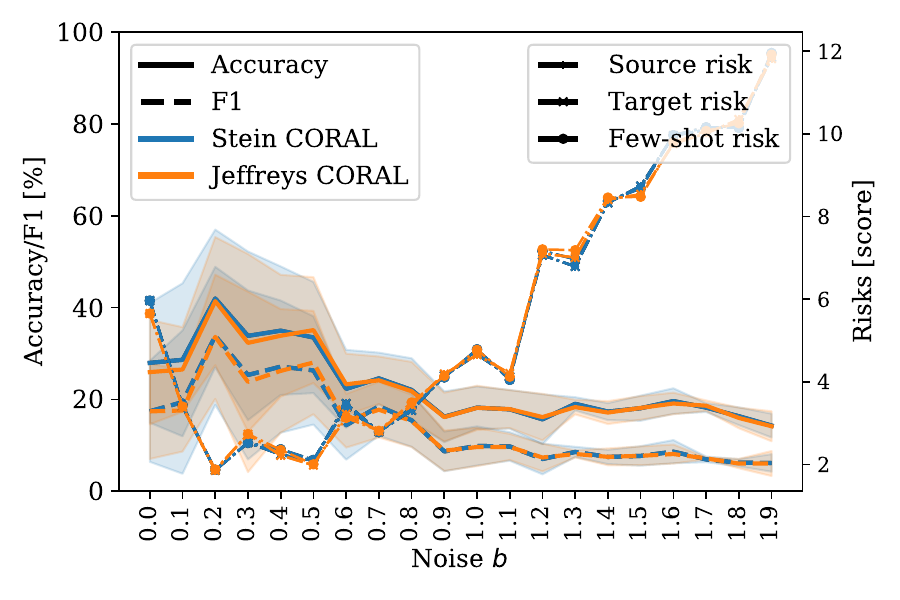}
        \end{minipage}
        \hfill
    	\begin{minipage}[t]{0.325\linewidth}
            \centering
        	\includegraphics[trim=11 45 11 11, clip, width=1.0\linewidth]{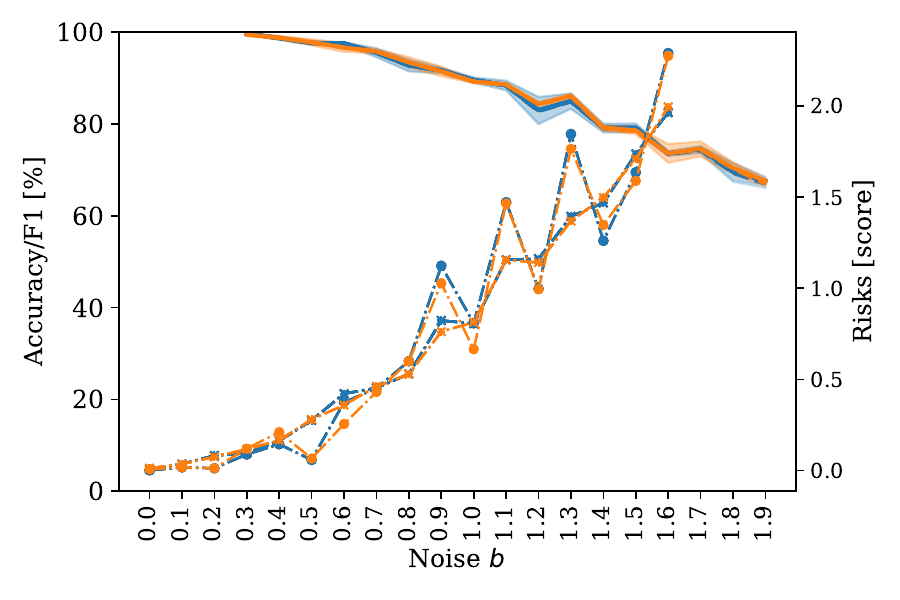}
        \end{minipage}
        \subcaption{Methods Stein CORAL and Jeffreys CORAL.}
        \label{figure_results_sin_cos7}
    \end{minipage}
    \hfill
	\begin{minipage}[t]{0.495\linewidth}
    	\begin{minipage}[t]{0.325\linewidth}
            \centering
        	\includegraphics[trim=11 45 11 11, clip, width=1.0\linewidth]{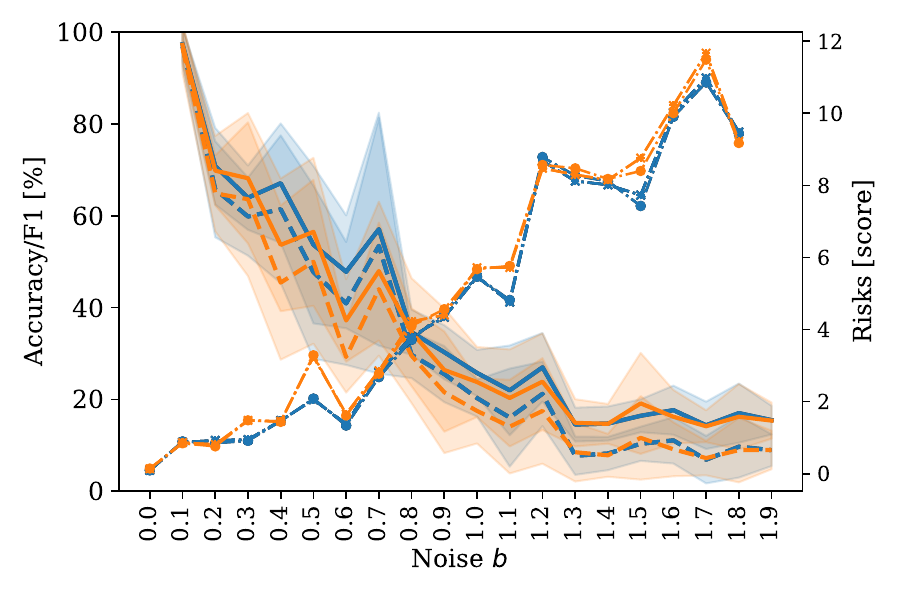}
        \end{minipage}
        \hfill
    	\begin{minipage}[t]{0.325\linewidth}
            \centering
        	\includegraphics[trim=11 45 11 11, clip, width=1.0\linewidth]{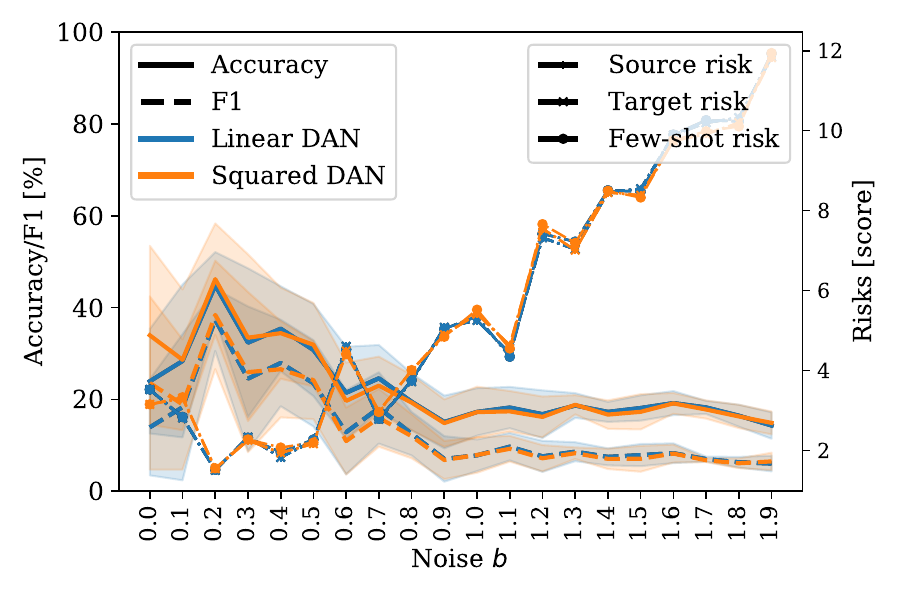}
        \end{minipage}
        \hfill
    	\begin{minipage}[t]{0.325\linewidth}
            \centering
        	\includegraphics[trim=11 45 11 11, clip, width=1.0\linewidth]{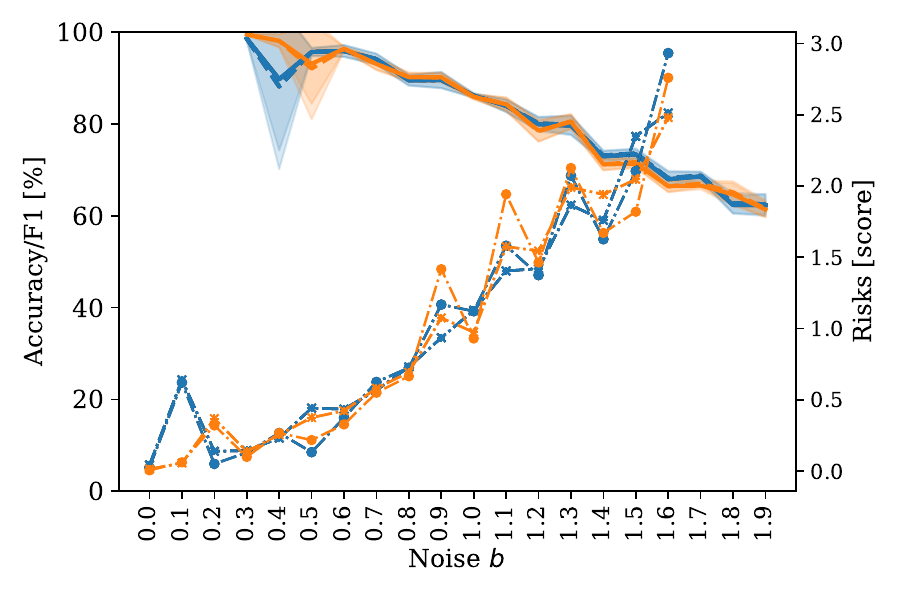}
        \end{minipage}
        \subcaption{Methods linear DAN and squared DAN.}
        \label{figure_results_sin_cos8}
    \end{minipage}
	\begin{minipage}[t]{0.495\linewidth}
    	\begin{minipage}[t]{0.325\linewidth}
            \centering
        	\includegraphics[trim=11 45 11 11, clip, width=1.0\linewidth]{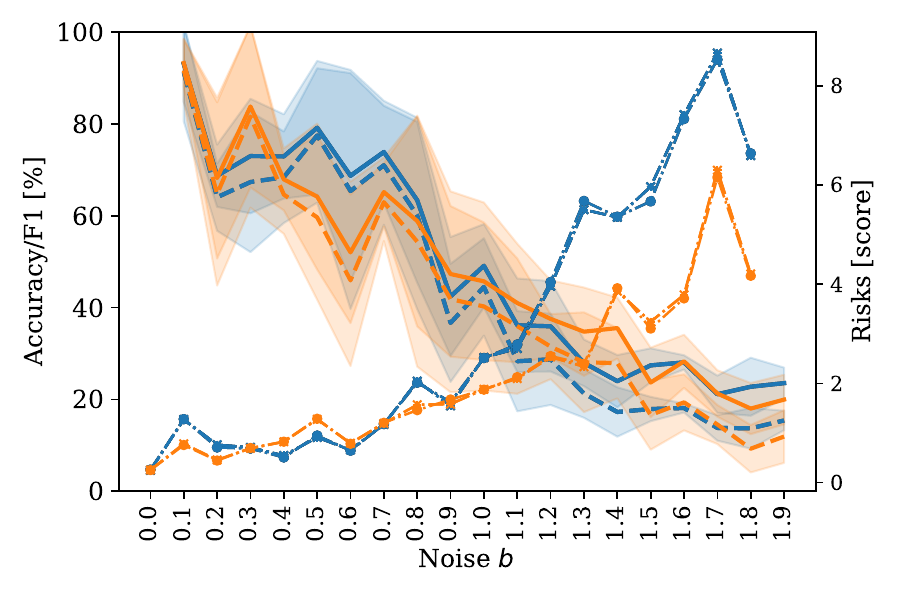}
        \end{minipage}
        \hfill
    	\begin{minipage}[t]{0.325\linewidth}
            \centering
        	\includegraphics[trim=11 45 11 11, clip, width=1.0\linewidth]{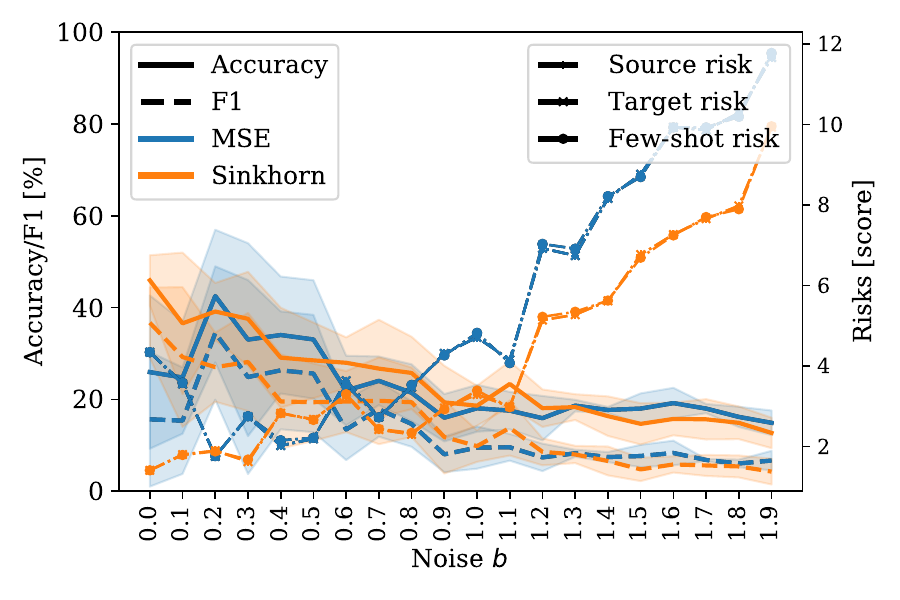}
        \end{minipage}
        \hfill
    	\begin{minipage}[t]{0.325\linewidth}
            \centering
        	\includegraphics[trim=11 45 11 11, clip, width=1.0\linewidth]{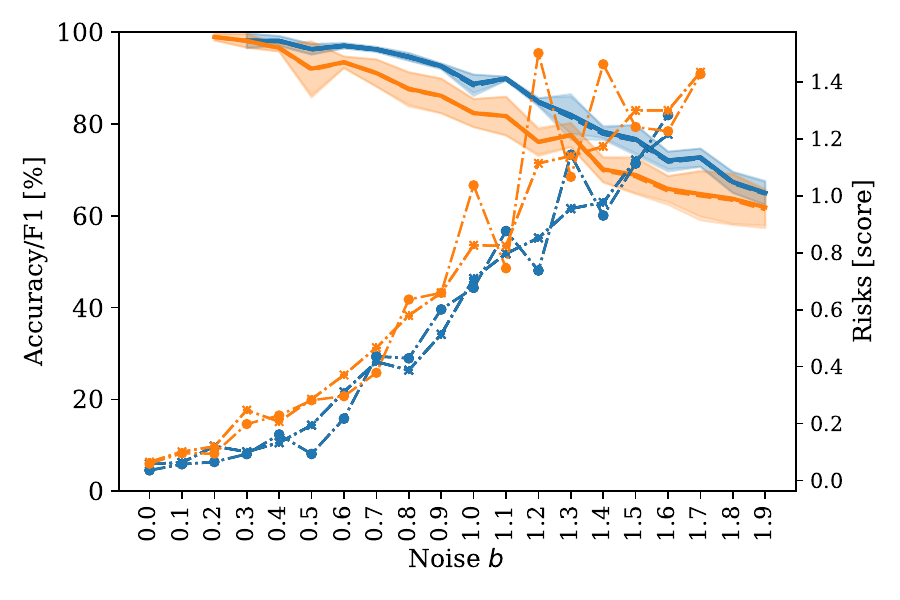}
        \end{minipage}
        \subcaption{Methods MSE and Sinkhorn.}
        \label{figure_results_sin_cos9}
    \end{minipage}
    \hfill
	\begin{minipage}[t]{0.495\linewidth}
    	\begin{minipage}[t]{0.325\linewidth}
            \centering
        	\includegraphics[trim=11 45 11 11, clip, width=1.0\linewidth]{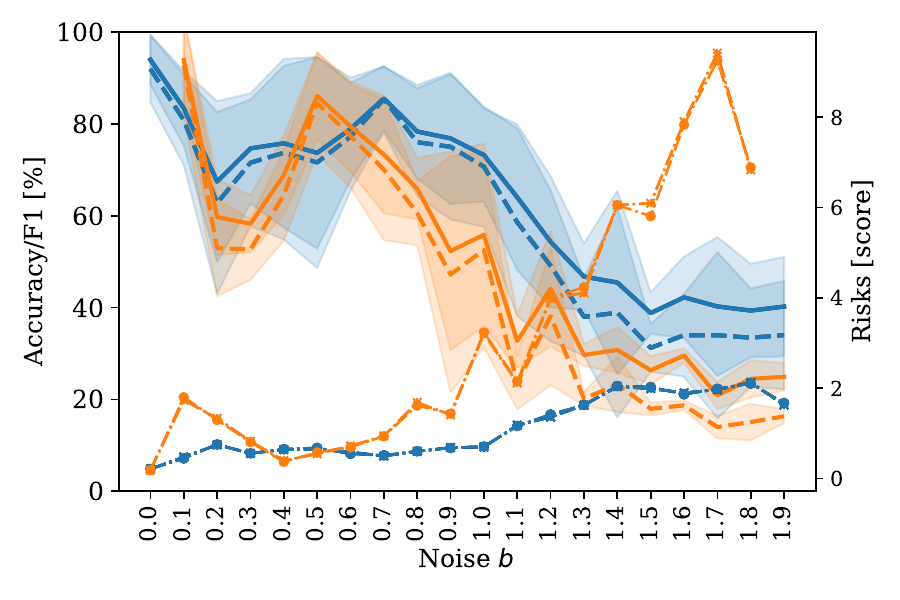}
        \end{minipage}
        \hfill
    	\begin{minipage}[t]{0.325\linewidth}
            \centering
        	\includegraphics[trim=11 45 11 11, clip, width=1.0\linewidth]{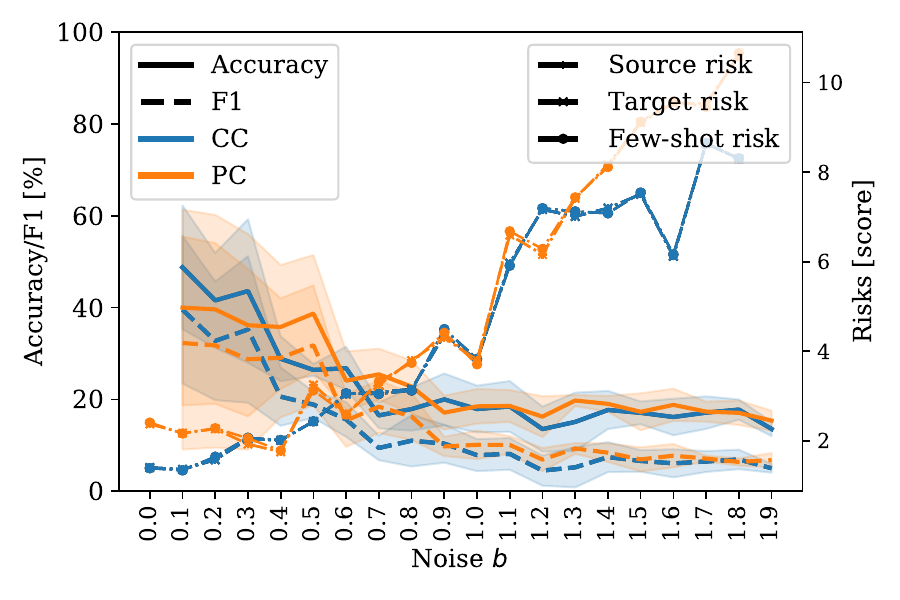}
        \end{minipage}
        \hfill
    	\begin{minipage}[t]{0.325\linewidth}
            \centering
        	\includegraphics[trim=11 45 11 11, clip, width=1.0\linewidth]{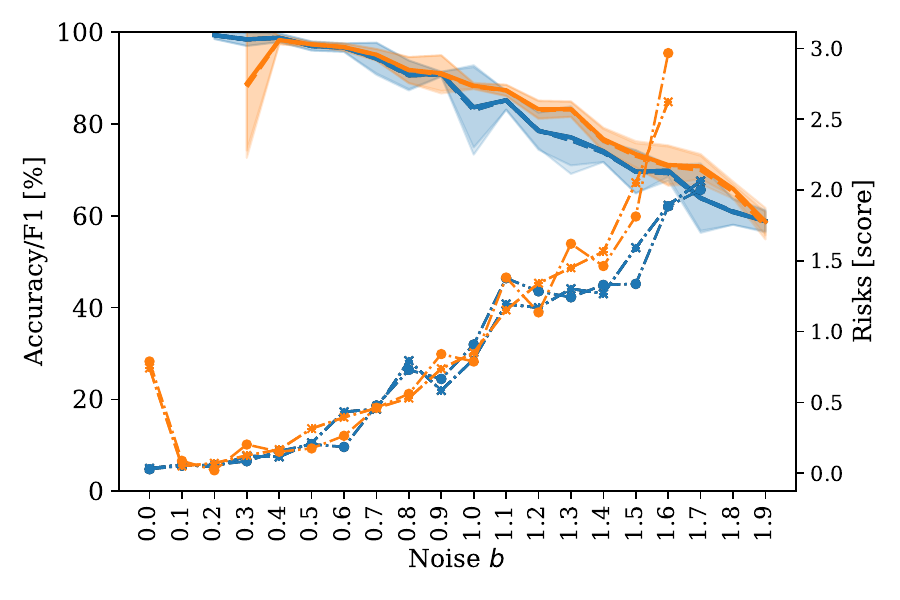}
        \end{minipage}
        \subcaption{Methods CS and PC.}
        \label{figure_results_sin_cos10}
    \end{minipage}
    \begin{minipage}[t]{0.495\linewidth}
    	\begin{minipage}[t]{0.325\linewidth}
            \centering
        	\includegraphics[trim=11 45 11 11, clip, width=1.0\linewidth]{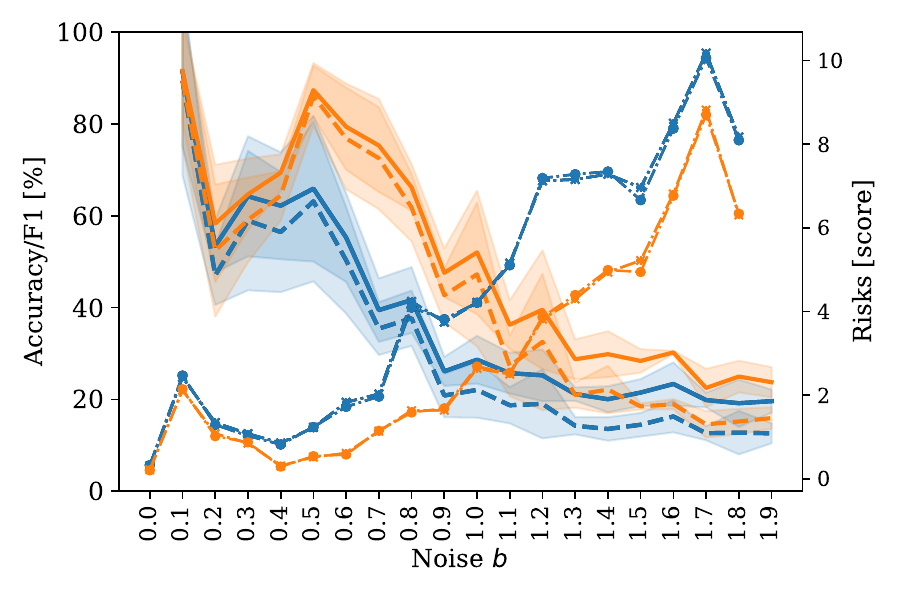}
        \end{minipage}
        \hfill
    	\begin{minipage}[t]{0.325\linewidth}
            \centering
        	\includegraphics[trim=11 45 11 11, clip, width=1.0\linewidth]{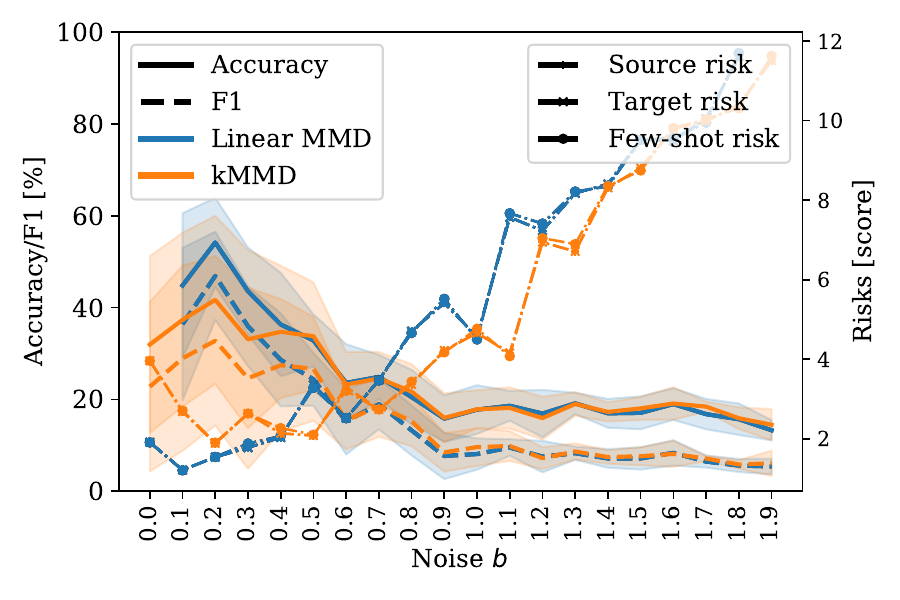}
        \end{minipage}
        \hfill
    	\begin{minipage}[t]{0.325\linewidth}
            \centering
        	\includegraphics[trim=11 45 11 11, clip, width=1.0\linewidth]{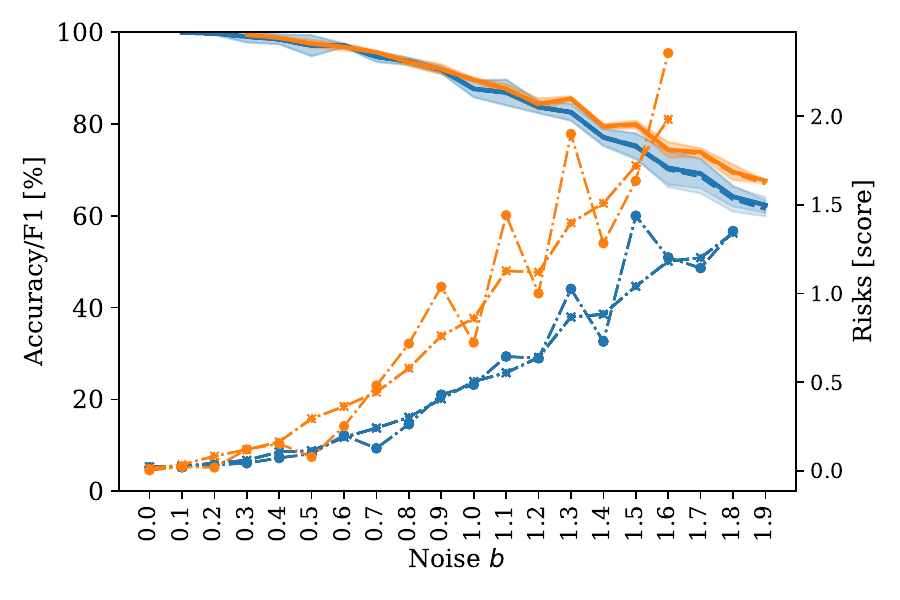}
        \end{minipage}
        \subcaption{Methods linear MMD and kMMD.}
        \label{figure_results_sin_cos11}
    \end{minipage}
    \hfill
	\begin{minipage}[t]{0.495\linewidth}
    	\begin{minipage}[t]{0.325\linewidth}
            \centering
        	\includegraphics[trim=11 45 11 11, clip, width=1.0\linewidth]{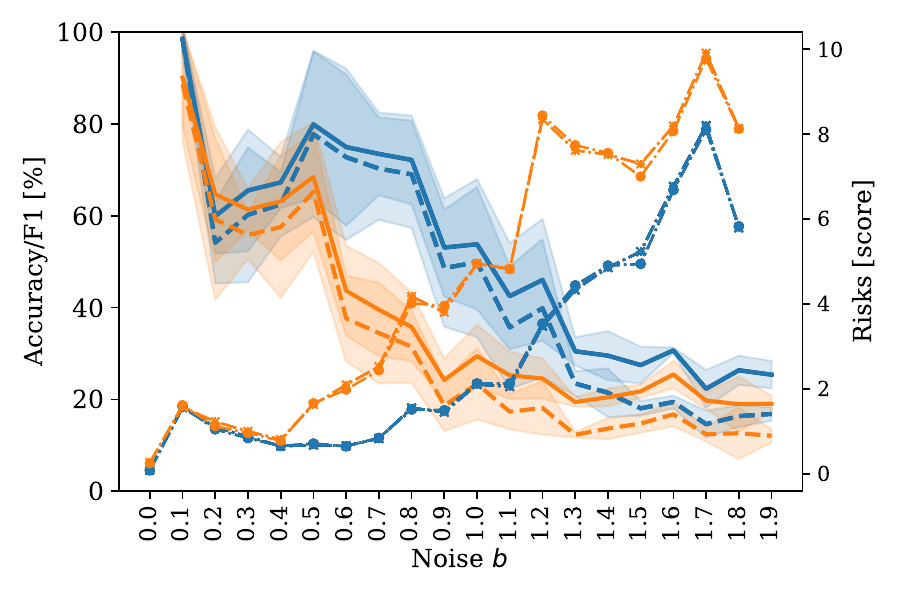}
        \end{minipage}
        \hfill
    	\begin{minipage}[t]{0.325\linewidth}
            \centering
        	\includegraphics[trim=11 45 11 11, clip, width=1.0\linewidth]{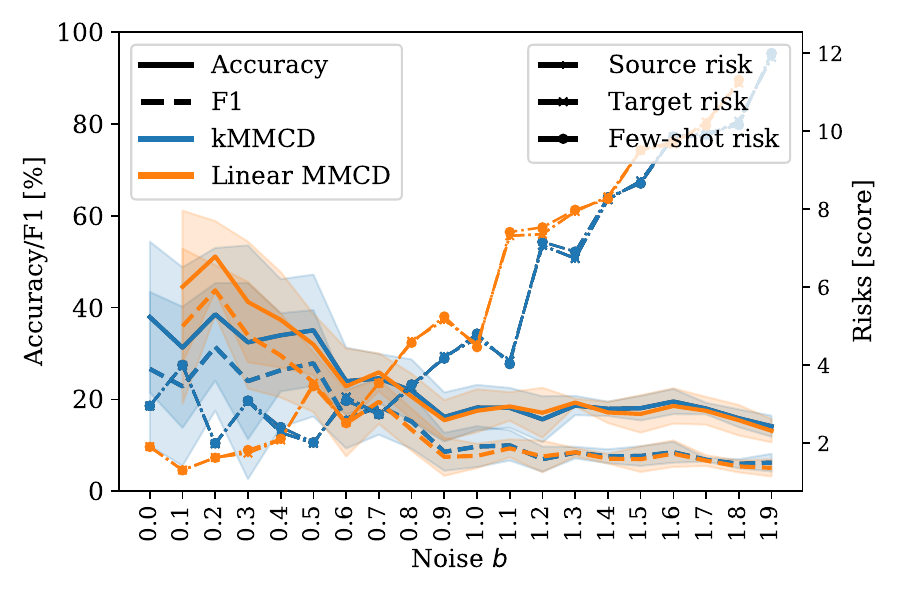}
        \end{minipage}
        \hfill
    	\begin{minipage}[t]{0.325\linewidth}
            \centering
        	\includegraphics[trim=11 45 11 11, clip, width=1.0\linewidth]{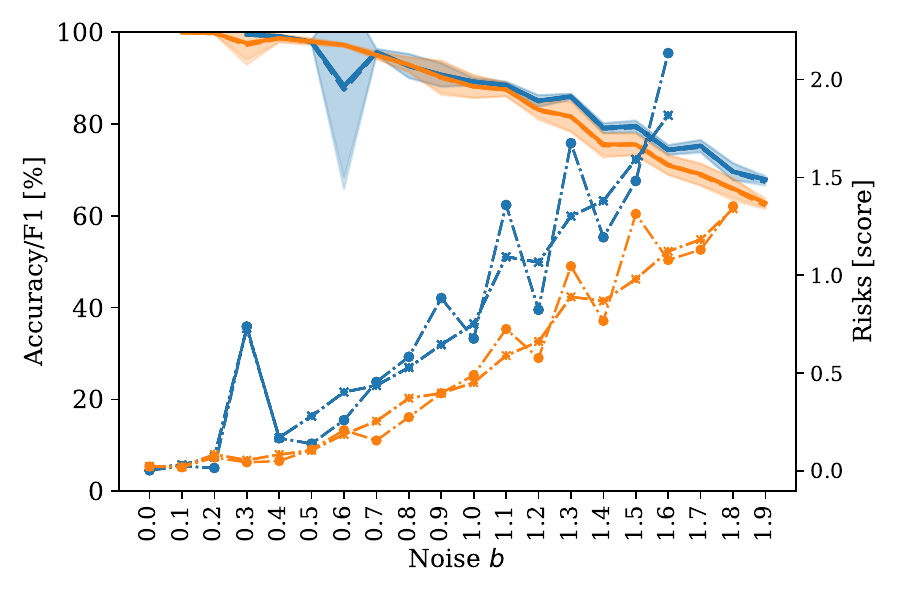}
        \end{minipage}
        \subcaption{Methods kMMCD and linear MMCD.}
        \label{figure_results_sin_cos12}
    \end{minipage}
	\begin{minipage}[t]{0.495\linewidth}
    	\begin{minipage}[t]{0.325\linewidth}
            \centering
        	\includegraphics[trim=11 45 11 11, clip, width=1.0\linewidth]{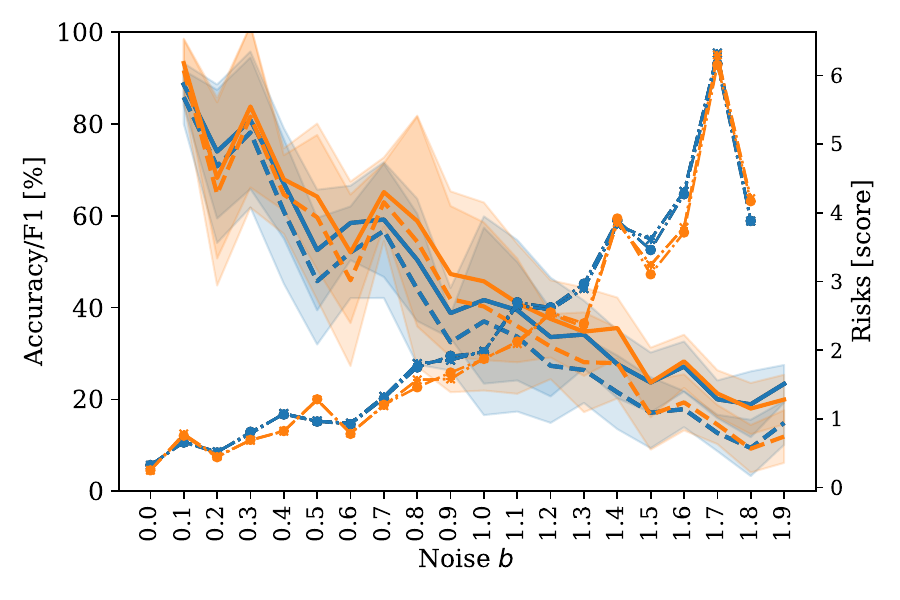}
        \end{minipage}
        \hfill
    	\begin{minipage}[t]{0.325\linewidth}
            \centering
        	\includegraphics[trim=11 45 11 11, clip, width=1.0\linewidth]{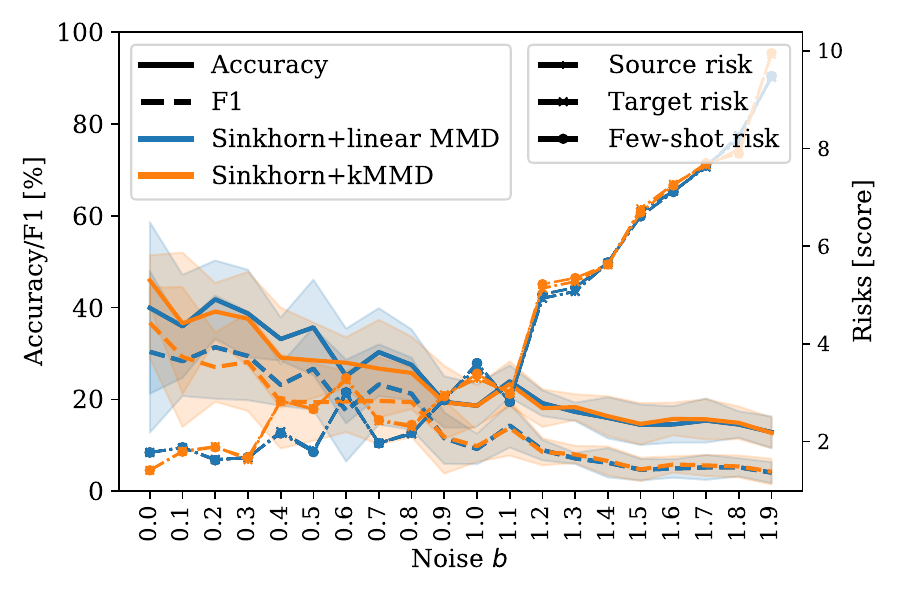}
        \end{minipage}
        \hfill
    	\begin{minipage}[t]{0.325\linewidth}
            \centering
        	\includegraphics[trim=11 45 11 11, clip, width=1.0\linewidth]{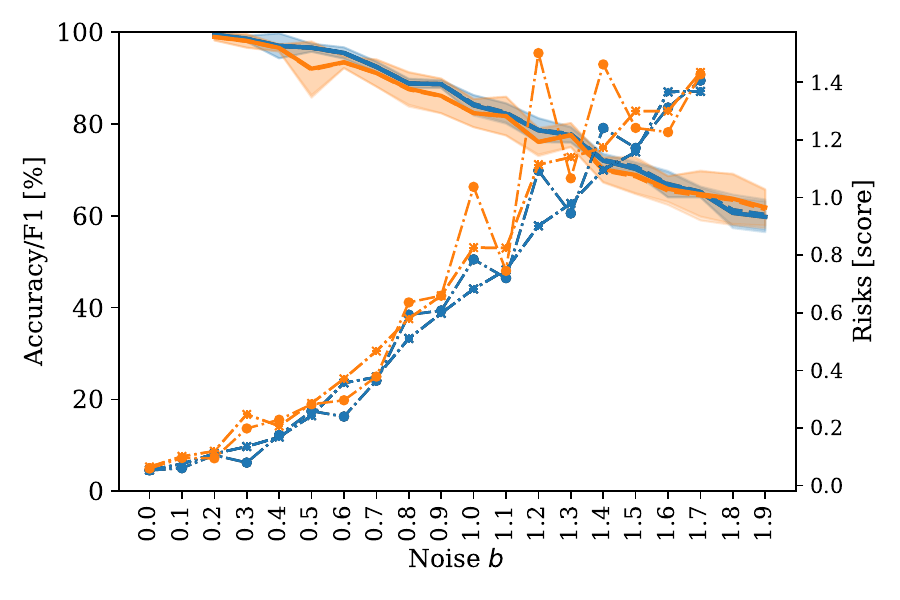}
        \end{minipage}
        \subcaption{Methods Sinkhorn+[linear MMD, kMMD].}
        \label{figure_results_sin_cos13}
    \end{minipage}
    \hfill
	\begin{minipage}[t]{0.495\linewidth}
    	\begin{minipage}[t]{0.325\linewidth}
            \centering
        	\includegraphics[trim=11 45 11 11, clip, width=1.0\linewidth]{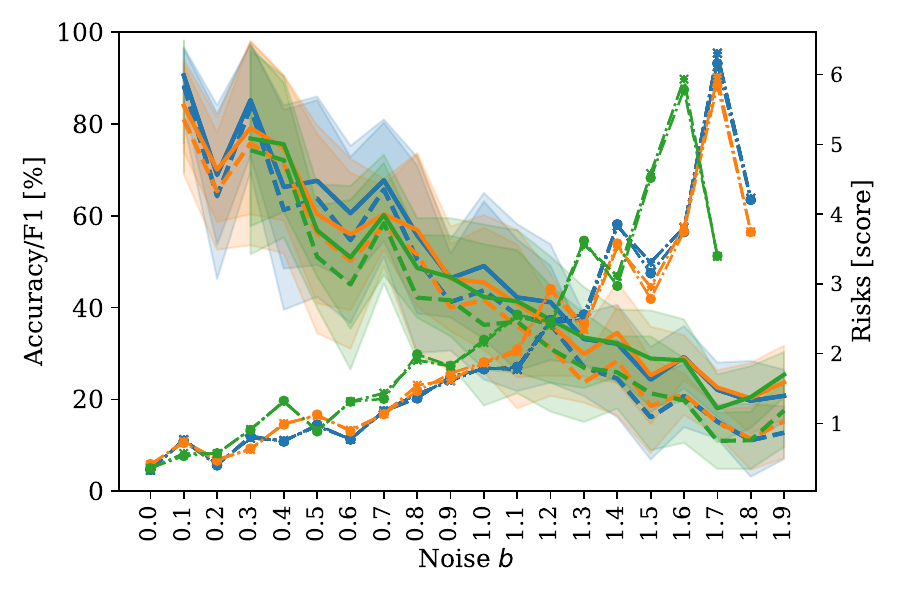}
        \end{minipage}
        \hfill
    	\begin{minipage}[t]{0.325\linewidth}
            \centering
        	\includegraphics[trim=11 45 11 11, clip, width=1.0\linewidth]{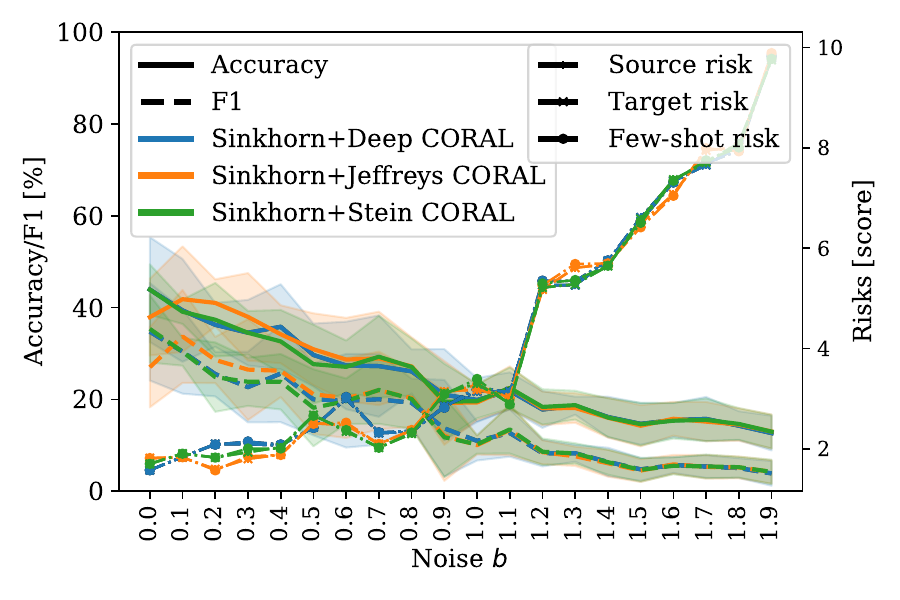}
        \end{minipage}
        \hfill
    	\begin{minipage}[t]{0.325\linewidth}
            \centering
        	\includegraphics[trim=11 45 11 11, clip, width=1.0\linewidth]{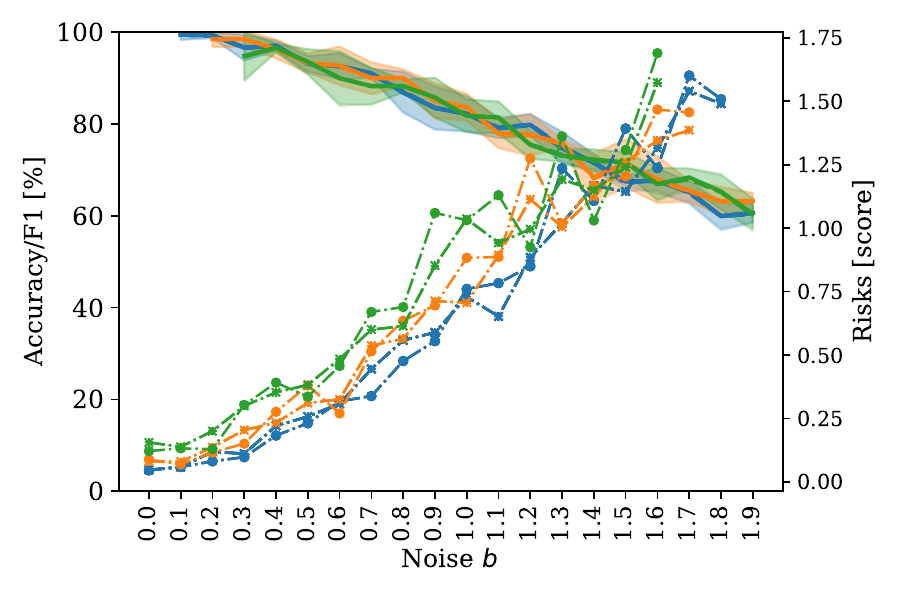}
        \end{minipage}
        \subcaption{Methods Sinkhorn+[DeepCORAL, Jeffreys and Stein CORAL].}
        \label{figure_results_sin_cos14}
    \end{minipage}
	\begin{minipage}[t]{0.495\linewidth}
    	\begin{minipage}[t]{0.325\linewidth}
            \centering
        	\includegraphics[trim=11 45 11 11, clip, width=1.0\linewidth]{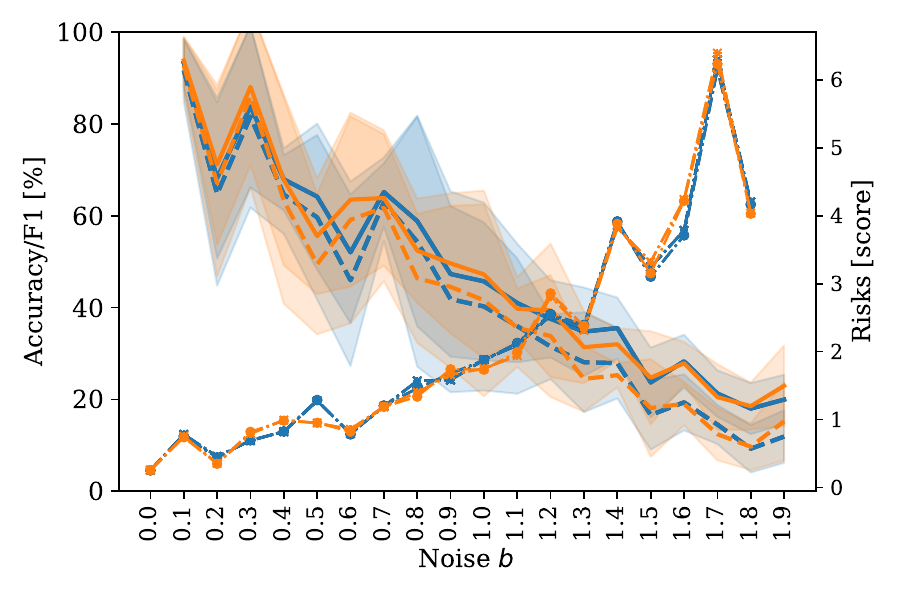}
        \end{minipage}
        \hfill
    	\begin{minipage}[t]{0.325\linewidth}
            \centering
        	\includegraphics[trim=11 45 11 11, clip, width=1.0\linewidth]{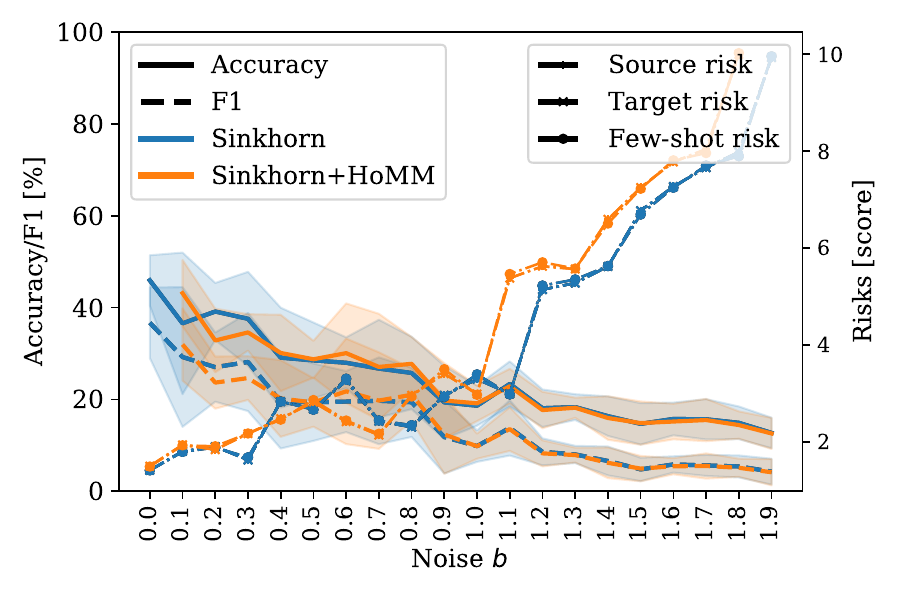}
        \end{minipage}
        \hfill
    	\begin{minipage}[t]{0.325\linewidth}
            \centering
        	\includegraphics[trim=11 45 11 11, clip, width=1.0\linewidth]{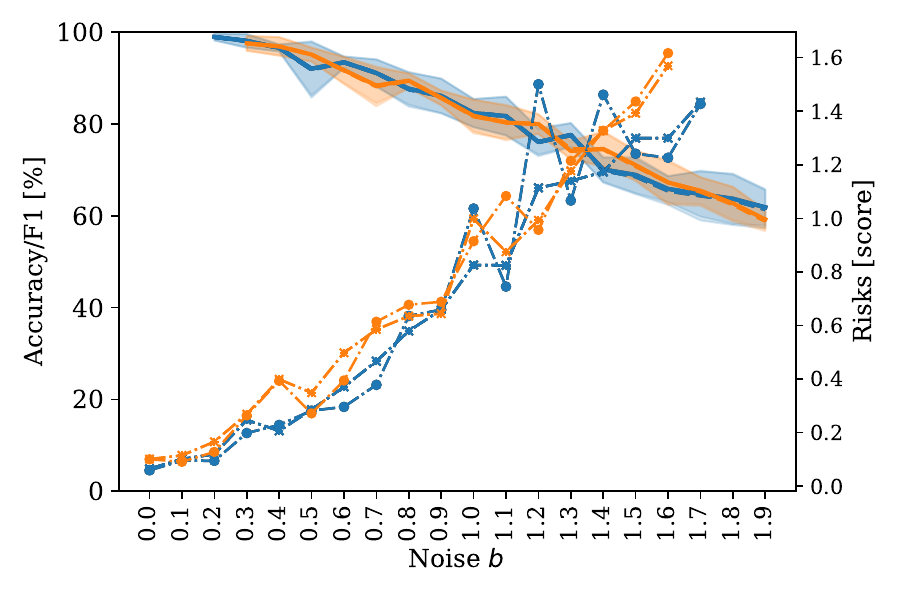}
        \end{minipage}
        \subcaption{Methods Sinkhorn and Sinkhorn+HoMM.}
        \label{figure_results_sin_cos15}
    \end{minipage}
    \hfill
	\begin{minipage}[t]{0.495\linewidth}
    	\begin{minipage}[t]{0.325\linewidth}
            \centering
        	\includegraphics[trim=11 45 11 11, clip, width=1.0\linewidth]{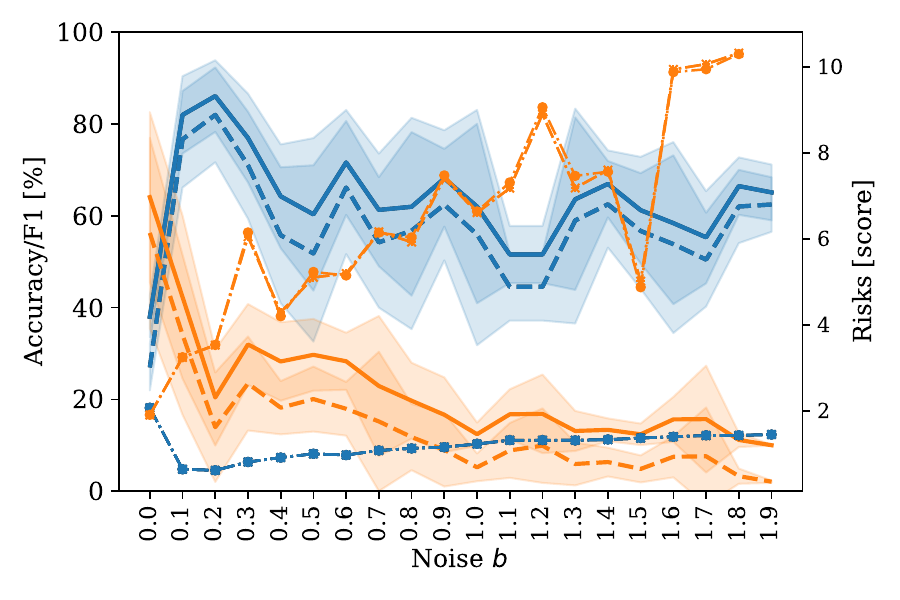}
        \end{minipage}
        \hfill
    	\begin{minipage}[t]{0.325\linewidth}
            \centering
        	\includegraphics[trim=11 45 11 11, clip, width=1.0\linewidth]{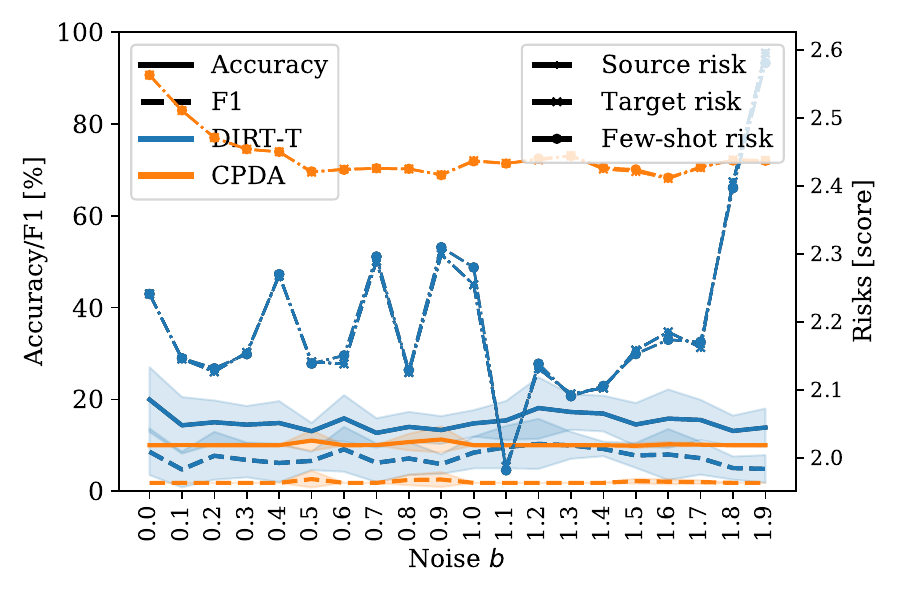}
        \end{minipage}
        \hfill
    	\begin{minipage}[t]{0.325\linewidth}
            \centering
        	\includegraphics[trim=11 45 11 11, clip, width=1.0\linewidth]{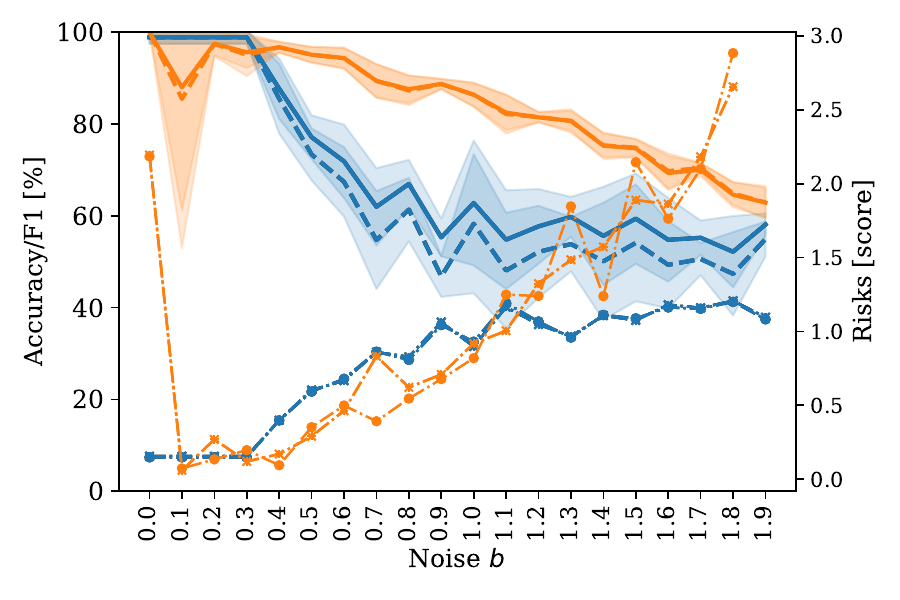}
        \end{minipage}
        \subcaption{Methods DIRT-T and \textbf{CPDA} (ours).}
        \label{figure_results_sin_cos16}
    \end{minipage}
    \caption{Performance of the evaluated domain-adaptation methods on the controlled sinusoidal benchmark as a function of the noise parameter $b \in \{0.0,0.1,\ldots,1.9\}$. The target domain contains class-specific sinusoidal signals corrupted by noise sampled from $\mathcal{U}(0,b)$, whereas the source signals are sign-inverted and corrupted by noise sampled from $\mathcal{U}(0,b/2)$. Each subfigure compares the methods indicated in its subcaption using a CNN (left), ResNet18 (center), and TCN (right) backbone. Solid and dashed curves report target-domain accuracy and F1-score, respectively, on the left vertical axis. The curves associated with the secondary vertical axis show the source, target, and few-shot risks. Lines represent the mean over five independent runs, while the shaded regions indicate the corresponding standard deviations. Increasing $b$ produces a progressively stronger stochastic domain shift and therefore reveals the robustness of the different adaptation objectives under increasing noise.}
    \label{figure_results_sin_cos}
\end{figure*}

\begin{figure*}[!t]
    \centering
	\begin{minipage}[t]{0.495\linewidth}
    	\begin{minipage}[t]{0.325\linewidth}
            \centering
        	\includegraphics[trim=11 45 11 11, clip, width=1.0\linewidth]{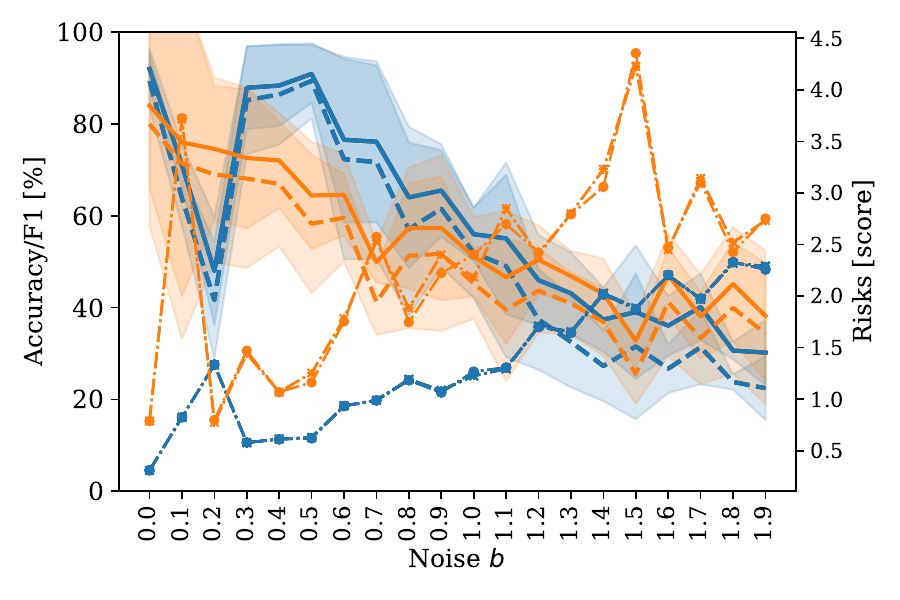}
        \end{minipage}
        \hfill
    	\begin{minipage}[t]{0.325\linewidth}
            \centering
        	\includegraphics[trim=11 45 11 11, clip, width=1.0\linewidth]{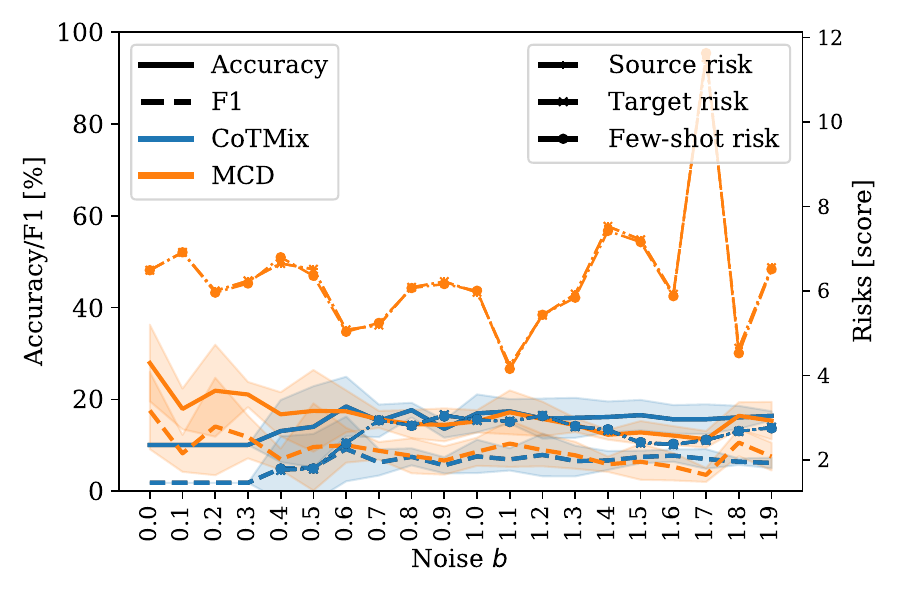}
        \end{minipage}
        \hfill
    	\begin{minipage}[t]{0.325\linewidth}
            \centering
        	\includegraphics[trim=11 45 11 11, clip, width=1.0\linewidth]{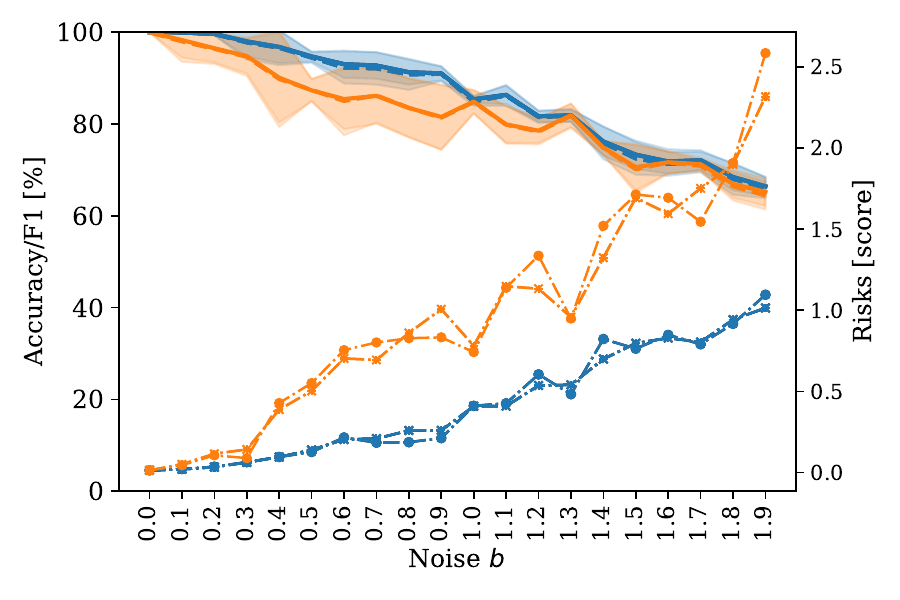}
        \end{minipage}
        \subcaption{Methods CoTMix and MCD.}
        \label{figure_results_sin_cos17}
    \end{minipage}
    \hfill
	\begin{minipage}[t]{0.495\linewidth}
    	\begin{minipage}[t]{0.325\linewidth}
            \centering
        	\includegraphics[trim=11 45 11 11, clip, width=1.0\linewidth]{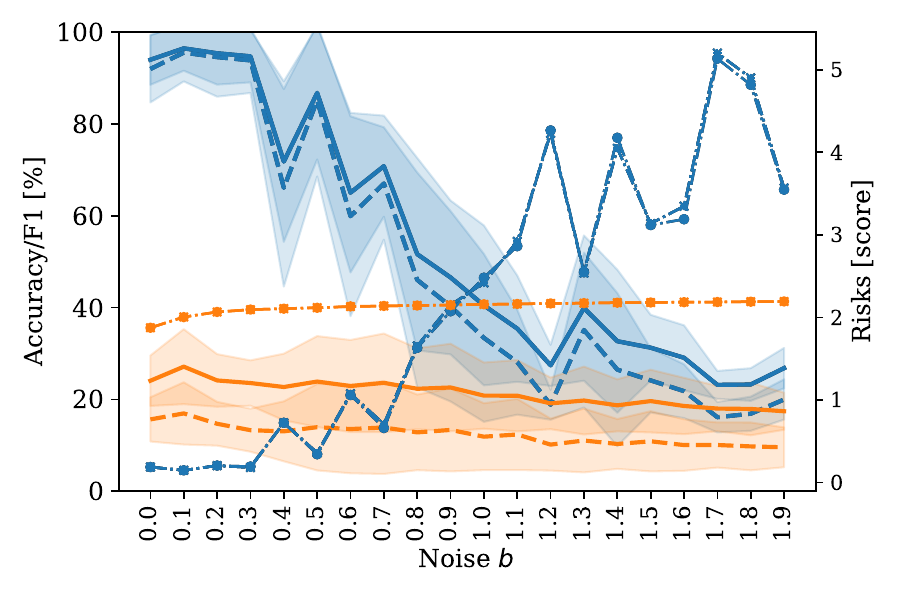}
        \end{minipage}
        \hfill
    	\begin{minipage}[t]{0.325\linewidth}
            \centering
        	\includegraphics[trim=11 45 11 11, clip, width=1.0\linewidth]{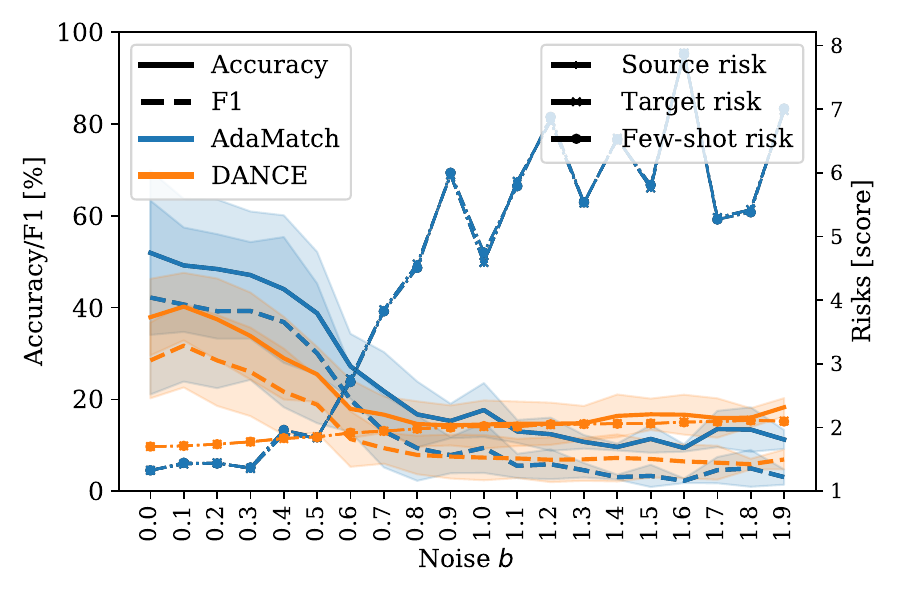}
        \end{minipage}
        \hfill
    	\begin{minipage}[t]{0.325\linewidth}
            \centering
        	\includegraphics[trim=11 45 11 11, clip, width=1.0\linewidth]{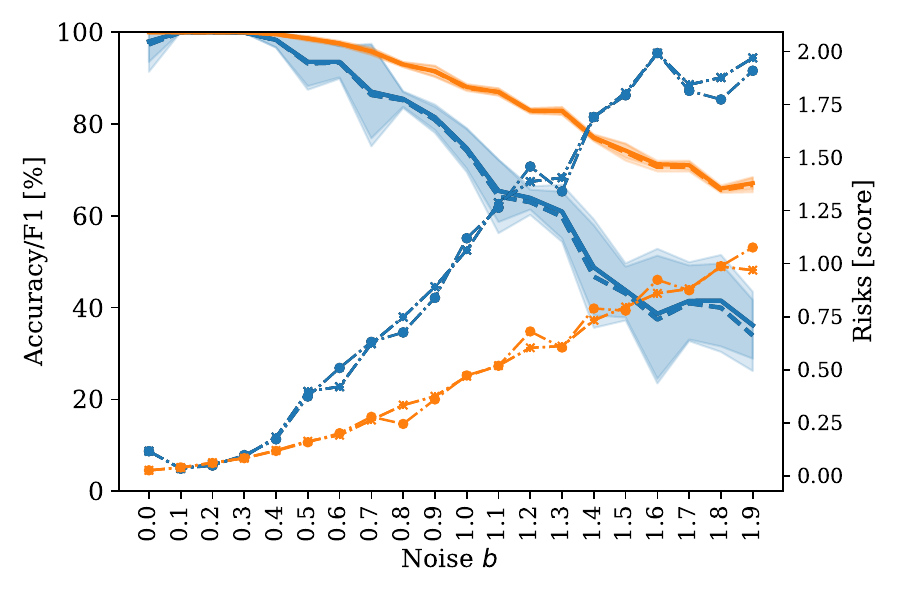}
        \end{minipage}
        \subcaption{Methods AdaMatch and DANCE.}
        \label{figure_results_sin_cos18}
    \end{minipage}
    \hfill
	\begin{minipage}[t]{0.495\linewidth}
    	\begin{minipage}[t]{0.325\linewidth}
            \centering
        	\includegraphics[trim=11 45 11 11, clip, width=1.0\linewidth]{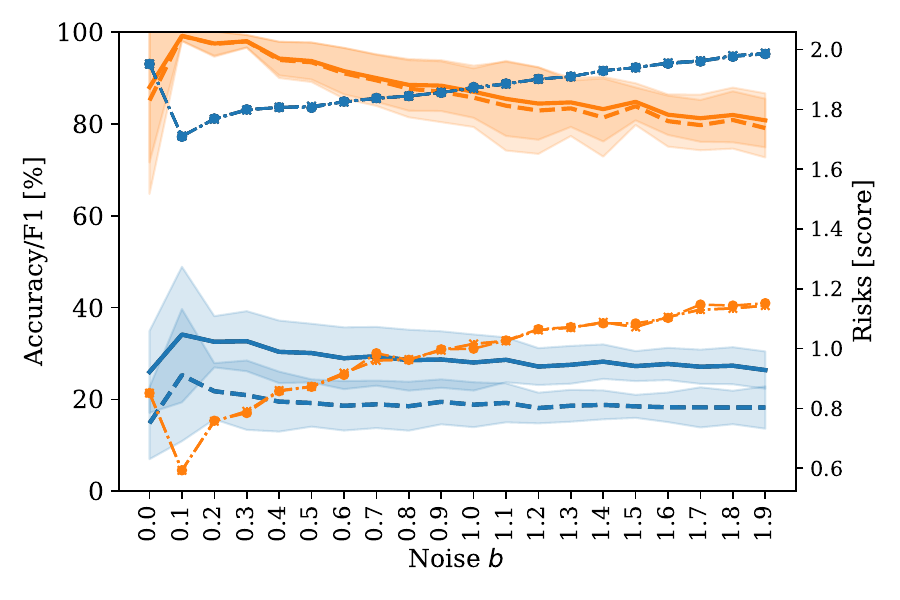}
        \end{minipage}
        \hfill
    	\begin{minipage}[t]{0.325\linewidth}
            \centering
        	\includegraphics[trim=11 45 11 11, clip, width=1.0\linewidth]{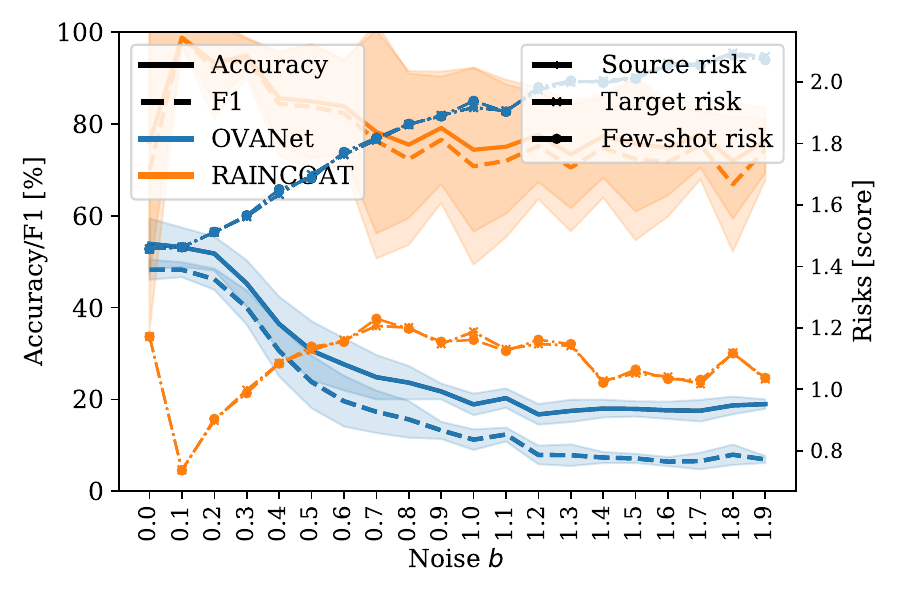}
        \end{minipage}
        \hfill
    	\begin{minipage}[t]{0.325\linewidth}
            \centering
        	\includegraphics[trim=11 45 11 11, clip, width=1.0\linewidth]{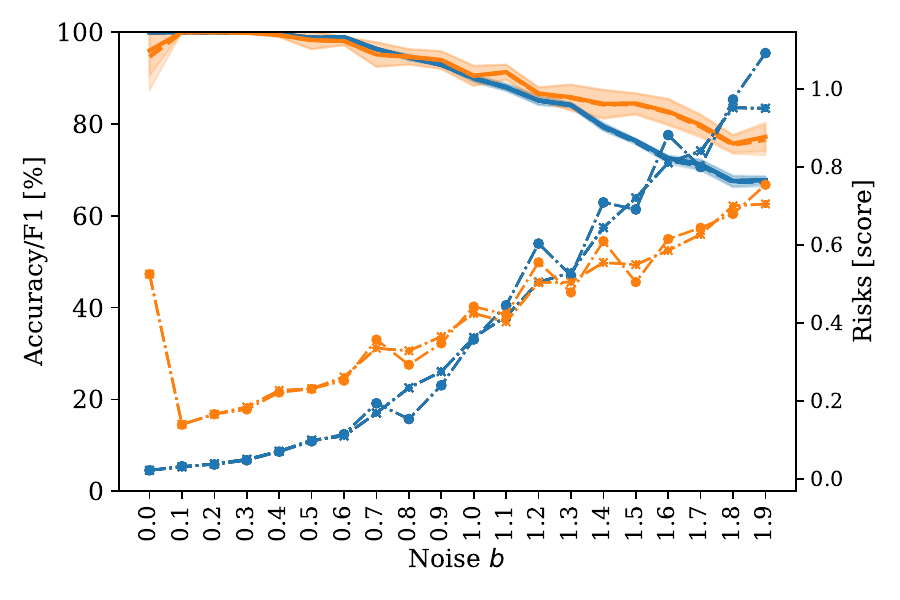}
        \end{minipage}
        \subcaption{Methods OVANet and RAINCOAT.}
        \label{figure_results_sin_cos19}
    \end{minipage}
    \hfill
    \caption{Continuation of Figure~\ref{figure_results_sin_cos}, comparing CoTMix, MCD, AdaMatch, DANCE, OVANet, and RAINCOAT on the controlled sinusoidal benchmark for increasing noise levels $b$. The columns show results for CNN, ResNet18, and TCN, respectively. Solid and dashed lines denote target-domain accuracy and F1-score, while the marker-based curves show the source, target, and few-shot risks. Results are averaged over five runs, with shaded regions indicating the standard deviations.}
    \label{figure_results_sin_cos_extended}
\end{figure*}

\paragraph{Evaluation of DA Methods and Encoder Networks on the Sinusoidal Benchmark.} Figure~\ref{figure_results_sin_cos} compares the evaluated DA methods under the controlled sinusoidal domain shift, for which the noise parameter $b$ progressively increases the corruption of the target signals. The results reveal a pronounced interaction between the adaptation objective and the encoder architecture. ResNet18 performs close to the chance level of $10\%$ for many methods and noise settings, indicating that its learned representation is not sufficiently discriminative for this benchmark. The CNN produces higher accuracies, but its performance is strongly method-dependent and frequently exhibits non-monotonic behavior and large variations across runs. In contrast, the TCN provides the most robust representations: most established discrepancy-, covariance-, moment-, and transport-based methods achieve almost perfect classification under weak noise and degrade more gradually as $b$ increases. KL~\citep{kullback_leibler} and JSD~\citep{menendez_pardo,dorent_golland} constitute clear exceptions (Figure~\ref{figure_results_sin_cos6}), showing unstable behavior and substantially lower classification performance, particularly with the TCN. The kernelized MMD and MMCD variants generally provide small advantages over their linear counterparts (Figures~\ref{figure_results_sin_cos11} and~\ref{figure_results_sin_cos12}), whereas combining Sinkhorn transport with additional discrepancy objectives does not yield a consistent improvement over the respective individual methods (Figures~\ref{figure_results_sin_cos13}--\ref{figure_results_sin_cos15}). The strongest CPDA results are obtained with the TCN, as shown in Figure~\ref{figure_results_sin_cos16}. CPDA maintains an accuracy above approximately $85\%$ throughout the range $b\leq1.0$ and remains above $80\%$ up to $b=1.3$, despite the progressively increasing target-domain corruption. Even at the maximum noise level $b=1.9$, CPDA retains an accuracy of approximately $63\%$. Its F1-score closely follows its accuracy, demonstrating that the improvement is not caused by a bias toward individual classes but reflects balanced performance across the ten signal classes. Compared with DIRT-T, CPDA begins to show a clear advantage at approximately $b=0.4$ and achieves improvements of roughly $15$--$30$ percentage points over much of the moderate- and high-noise range. While the performance of DIRT-T decreases rapidly to approximately $50$--$60\%$, CPDA exhibits a substantially smoother degradation and remains within the group of the most robust methods. The corresponding target- and few-shot-risk\footnote{\textbf{Source, Target, and Few-Shot Risks.} Following AdaTime~\citep{ragab_eldele_tan}, the source and target risks are the mean cross-entropy losses evaluated on the complete labeled source and target test sets, respectively. The few-shot target risk is computed on a class-balanced subset of the target test set containing up to five randomly selected labeled samples per class, i.e., at most \(5K\) target samples for a \(K\)-class problem. These labeled target samples are used only to calculate the diagnostic risk after training and do not influence domain adaptation, parameter updates, or the reported target accuracy and macro-F1 score. The few-shot risk therefore represents the limited-label model-selection setting considered by AdaTime: it approximates target-domain performance when only a small number of labeled target samples are available, whereas the full target risk is an oracle diagnostic requiring all target labels.} estimates are also generally lower for CPDA throughout the moderate-noise regime, supporting the observed accuracy improvements. Moreover, except for an isolated fluctuation at very low noise, the comparatively narrow uncertainty intervals indicate that the CPDA results are consistent across the five independent runs. The same advantage is not observed with the CNN and ResNet18, showing that CPDA cannot fully compensate for a weak encoder representation and that its effectiveness depends on a backbone capable of extracting discriminative temporal features. Overall, the results demonstrate that the combination of CPDA and TCN is particularly robust against the simultaneous sign-induced and noise-induced distribution shifts considered in this controlled experiment. The additional comparisons in Figure~\ref{figure_results_sin_cos_extended} confirm that the encoder architecture remains a decisive factor for adaptation performance. CoTMix and MCD achieve high and gradually decreasing accuracies with the TCN, retaining approximately $65\%$ performance at the highest noise level, whereas both methods perform substantially worse with the CNN and ResNet18 (Figure~\ref{figure_results_sin_cos17}). AdaMatch exhibits a pronounced performance degradation as the noise intensity increases, while DANCE is considerably more robust with the TCN and maintains an accuracy of approximately $65\%$ at $b=1.9$ (Figure~\ref{figure_results_sin_cos18}). The strongest additional competitor is RAINCOAT, which benefits from explicitly combining temporal- and frequency-domain information. With the TCN, RAINCOAT retains an accuracy of approximately $75\%$ under the strongest noise corruption and generally outperforms OVANet. RAINCOAT also achieves high mean accuracy with ResNet18, although its wide confidence intervals indicate considerable variability across runs (Figure~\ref{figure_results_sin_cos19}). Overall, these results place CPDA among the most robust methods on the sinusoidal benchmark rather than establishing uniform dominance over every baseline. In particular, CPDA provides consistently strong and balanced accuracy and F1-score across increasing noise levels, while RAINCOAT and DANCE constitute its closest competitors in the high-noise regime. For more information on source-free unsupervised DA, refer to \citet{wang_gong}.

\begin{figure*}[!t]
    \centering
	\begin{minipage}[t]{0.495\linewidth}
        \centering
    	\includegraphics[trim=10 10 10 10, clip, width=1.0\linewidth]{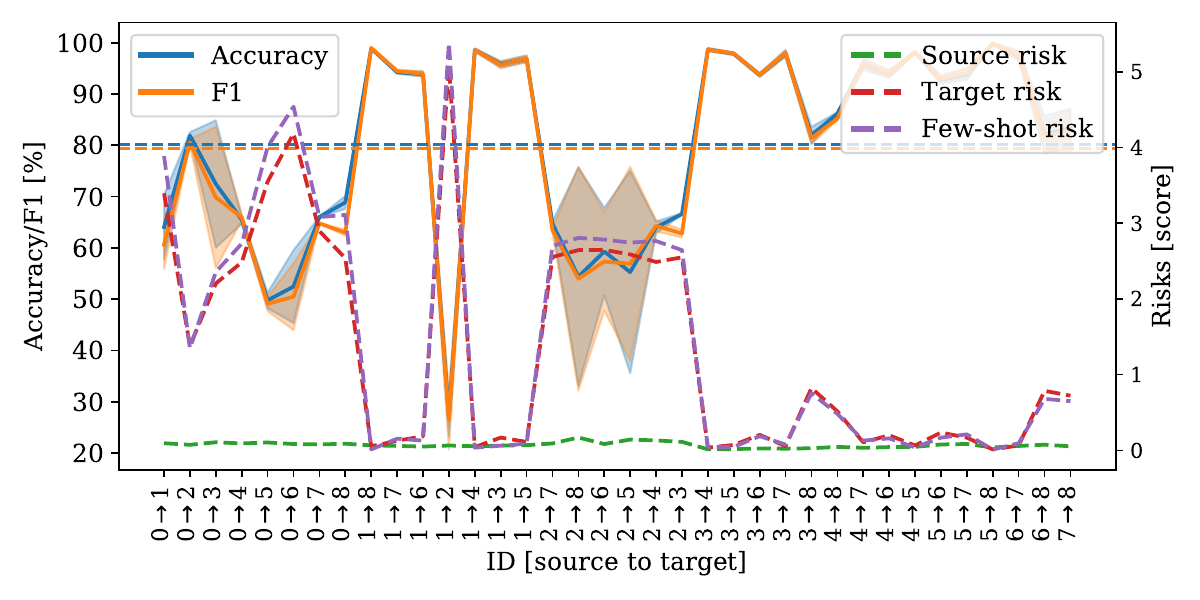}
        \subcaption{HHAR.}
        \label{label_figure_results_source_target1}
    \end{minipage}
    \hfill
	\begin{minipage}[t]{0.495\linewidth}
        \centering
    	\includegraphics[trim=10 10 10 10, clip, width=1.0\linewidth]{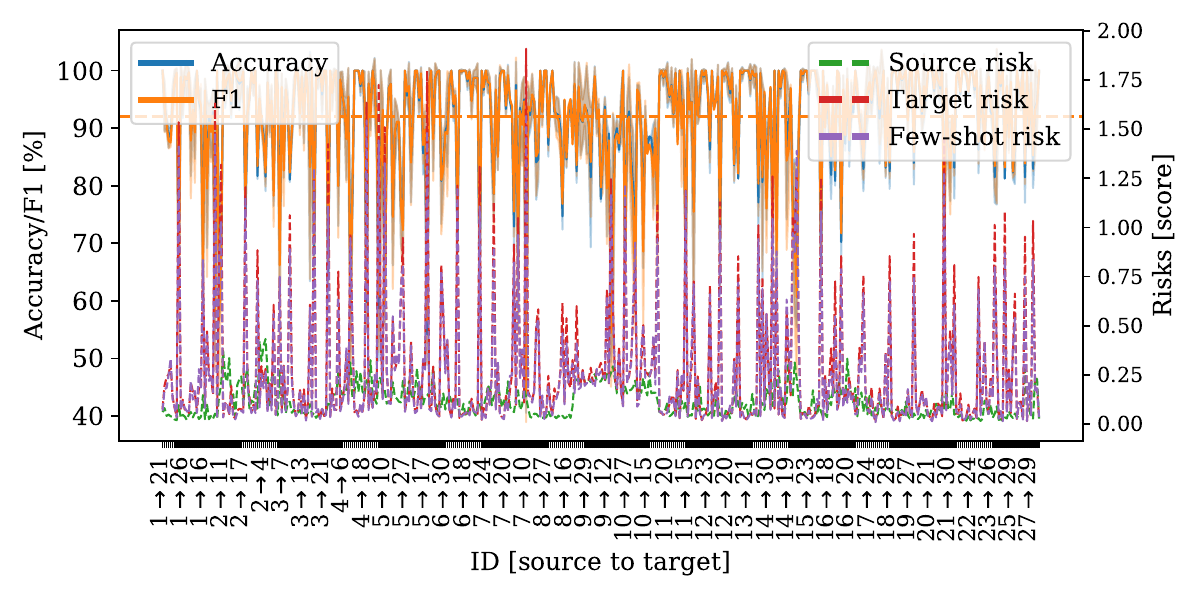}
        \subcaption{HAR.}
        \label{label_figure_results_source_target2}
    \end{minipage}
    \hfill
	\begin{minipage}[t]{0.495\linewidth}
        \centering
    	\includegraphics[trim=10 10 10 10, clip, width=1.0\linewidth]{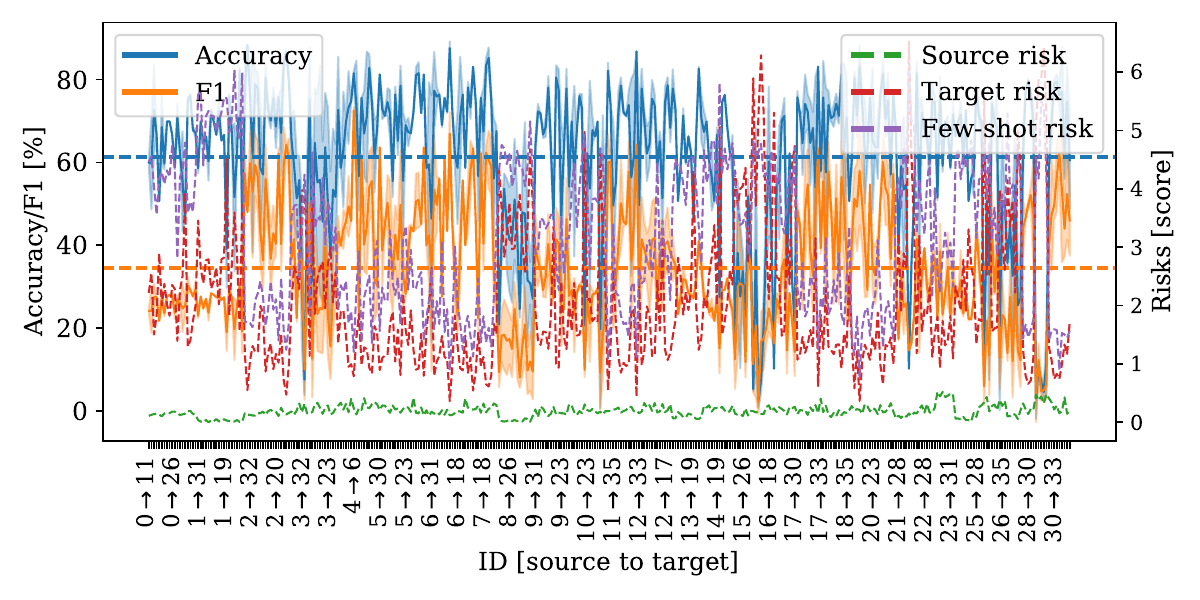}
        \subcaption{WISDM.}
        \label{label_figure_results_source_target3}
    \end{minipage}
    \hfill
	\begin{minipage}[t]{0.495\linewidth}
        \centering
    	\includegraphics[trim=10 10 10 10, clip, width=1.0\linewidth]{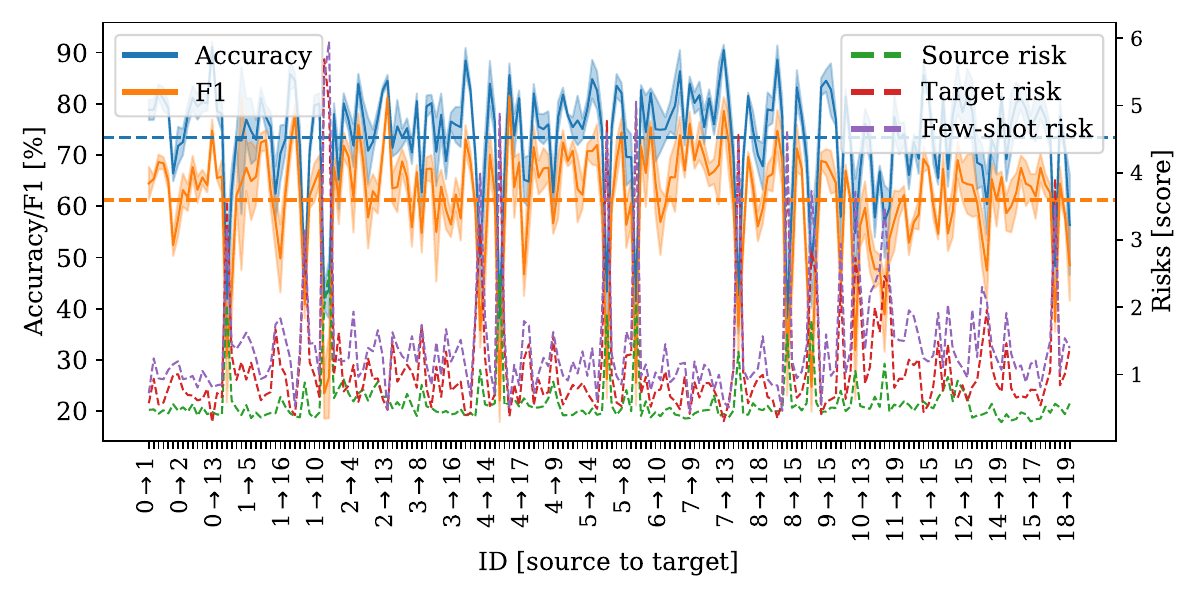}
        \subcaption{EEG.}
        \label{label_figure_results_source_target4}
    \end{minipage}
    \hfill
	\begin{minipage}[t]{0.495\linewidth}
        \centering
    	\includegraphics[trim=10 10 10 10, clip, width=1.0\linewidth]{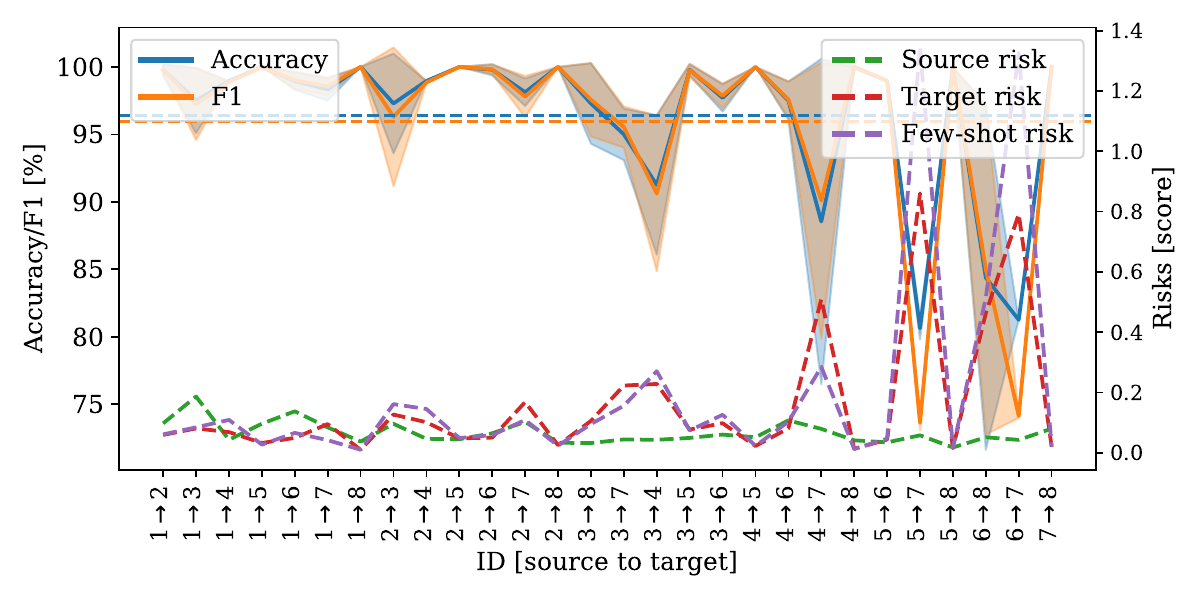}
        \subcaption{uWave.}
        \label{label_figure_results_source_target5}
    \end{minipage}
    \hfill
	\begin{minipage}[t]{0.495\linewidth}
        \centering
    	\includegraphics[trim=10 10 10 10, clip, width=1.0\linewidth]{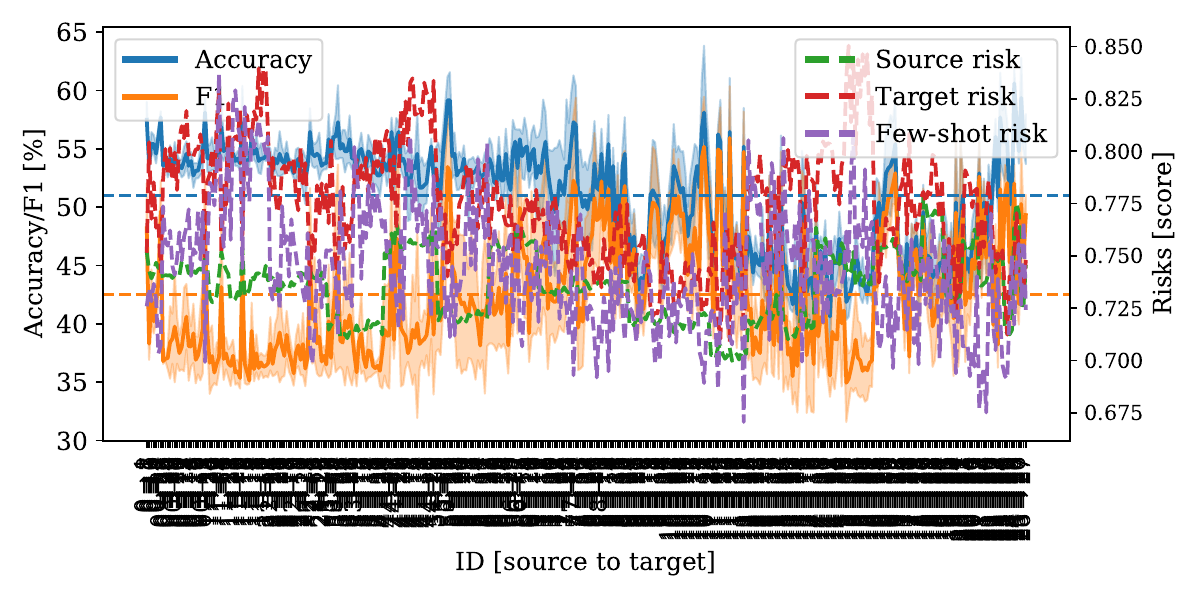}
        \subcaption{Finger movements (univariate).}
        \label{label_figure_results_source_target6}
    \end{minipage}
    \caption{Evaluation results of a large collection of source--target DA pairs for each dataset. The dashed lines show average results over all source--target combinations.}
    \label{label_figure_results_source_target}
\end{figure*}

\subsection{Evaluation Results for Source\,/\,Target Domains} 
\label{app:experimental_results_source_target}

Figure~\ref{label_figure_results_source_target} shows that the DA performance strongly depends on the selected source--target pair. For HHAR (Figure~\ref{label_figure_results_source_target1}) and uWave (Figure~\ref{label_figure_results_source_target5}), most combinations yield high and stable accuracy/F1, indicating that the learned representation transfers well for many user pairs; however, individual transfer directions still show pronounced drops, reflecting asymmetric domain shifts between specific subjects. HAR (Figure~\ref{label_figure_results_source_target2}) also achieves high average performance, but the dense set of source--target combinations reveals many local fluctuations, suggesting that some subjects are substantially easier transfer sources than others. In contrast, WISDM (Figure~\ref{label_figure_results_source_target3}) and Finger movements (Figure~\ref{label_figure_results_source_target6}) exhibit considerably higher variability: WISDM shows large oscillations in accuracy and F1 across pairs, while Finger movements remains close to chance level for many transfers, indicating weak transferability and unreliable target adaptation. EEG (Figure~\ref{label_figure_results_source_target4}) lies between these extremes, with moderate average performance and several difficult source--target pairs. The risk curves further confirm this behavior: low source risk does not always imply low target or few-shot risk, showing that good source-domain fitting alone is insufficient for robust DA. Overall, the figure highlights that averaged benchmark results can hide substantial pairwise variability, and that DA methods should be evaluated across all source--target combinations rather than only on selected transfer scenarios (as done in previous DA studies).

\section{Source-Code Structure}
\label{app:source_code}

The supplementary archive contains the complete implementation of CPDA together with the training and evaluation pipeline, the evaluated DA baselines, encoder backbones, dataset and hyperparameter configurations, runtime measurements, and result-aggregation utilities. The implementation follows the experimental workflow described in the main paper and supports cross-domain experiments, source-only and target-supervised reference models, controlled sinusoidal experiments, hyperparameter studies, and repeated evaluations with different random seeds. The processed datasets are not included because of their size and separate licensing conditions; the expected data format and directory organization are documented in the accompanying \texttt{README.md}. The archive is organized as follows:

\begingroup
\small
\begin{verbatim}
AAAI_CPDA_2027_supplementary_material/
|-- README.md
|-- requirements.txt
`-- source_code/
    |-- main.py
    |-- main_time.py
    |-- main_hyper.py
    |-- trainer.py
    |-- trainer_hyper.py
    |-- same_domain_trainer.py
    |-- utils.py
    |
    |-- algorithms/
    |   `-- algorithms.py
    |
    |-- configs/
    |   |-- data_model_configs.py
    |   |-- hparams.py
    |   `-- sweep_params.py
    |
    |-- dataloader/
    |   `-- dataloader.py
    |
    |-- models/
    |   |-- models.py
    |   |-- loss.py
    |   `-- augmentations.py
    |
    |-- experiments_logs/
    |   |-- count.sh
    |   |-- coun_onhw.sh
    |   |-- plot_logs.py
    |   `-- plot_results.py
    |
    |-- src_only_saved_models/
    |   |-- plot_backbones_all_models.png
    |   |-- plot_backbones_sin_cos.png
    |   |-- plot_results_all_models.py
    |   `-- plot_results_sin_cos.py
    |
    |-- run_backbone.sh
    |-- run_cpda.sh
    |-- run_cpda_onhw.sh
    `-- run_cpda_sincos.sh
\end{verbatim}
\endgroup

\paragraph{Implementation and Reproducibility.} The experimental pipeline is controlled by \texttt{main.py}, with training, evaluation, runtime measurement, hyperparameter studies, and source-only or target-supervised reference experiments implemented in \texttt{trainer.py}, \texttt{main\_time.py}, \texttt{main\_hyper.py}, \texttt{trainer\_hyper.py}, and \texttt{same\_domain\_trainer.py}. The file \texttt{algorithms/algorithms.py} contains CPDA and all evaluated baselines, including the class-conditional path discrepancy, information-maximization and VAT regularization, target-prior weighting, and active-class normalization. Encoder architectures, losses, and augmentations are defined in \texttt{models/}, while dataset properties, source--target scenarios, optimization parameters, and search spaces are specified in \texttt{configs/}; processed data are loaded by \texttt{dataloader/dataloader.py}. The provided shell scripts launch the benchmark, OnHW, sinusoidal, and backbone experiments, and the utilities in \texttt{utils.py}, \texttt{experiments\_logs/}, and \texttt{src\_only\_saved\_models/} handle deterministic initialization, logging, checkpointing, metric aggregation, and figure generation. Installation instructions, data formatting, execution commands, supported methods and backbones, output conventions, and reproduction details are documented in \texttt{README.md}, while \texttt{requirements.txt} records the software environment.

\clearpage

\section*{Reproducibility Checklist}
\label{app:reproducibility_checklist}

\checksubsection{General Paper Structure}
\begin{itemize}

\question{Includes a conceptual outline and/or pseudocode description of AI methods introduced}{(yes/partial/no/NA)}
Yes. We include both a conceptual description and pseudocode for the proposed method. Specifically, App.~\ref{app:algorithm} provides a step-by-step description of one CPDA training iteration, including latent path extraction, target soft pseudo-label estimation, construction of pooled, path, spectral, and signature-based features, computation of the composite kernel matrices, class-wise weighted MMD estimation, and optimization of the final training objective. The corresponding pseudocode is given in Algorithm~\ref{alg:cpda}. It explicitly states the required source and target mini-batches, the feature extractor and classifier, the construction of class-conditional weights, the computation of the CPDA discrepancy, and the final update objective.

\question{Clearly delineates statements that are opinions, hypothesis, and speculation from objective facts and results}{(yes/no)}
Yes. Formal methodological and theoretical claims are stated through definitions, propositions, theorems, and proofs in Sections~\ref{sec:method} and~\ref{sec:theory} and App.~\ref{app:theory}. Measured experimental findings are reported separately in Section~\ref{label_evaluation} and App.~\ref{app:experimental_results}, while explanatory interpretations and limitations are presented as discussion rather than as established facts.

\question{Provides well-marked pedagogical references for less-familiar readers to gain background necessary to replicate the paper}{(yes/no)}
Yes. App.~\ref{app:notations} summarizes the notation used throughout the paper. We provide pedagogical background material in the appendices to support reproducibility. App.~\ref{app:comparison_methods} gives a unified mathematical overview of the main discrepancy-based baselines, including MMD, CORAL, MMDA, DAN, HoMM, and VAT, together with their objectives and limitations for time-series data. App.~\ref{app:theory} states the assumptions used in the CPDA analysis and provides additional theoretical details, including consistency, finite-sample concentration, and pseudo-label stability results. These sections are intended to help less-familiar readers understand both the proposed method and the comparison methods needed to replicate the study. Section~\ref{sec:related_work} introduces the main classes of DA methods and their relationships to CPDA. App.~\ref{app:special_cases} explains how established discrepancy objectives arise as restricted cases of the proposed framework.

\end{itemize}
\checksubsection{Theoretical Contributions}
\begin{itemize}

\question{Does this paper make theoretical contributions?}{(yes/no)}
Yes. Section~\ref{sec:theory} establishes kernel validity, the integral-probability-metric interpretation, a class-conditional target-risk bound, the effect of pseudo-label error, and connections to existing discrepancy objectives. App.~\ref{app:theory} provides additional consistency, finite-sample concentration, and pseudo-label stability results.

	\ifyespoints{\vspace{1.2em}If yes, please address the following points:}
        \begin{itemize}
	
	\question{All assumptions and restrictions are stated clearly and formally}{(yes/partial/no)}
	Yes. The assumptions required by the main target-risk and pseudo-label results are stated in Section~\ref{sec:theory}. App.~\ref{app:assumptions} additionally states the bounded-loss, RKHS-regularity, bounded-kernel, and pseudo-label assumptions used by the supplementary theoretical results. Implementation restrictions, including one-step paths for vector-output backbones, detached target weights, feature normalization, and fixed random projections, are documented in App.~\ref{app:implementation}.

	\question{All novel claims are stated formally (e.g., in theorem statements)}{(yes/partial/no)}
	Yes. The CPDA discrepancy and training objective are formally defined in Section~\ref{sec:method}. The principal theoretical claims are stated as propositions, theorems, and a corollary in Section~\ref{sec:theory} and App.~\ref{app:theory}. The relationship to existing discrepancy methods is formalized in Proposition~\ref{prop:connections} and Table~\ref{tab:cpda_special_cases}.

	\question{Proofs of all novel claims are included}{(yes/partial/no)}
	Yes. Proofs for kernel validity, the IM interpretation, the class-conditional target-risk bound, the equal-prior corollary, the pseudo-label risk result, and the restricted-case relationships are included in Section~\ref{sec:theory}. Proofs for empirical consistency, finite-sample concentration, and stability under pseudo-label perturbations are included in App.~\ref{app:theory}.

	\question{Proof sketches or intuitions are given for complex and/or novel results}{(yes/partial/no)}
	Yes. We provide intuition and a conceptual explanation for the proposed method in addition to the formal derivations. In particular, Fig.~\ref{fig:cpda_overview} gives a visual overview of CPDA by illustrating how source and target time-series are mapped to latent temporal paths, transformed into pooled, temporal-path, spectral, and signature-based representations, and aligned using class-conditional weights from source labels and target soft pseudo-labels. Sections~\ref{sec:problem_statement} and~\ref{sec:method} explain why marginal vector-level alignment may be insufficient for time-series. Algorithm~\ref{alg:cpda} and App.~\ref{app:special_cases} provide operational intuition for the training procedure and its relationship to established discrepancy losses.

	\question{Appropriate citations to theoretical tools used are given}{(yes/partial/no)}
	Yes. The paper cites the relevant work on RKHS mean embeddings and MMD, sequential kernels, path signatures, Fourier representations, covariance and moment matching, pseudo-labeling, and VAT. These references appear in Sections~\ref{sec:related_work}--\ref{sec:theory} and are summarized mathematically in App.~\ref{app:comparison_methods}.

	\question{All theoretical claims are demonstrated empirically to hold}{(yes/partial/no/NA)}
	Partially. The practical effectiveness and stability of CPDA are evaluated in Table~\ref{table_all_results_da}, Figures~\ref{label_cpda_training_dynamics} and~\ref{label_hyperparameter_search}, and App.~\ref{app:experimental_results}. The controlled sinusoidal study in Figures~\ref{figure_results_sin_cos} and~\ref{figure_results_sin_cos_extended} further evaluates robustness under increasing domain shift. Mathematical properties such as positive semidefiniteness, consistency, and finite-sample concentration are established analytically and are not separately tested as empirical hypotheses.

	\question{All experimental code used to eliminate or disprove claims is included}{(yes/no/NA)}
	Yes. The supplementary archive includes the implementations, experiment launchers, hyperparameter configurations, runtime evaluation, metric aggregation, and plotting utilities used for the benchmark, sensitivity analysis, controlled sinusoidal study, and source- and target-only reference experiments. The archive structure is documented in App.~\ref{app:source_code} and the accompanying \texttt{README.md}.
	
	\end{itemize}
\end{itemize}

\checksubsection{Dataset Usage}
\begin{itemize}

\question{Does this paper rely on one or more datasets?}{(yes/no)}
Yes. The experiments use the time-series datasets summarized in Table~\ref{tab:datasets} and described in App.~\ref{app:datasets}. They cover human activity recognition, sleep-stage classification, gesture recognition, biomedical signals, handwriting recognition, and controlled synthetic waveforms.

\ifyespoints{If yes, please address the following points:}
\begin{itemize}

	\question{A motivation is given for why the experiments are conducted on the selected datasets}{(yes/partial/no/NA)}
	Yes. Section~\ref{label_experiments} and App.~\ref{app:datasets} explain that the collection covers different sensors, modalities, temporal lengths, numbers of classes, domain definitions, dataset sizes, and source--target shifts. This diversity is used to evaluate whether CPDA generalizes beyond one application or encoder configuration.

	\question{All novel datasets introduced in this paper are included in a data appendix}{(yes/partial/no/NA)}
	NA. We do not introduce a new real-world dataset in this paper. All evaluation datasets are existing public benchmark datasets, used for controlled ablations, which is fully described in the dataset section App.~\ref{app:datasets} and can be regenerated from the specified data-generation procedure.

	\question{All novel datasets introduced in this paper will be made publicly available upon publication of the paper with a license that allows free usage for research purposes}{(yes/partial/no/NA)}
	NA. No novel dataset is introduced in this paper.

	\question{All datasets drawn from the existing literature (potentially including authors' own previously published work) are accompanied by appropriate citations}{(yes/no/NA)}
	Yes. All datasets drawn from the existing literature are accompanied by appropriate citations in Table~\ref{tab:datasets}. The table lists the dataset name, reference, description, number of classes, and sample count, and also provides the public sources from which the datasets can be obtained.

	\question{All datasets drawn from the existing literature (potentially including authors' own previously published work) are publicly available}{(yes/partial/no/NA)}
	Yes. Table~\ref{tab:datasets} provides the public benchmark sources, including the AdaTime data repository, the UCR/UEA archive, and the OnHW dataset website. The controlled synthetic benchmark is cited to its original publication.

	\question{All datasets that are not publicly available are described in detail, with explanation why publicly available alternatives are not scientifically satisficing}{(yes/partial/no/NA)}
	NA. No private or access-restricted dataset is used. The supplementary archive does not duplicate the processed datasets because of their size and separate distribution conditions; their public sources, expected tensor format, and directory structure are documented in Table~\ref{tab:datasets}, App.~\ref{app:datasets}, and the supplementary \texttt{README.md}.

\end{itemize}
\end{itemize}

\checksubsection{Computational Experiments}
\begin{itemize}

\question{Does this paper include computational experiments?}{(yes/no)}
Yes. The experimental design is described in Section~\ref{label_experiments}, the principal results are reported in Section~\ref{label_evaluation}, and the complete results are provided in App.~\ref{app:experimental_results}.

\ifyespoints{If yes, please address the following points:}
\begin{itemize}

	\question{This paper states the number and range of values tried per (hyper-) parameter during development of the paper, along with the criterion used for selecting the final parameter setting}{(yes/partial/no/NA)}
	Yes. Table~\ref{tab:cpda_hparams} lists the default CPDA configuration and the explored ranges for the learning rate, objective weights, kernel-component weights, signature dimension and order, and maximum latent sequence length. Figure~\ref{label_hyperparameter_search} reports the resulting sensitivity across datasets, and the full sweep spaces are included in \texttt{configs/sweep\_params.py}.

	\question{Any code required for pre-processing data is included in the appendix}{(yes/partial/no)}
	Yes. The supplementary material includes the dataset configurations, source--target scenarios, expected tensor dimensions, and the loader for the processed domain-specific files. App.~\ref{app:datasets} documents the benchmark preprocessing and splits, and the \texttt{README.md} specifies the required file format and directory structure. The archive does not include complete raw-data download, but does contain the conversion scripts for every external dataset.

	\question{All source code required for conducting and analyzing the experiments is included in a code appendix}{(yes/partial/no)}
	Yes. The supplementary archive contains CPDA and baseline implementations, encoder architectures, losses, augmentations, data loading, dataset and hyperparameter configurations, cross-domain training, source- and target-only reference training, runtime measurement, metric calculation, result aggregation, and plotting utilities. The complete structure and the correspondence between files and experimental stages are documented in App.~\ref{app:source_code}.

	\question{All source code required for conducting and analyzing the experiments will be made publicly available upon publication of the paper with a license that allows free usage for research purposes}{(yes/partial/no)}
	Yes. The source code included in the supplementary archive will be made publicly available upon publication under a license permitting free research use. This includes the proposed method, evaluated baselines, training and evaluation pipeline, configurations, and result-analysis utilities described in App.~\ref{app:source_code}.
        
	\question{All source code implementing new methods have comments detailing the implementation, with references to the paper where each step comes from}{(yes/partial/no)}
	Partial. The released source code includes comments for the implementation of the proposed CPDA components and identifies the main steps corresponding to the method described in the paper, including latent path construction, signature--spectral feature extraction, class-conditional weighting, CPDA discrepancy computation, and VAT regularization. We will ensure that the final public release contains sufficient inline comments and references to the relevant equations, algorithm, and appendix sections.

	\question{If an algorithm depends on randomness, then the method used for setting seeds is described in a way sufficient to allow replication of results}{(yes/partial/no/NA)}
	Yes. Section~\ref{label_evaluation} states that each source--target configuration is repeated five times with different random seeds. The supplementary implementation sets the Python, NumPy, PyTorch, and CUDA seeds and configures deterministic cuDNN execution before every run. The per-run seed is determined by the run identifier.

	\question{This paper specifies the computing infrastructure used for running experiments (hardware and software), including GPU/CPU models; amount of memory; operating system; names and versions of relevant software libraries and frameworks}{(yes/partial/no)}
	Yes. We specify the hardware infrastructure used for the experiments, including NVIDIA Tesla V100-SXM2 GPUs with 32 GB VRAM, Intel Xeon CPUs, and 192 GB RAM. The final version and released code will additionally document the relevant software environment, including the operating system, Python version, PyTorch version, CUDA version, and main library dependencies required to reproduce the experiments.

	\question{This paper formally describes evaluation metrics used and explains the motivation for choosing these metrics}{(yes/partial/no)}
	Yes. Section~\ref{label_evaluation}, Table~\ref{table_all_results_da}, and the tables in App.~\ref{app:experimental_results_da_methods} consistently report classification accuracy and F1-score. The supplied evaluation code computes macro-averaged F1, which is suitable for datasets with unequal class frequencies. The current manuscript does not provide explicit mathematical definitions of accuracy and macro-F1 or explicitly identify the reported F1 variant in the main text.

	\question{This paper states the number of algorithm runs used to compute each reported result}{(yes/no)}
	Yes. Section~\ref{label_evaluation} states that every source--target configuration is trained five times with different random seeds. Table~\ref{table_all_results_da}, the tables in App.~\ref{app:experimental_results_da_methods}, and Figures~\ref{figure_results_sin_cos} and~\ref{figure_results_sin_cos_extended} report results aggregated over these five runs.

	\question{Analysis of experiments goes beyond single-dimensional summaries of performance (e.g., average; median) to include measures of variation, confidence, or other distributional information}{(yes/no)}
	Yes. The tables in App.~\ref{app:experimental_results_da_methods} report means and standard deviations over source--target scenarios and repeated runs. The sinusoidal results in Figures~\ref{figure_results_sin_cos} and~\ref{figure_results_sin_cos_extended} include shaded standard-deviation regions, while Figure~\ref{label_figure_results_source_target} exposes performance and risk variation across individual source--target pairs. Figure~\ref{label_cpda_training_dynamics} additionally reports optimization and target-prediction dynamics.

	\question{The significance of any improvement or decrease in performance is judged using appropriate statistical tests (e.g., Wilcoxon signed-rank)}{(yes/partial/no)}
	No. We do not currently use formal statistical significance tests such as the Wilcoxon signed-rank test. Instead, we report mean and standard deviation over five random seeds for each source--target configuration. We therefore interpret improvements based on average performance and variability, and we avoid making claims of statistical significance unless supported by an explicit test.

	\question{This paper lists all final (hyper-)parameters used for each model/algorithm in the paper’s experiments}{(yes/partial/no/NA)}
	Partial. Table~\ref{tab:cpda_hparams} lists the final CPDA optimizer, loss, kernel, signature, pseudo-label, normalization, and ramp-up parameters. Backbone configurations are summarized in Tables~\ref{tab:backbone_summary} and~\ref{tab:backbone_parameters}. The complete dataset- and method-specific settings for CPDA and the evaluated baselines are included in \texttt{configs/data\_model\_configs.py} and \texttt{configs/hparams.py}, but the complete per-baseline configurations are not reproduced as tables in the paper.
    
\end{itemize}
\end{itemize}

\end{document}